\documentclass[twoside,11pt]{article}

\usepackage{amsmath}
\usepackage{amssymb}
\usepackage{amsthm}
\usepackage{bbm}
\usepackage{thmtools}
\usepackage{thm-restate}

\usepackage[font=small]{caption}
\usepackage{subcaption}

\usepackage{jmlr2e}

\usepackage[utf8]{inputenc} 
\usepackage[T1]{fontenc}    
\usepackage{hyperref}
\usepackage{url}            
\usepackage{booktabs}       
\usepackage{amsfonts}       
\usepackage{microtype}      
\usepackage{tikz}
\usepackage{graphicx}
\usepackage{booktabs}
\usepackage{multirow}
\usepackage[shortlabels]{enumitem}
\setlist{nosep}
\usepackage{xspace}
\usepackage{booktabs}
\usepackage{float}
\usepackage{color}
\usepackage{mathtools}
\usepackage{bm}
\usepackage{xcolor}
\usepackage[normalem]{ulem}
\usepackage{siunitx}
\usepackage{makecell}
\usepackage[noend]{algpseudocode}
\usepackage{algorithm,algorithmicx}

\usepackage{tabularx}
\usepackage{cleveref}
\usepackage{anyfontsize}

\usepackage{pgfplots,tikz,pgfplotstable}
\pgfplotsset{compat=1.9}

\newif\ifcompiletikz
\compiletikztrue

\ifcompiletikz
 \usetikzlibrary{external}
\else
 \usepackage{environ}
 \newcounter{tikzfigcntr}
 \RenewEnviron{tikzpicture}[1][]{%
  \par
  \stepcounter{tikzfigcntr}
  THIS IS \texttt{tikzpicture}~\thetikzfigcntr
  \par
 }
\fi

\def\mystrut{\vphantom{hg}}
\pgfplotsset{
    legend image with text/.style={
        legend image code/.code={%
            \node[anchor=center] at (0.3cm,0cm) {#1};
        }
    },
}
\pgfplotsset{%
  redweak/.style = {red!70!black,densely dotted,
                    mark=o, mark options={scale=0.7,solid}}}
\pgfplotsset{%
  redhollow/.style = {red!100!white,densely dashed,
                      mark=o, mark options={scale=1.0,solid}}}
\pgfplotsset{%
  redsolid/.style = {red!50!white,mark=*,mark options={scale=1.0,solid}}}

\pgfplotsset{%
  orangeweak/.style = {orange!70!black,densely dotted,mark=diamond,
                       mark options={scale=1.2*0.7,solid}}}
\pgfplotsset{%
  orangehollow/.style = {orange!100!white,densely dashed,mark=diamond,
                         mark options={scale=1.2,solid}}}
\pgfplotsset{%
  orangesolid/.style = {orange!50!white,mark=diamond*,
                        mark options={scale=1.2,solid}}}

\pgfplotsset{%
  violetweak/.style = {violet!70!black,densely dotted,mark=diamond,
                       mark options={scale=1.2*0.7,solid}}}
\pgfplotsset{%
  violethollow/.style = {violet!100!white,densely dashed,mark=diamond,
                         mark options={scale=1.2,solid}}}
\pgfplotsset{%
  violetsolid/.style = {violet!50!white,mark=diamond*,
                        mark options={scale=1.2,solid}}}

\pgfplotsset{%
  blueweak/.style = {blue!70!black,densely dotted,mark=square,
                     mark options={scale=0.8*0.7,solid}}}
\pgfplotsset{%
  bluehollow/.style = {blue!100!white,densely dashed,mark=square,
                       mark options={scale=0.8,solid}}}
\pgfplotsset{%
  bluesolid/.style = {blue!50!white,mark=square*,
                      mark options={scale=0.8,solid}}}

\pgfplotsset{%
  tealweak/.style = {teal!70!black,densely dotted,mark=triangle,
                     mark options={scale=1.2*0.7,solid}}}
\pgfplotsset{%
  tealhollow/.style = {teal!100!white,densely dashed,mark=triangle,
                       mark options={scale=1.2,solid}}}
\pgfplotsset{%
  tealsolid/.style = {teal!50!white,mark=triangle*,
                      mark options={scale=1.2,solid}}}

\usepackage{marvosym}

\usepackage{pdfrender}

\bibpunct{(}{)}{;}{a}{}{,}

\newcommand\numberthis{\addtocounter{equation}{1}\tag{\theequation}}

\DeclareMathOperator*{\argmax}{arg\,max}

\DeclareMathOperator*{\softmax}{softmax}
\DeclareMathOperator*{\diag}{diag}

\newcommand{\N}{\mathcal{N}}
\newcommand{\Nb}{\mathbb{N}}
\newcommand{\Rb}{\mathbb{R}}

\newcommand{\indcia}{\mathbbm{1}_{[1,k]}(i)z}
\newcommand{\E}{\mathop{{}\mathbb{E}}}

\DeclareMathOperator{\KL}{KL}

\DeclareMathOperator{\elbo}{\mathcal{L}_{\text{ELBO}}}

\newcommand{\indep}{\perp \!\!\! \perp}

\newcommand{\betavae}{$\boldsymbol{\beta}$-VAE\xspace}

\newcommand{\nlltoppl}[1]{%
  \pgfmathparse{exp(#1)}%
  \pgfmathprintnumber[fixed,zerofill,precision=1,assume math mode=true]{\pgfmathresult}}

\newcommand{\cX}{\mathcal{X}}

\usepackage{xargs} 
\usepackage[colorinlistoftodos,prependcaption,textsize=tiny,disable]{todonotes}

\newcommand{\reducedstrut}{\vrule width 0pt height .9\ht\strutbox depth .9\dp\strutbox\relax}
\newcommand{\colorboxx}[2]{%
  \begingroup%
  \setlength{\fboxsep}{0pt}%
  \colorbox{#1}{\reducedstrut#2\/}%
  \endgroup
}
\newcommandx{\mgl}[2][1=]{\colorboxx{blue!30}{\tiny\,}\todo[color=blue!30,#1]{MG: #2}}
\newcommandx{\todoa}[2][1=]{\colorboxx{brown!30}{\tiny\,}\todo[color=brown!30,#1]{AG: #2}}
\newcommandx{\pb}[2][1=]{\colorboxx{gray!30}{\tiny\,}\todo[color=gray!30,#1]{PB: #2}}
\newcommandx{\cd}[2][1=]{\colorboxx{cyan!30}{\tiny\,}\todo[color=cyan!30,#1]{CD: #2}}
\newcommandx{\jr}[2][1=]{\colorboxx{red!30}{\tiny\,}\todo[color=red!30,#1]{JR: #2}}

\pgfplotscreateplotcyclelist{acycle}{%
solid, every mark/.append style={solid, fill=gray}, mark=*\\%
dotted, every mark/.append style={solid, fill=gray}, mark=square*\\%
densely dotted, every mark/.append style={solid, fill=gray}, mark=otimes*\\%
loosely dotted, every mark/.append style={solid, fill=gray}, mark=triangle*\\%
dashed, every mark/.append style={solid, fill=gray},mark=diamond*\\%
loosely dashed, every mark/.append style={solid, fill=gray},mark=*\\%
densely dashed, every mark/.append style={solid, fill=gray},mark=square*\\%
dashdotted, every mark/.append style={solid, fill=gray},mark=otimes*\\%
dashdotdotted, every mark/.append style={solid},mark=star\\%
densely dashdotted,every mark/.append style={solid, fill=gray},mark=diamond*\\%
}

\allowdisplaybreaks

\title{Mutual Information Constraints for Monte-Carlo Objectives\\
  to Prevent Posterior Collapse\\
  Especially in Language Modelling}
\author{G\'abor Melis \email melisgl@google.com \\
  \addr DeepMind, London, UK; University College London, UK
  \AND
  Andr\'as Gy\"orgy \email agyorgy@google.com \\
  \addr DeepMind, London, UK
  \AND
  Phil Blunsom \email pblunsom@google.com \\
  \addr DeepMind, London, UK; Oxford University, UK
}
\editor{David Sontag}

\usepackage{lastpage}
\jmlrheading{23}{2022}{1-\pageref{LastPage}}{12/20; Revised
2/22}{4/22}{20-1358}{G\'abor Melis, Andr\'as Gy\"orgy and Phil Blunsom}
\ShortHeadings{MI Constraints for MC Objectives}{Melis, Gy\"orgy and Blunsom}

\begin{document}

\maketitle

\begin{abstract}%
Posterior collapse is a common failure mode of density models trained as variational autoencoders, wherein they model the data without relying on their latent variables, rendering these variables useless.
We focus on two factors contributing to posterior collapse, that have been studied separately in the literature.
First, the underspecification of the model, which in an extreme but common case allows posterior collapse to be the theoretical optimium.
Second, the looseness of the variational lower bound and the related underestimation of the utility of the latents.
We weave these two strands of research together, specifically the tighter bounds of multi-sample Monte-Carlo objectives and constraints on the mutual information between the observable and the latent variables.
The main obstacle is that the usual method of estimating the mutual information as the average Kullback-Leibler divergence between the easily available variational posterior $q(z|x)$ and the prior does not work with Monte-Carlo objectives because their $q(z|x)$ is not a direct approximation to the model's true posterior $p(z|x)$.
Hence, we construct estimators of the Kullback-Leibler divergence of the true posterior from the prior by recycling samples used in the objective, with which we train models of continuous and discrete latents at much improved rate-distortion and no posterior collapse.
While alleviated, the tradeoff between modelling the data and using the latents still remains, and we urge for evaluating inference methods across a range of mutual information values.
\end{abstract}


\section{Introduction}

The promise of latent variable models is in learning about the underlying generative process, discovering structure in the data, principled representation learning, improved generalization and controllable generation; all made possible by judicious choice of model structure, such as the prior, the likelihood, and any conditional independence assumptions.
Variational autoencoders (VAEs, \citealp{kingma2013auto,rezende2014stochastic}) provide a general framework for statistical inference in latent variable models of the form $p_\theta(x,z)=p_\theta(x|z)p(z)$, where $x$ is the observable data, $z$ is the vector of latent variables, and the objective is to learn the parameters $\theta$ so that the resulting marginal distribution $p_\theta(x)$ well approximates the empirical data distribution $p_D(x)$.
The generality of VAEs comes at a price, as the variational posterior $q_\phi(z|x)$, used to approximate the true posterior $p_\theta(z|x)$, usually underestimates the variance of the latter \citep{maddison2017filtering},\footnote{While we follow established terminology, strictly speaking $q_\phi$ does not ``estimate'' the variance of $p_\theta$, and underestimation of the variance means that the distribution $q_\phi(z|x)$ has lower variance than that of $p_\theta(z|x)$.} which is often observed as the underuse of latent variables.
In the extreme case, the underestimation leads to ignoring the latents entirely, which is known as posterior collapse \citep{zhao2019infovae} and is the main focus of this work.
The issue of posterior collapse is especially acute with auto-regressive decoders, which are capable of modelling the data without using the latents at all \citep{yang2017improved}.
\Citet{bowman2015generating} attributed this to a ``difficult learning problem'', and dozens of attempts to remedy it followed \citep{alemi2017fixing,dieng2017variational,van2017neural,kim2018semi} to help VAEs fulfill their promise in representation learning.

This work aims to understand and remedy posterior collapse in VAEs with the long-term goal of facilitating research into latent variable models.
While acknowledging that their ultimate evaluation is necessarily in terms of performance on down-stream tasks or as density models, we demonstrate that suboptimal inference can present a severe tradeoff between latent variable usage and data fit.
This inefficiency of inference renders the posterior unfit for its purpose as a representation of the data.
Therefore, instead of measuring the performance of the learned models on specific down-stream tasks, we evaluate this tradeoff in terms of their rate-distortion behaviour \citep{alemi2017fixing}, by measuring the rate as the mutual information between the observables $x$ and the latents $z$, and distortion based on the negative log-likelihood assigned by the model to the data.

Several interacting factors (outlined in \S\ref{sec:posterior-collapse}) play a role in posterior collapse, but the two most pertinent to this paper are the looseness of the lower bound and underspecification.
\emph{The looseness of the variational lower bound} (such as the ELBO objective) biases solutions found by VAEs away from the theoretical optimum.
Hence, designing tighter lower bounds has been a mainstay of research on variational inference \citep{rezende2015variational,kingma2016improved,tomczak2018vae}, and one approach is to take multiple samples from a variational posterior distribution $q_\phi(z|x)$ to form an approximation of the marginal likelihood $p_\theta(x)$.\footnote{Throughout we do not explicitly distinguish between densities and probability mass functions unless it is necessary; these are naturally dictated by the type (continuous/discrete) of the underlying variables.}
In these so-called Monte-Carlo objectives \citep{mnih2016variational}, such as IWAE \citep{burda2015importance}, $q_\phi(z|x)$ does not represent the true posterior $p_\theta(z|x)$ explicitly, but it can be interpreted as a factor in a more elaborate, implicit approximate posterior \citep{cremer2017reinterpreting}.
In this context, it is thus more correct to refer to $q_\phi(z|x)$ as the proposal distribution.

The second factor we consider is \emph{underspecification}.
We use underspecification in the sense that models that are optimal in terms of marginal likelihood can differ greatly in their posteriors \citep{huszar2017representation}, thus the optimization aiming to fit the data by maximizing the likelihood is underspecified.
This issue is most apparent in that, for VAEs with powerful function classes expressing the likelihoods, posterior collapse can be an optimal solution in terms of data fit because the usual evidence lower bound (ELBO) objective is neutral with respect to the mutual information between the latents and the data.
However, as \citet{huszar2017representation} argues and as we show in \S\ref{sec:cia-an-pc}, the marginal likelihood objective leaves the posterior underspecified, so underspecification is a shortcoming of the ELBO in only as much as it does not correct for it.
Many proposed methods aim to address underspecification by constraining the mutual information between the observable and the latent variables \citep{higgins2016beta,phuong2018mutual,zhao2019infovae} relying on the availability of a good posterior approximation.
As discussed next, with Monte-Carlo objectives we do not have access to a good posterior approximation, and this must be addressed if we hope constrain mutual information to reduce underspecification.

When the proposal $q_\phi(z|x)$ is close to $p_\theta(z|x)$, we can approximate the mutual information $I_p(X,Z) = \E_{p_\theta(x)}\KL(p_\theta(z|x)\|p(z))$
(where $\KL$ denotes the Kullback-Leibler divergence) with $\E_{p_\theta(x)}\KL(q_\phi(z|x)\|p(z))$, which features the proposal $q_\phi$.
Unfortunately, with Monte-Carlo objectives we cannot expect the proposal to approximate the posterior well \citep{mnih2016variational}, and $\KL(q_\phi(z|x)\|p(z))$ (called the \emph{representational KL}) is a highly biased estimate of $\KL(p_\theta(z|x)\|p(z))$ (the \emph{true KL} from now on).

\emph{Our main contribution} is a novel method to constrain the mutual information between the observable and the latent variables in the context of multi-sample Monte-Carlo objectives, bringing research on loose bounds and underspecification together.
More specifically, we introduce an optimization objective which features two terms, one coming from the variational lower bound and another from the mutual information, where both terms are based on multiple samples taken from the proposal distribution $q_\phi$.
Compared to the single-sample case, we get the benefit of the tighter lower bounds Monte-Carlo objectives offer without having to give up control of the mutual information.
At the same time, our multi-sample estimators for the mutual information are much more efficient than in the single-sample case and can better tolerate low-quality posterior approximations.
Our mutual information term is computed from the recycled samples of the Monte-Carlo estimator of the marginal likelihood, hence the method has negligible computational overhead.
Combined with best-of-breed gradient estimators, such as DReG \citep{tucker2018doubly} and VIMCO \citep{mnih2016variational} for a multi-sample objective, we train models with continuous and discrete latents at much improved rate-distortion.

The rest of the paper is structured as follows.
\begin{itemize}
\item \S\ref{sec:posterior-collapse} provides an overview of the known causes of posterior collapse, which are all shortcomings of the inference method except for underspecification.
\item In \S\ref{sec:cia-an-pc}, we characterize underspecification as the lack of sufficient conditional independence assumptions, which may be partially offset by constraining mutual information.
\item \S\ref{sec:mi-objective} proposes reusing samples from Monte-Carlo objectives to better estimate the mutual information, which is the main contribution of this paper.
\item \S\ref{sec:connection-to-the-representational-kl} and \S\ref{sec:connection-to-the-beta-vae} show that the representational KL, which underlies many mutual information estimates, corresponds to the single-sample case of our estimators, and the single-sample objectives built on them are equivalent to the \betavae objective of \citet{higgins2016beta}.
\item \S\ref{sec:experiments} experimentally verifies the effectiveness of the proposed methods on synthetic and language modelling tasks, emphasizing evaluation in terms of the data fit vs latent usage tradeoff.
\end{itemize}

\section{Variational Autoencoders and Posterior Collapse}
\label{sec:posterior-collapse}

This section introduces variational autoencoders and describes the known causes of posterior collapse.
Contrary to what the name variational autoencoder may suggest, a VAE is not a model itself but an inference\footnote{We use \emph{inference} in the statistical sense, commonly referred to as \emph{training} or \emph{learning} in the machine learning literature.} mechanism for models of the form $p_\theta(x,z)=p_\theta(x|z)p(z)$ where $\{p_\theta(x|z)\}$ is a parametric family of conditional distributions \citep{kingma2013auto}.
VAE training constructs an approximate maximum likelihood estimate of the model parameters $\theta$ with the aim of maximizing the probability over the empirical data distribution: $\argmax_{\theta} \E_{x \sim p_D(x)}[ \ln p_\theta(x)]$.
Since $\ln p_\theta(x)$ has no analytic form in general, VAEs posit a variational family of distributions $\mathcal{Q} = \{q_\phi(z|x)\}$ parameterized by $\phi$ to approximate the true posterior $p_\theta(z|x)$ and construct a lower bound on the marginal likelihood $p_\theta(x)$, also called the evidence:
\begin{align}
\label{eq:elbo-posterior-contrastive}
\ln p_\theta(x) \geqslant \elbo(x, \theta, \phi)
&= \ln p_\theta(x) - \KL\big(q_\phi(z|x)\|p_\theta(z|x)\big) \\
\label{eq:elbo-prior-contrastive}
&= \E_{z \sim q_\phi(z|x)} [\ln p_\theta(x|z)] - \KL\big(q_\phi(z|x)\|p(z)\big).
\end{align}
As evident in the ``posterior-contrastive'' form of the ELBO \eqref{eq:elbo-posterior-contrastive}, it is a lower bound on $\ln p_\theta(x)$ due to the non-negativity of the KL divergence \citep{kullback1951information}.
In the ``prior-contrastive'' form \eqref{eq:elbo-prior-contrastive}, both the expectation and the KL divergence  term can be estimated by taking a single sample from $q$, which forms the basis of its optimization.
Alternatively, the KL may be computable analytically.
Putting it all together, gradient-based optimization with VAEs is performed jointly over parameters $\theta$ of the model and $\phi$ of the approximate posterior as
\begin{align*}
\argmax_{\theta,\phi} \E_{x \sim p_D(x)} \elbo(x, \theta, \phi).
\end{align*}
In practice, the expectation in \eqref{eq:elbo-prior-contrastive} is approximated with a single sample from $q_\phi(z|x)$ and the expectation over $p_D(x)$ with a ``minibatch'', which goes by the name of doubly stochastic variational inference \citep{titsias2014doubly}.
For a broader context, we refer the reader to \citet{zhang2018advances}, who provide a comprehensive overview of developments in variational inference \citep{jordan1999introduction}.
From this point on, wherever possible we drop the subscripts in $p_\theta$ and $q_\phi$ to declutter the notation.

The possibility of posterior collapse (also referred to as over-pruning, \citealt{yeung2017tackling}, or information preference, \citealt{zhao2019infovae}) is most evident in the prior-contrastive ELBO \eqref{eq:elbo-prior-contrastive}.
If the likelihood $p(x|z)$ is able to model the distribution of $x$ without the latents $z$, then the reconstruction term $\E_{q(z|x)} \ln p(x|z)$ is just $\ln p(x)$ independently of $q(z|x)$.
Since $q$ does not affect the reconstruction term, we can just set it to the prior $p(z)$, which is often in $\mathcal{Q}$, so that the KL penalty is zero.
In this case, $q(z|x) = p(z|x)$ is also satisfied, but unfortunately the two posteriors have now ``collapsed'' to match the prior $p(z)$, which renders the latent variables useless.

\begin{figure}
  \begin{minipage}{0.9\linewidth}
    \hspace*{0.9cm} \includegraphics[scale=0.6,trim={4.5cm 3.2cm 2.5cm
        0.0cm},clip]{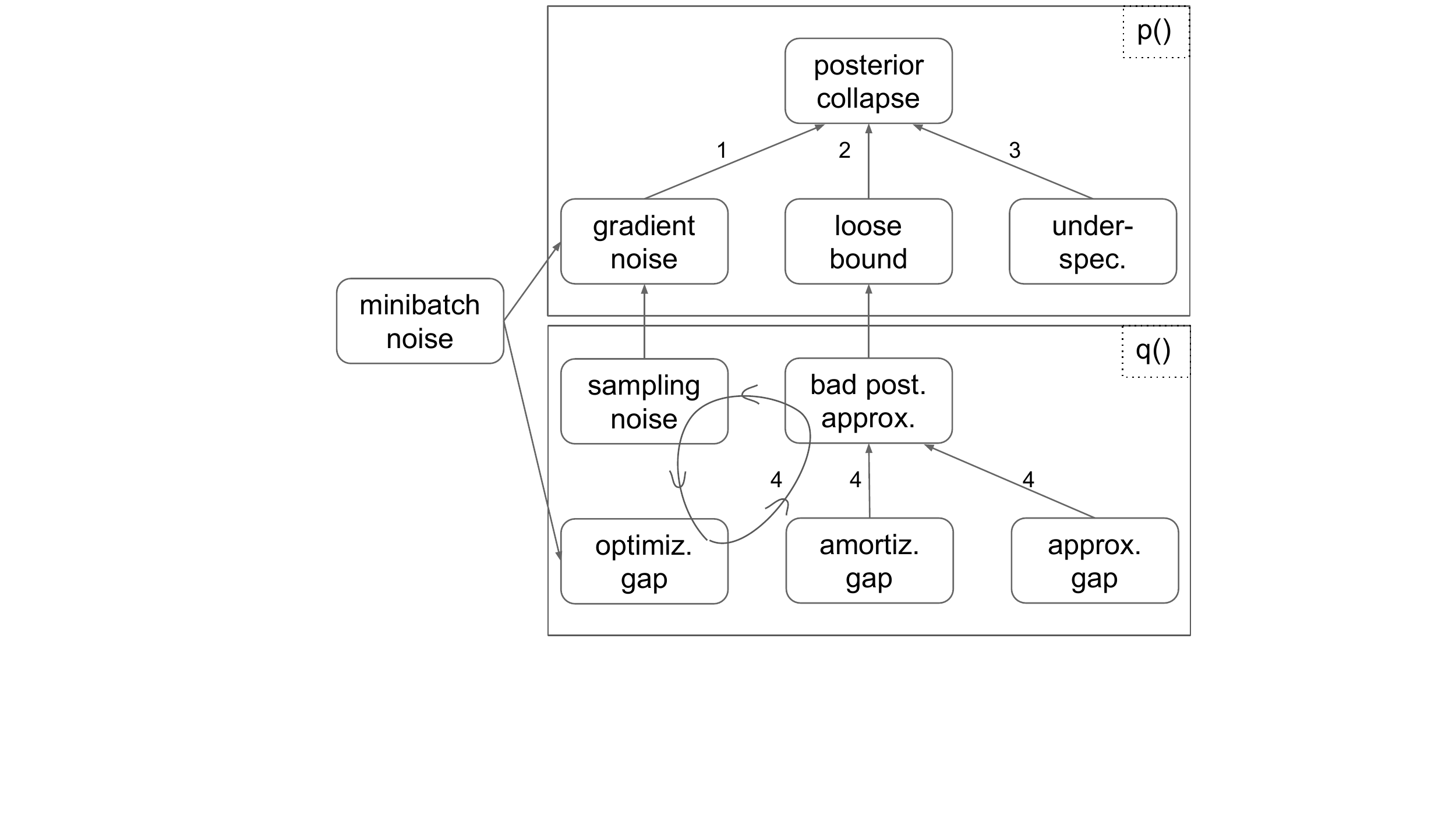}
    \captionof{figure}{\small Causes of posterior collapse in VAEs.
     1.~High variance gradient estimates and SGD's preference for flat minima exerts pressure to reduce variance by ignoring the latents.
     2.~A loose lower bound underestimates the benefits of the latents.
     3.~Underspecification can allow posterior collapse to be the theoretical optimum.
     4.~Posterior collapse reduces or eliminates all of the gaps and the sampling noise.}
    \label{fig:posterior-collapse}
  \end{minipage}
\end{figure}

The issues in what we observe as posterior collapse (or its milder form, the underuse of latents) span a number of causes.
\Cref{fig:posterior-collapse} summarizes the main known contributors to posterior collapse and their interactions.
At a glance, the immediate causes are high-variance gradient estimates, a loose lower bound, and underspecification.
\begin{itemize}
\item \textbf{Gradient noise}:
Optimization can be adversely affected by high-variance gradient estimators \citep{roeder2017sticking,tucker2018doubly} and minibatch noise in stochastic gradient descent \citep{titsias2014doubly}.
Unlike minibatch noise, sampling noise - the variance induced by the latents - can be eliminated by ignoring the latents.
Coupled with SGD's preference for flat minima \citep{hochreiter1995simplifying}, this biases optimization towards posterior collapse.
\item \textbf{Loose lower bound}:
As its posterior-contrastive form \eqref{eq:elbo-posterior-contrastive} suggests, the ELBO is tight if $q(z|x)$ matches $p(z|x)$, and the worse the approximation, the looser the ELBO.
However, perfect posterior approximation might be hard or impossible to achieve, depending on the following factors \citep{cremer2017reinterpreting}.
\setlist{nosep}
\begin{itemize}[leftmargin=0.4cm]
\item[-] The \textbf{approximation gap} is the distance of the true posterior from the variational family $\mathcal{Q}$ if $p(z|x) \not\in \mathcal{Q}$ \citep{DBLP:conf/nips/ShuBZKE18,razavi2018preventing}.
\item[-] The \textbf{amortization gap} is the gap caused by the encoder's inability to represent the optimal $q_{\phi^\star}(z|x)$ (which is $p_\theta(z|x)$) for all $x$ with the same $\phi$ \citep{kim2018semi,DBLP:conf/nips/ShuBZKE18}.
\item[-] The \textbf{optimization gap} is caused the suboptimality of $q_\phi(z|x)$ found by the optimizer relative to $q_{\phi^\star}(z|x)$ \citep{he2018lagging}.
\end{itemize}
Notably, there is a negative feedback loop: optimization difficulties cause bad posterior approximation, which increases the variance induced by sampling the latents, which makes optimization harder.
\item \textbf{Underspecification}: As argued by \citet{alemi2017fixing}, the ELBO is neutral with respect to the mutual information, and even perfect optimization can land anywhere on the rate-distortion curve (the expected likelihood of the data as a function of the mutual information of the data and the latents).
More generally, the optimization task for generative models of the form $p(x,z)= p(x|z)p(z)$ is often underspecified \citep{huszar2017representation}, which we explore in \S\ref{sec:cia-an-pc}.
\end{itemize}

Unfortunately, posterior collapse reduces or eliminates all of the gaps and also the sampling noise, which makes the bound tight.
For VAEs to use their latent variables reliably and optimally, all of the above issues must be addressed.
The direction we take in this paper delegates the problem of dealing with the looseness of the bound to a Monte-Carlo estimator of the marginal likelihood, such as IWAE, and that of reducing gradient noise to a corresponding gradient estimator, such as DReG or VIMCO.
Importantly, in addition to providing tighter bounds, Monte-Carlo estimators employ multiple samples from $q(z|x)$, which benefits our efforts to tackle the issue of underspecification by designing better estimators of the mutual information.

\section{Related Work}

With several interacting issues to disentangle, the body of work on posterior collapse has grown large and disparate.
Many works aim to make the prior or the posterior more flexible to tighten the lower bound \citep{rezende2015variational,kingma2016improved,tomczak2018vae}.
At the same time, \citet{DBLP:conf/nips/ShuBZKE18} argue that while it is tempting to think that the variational familiy $\mathcal{Q}$ and $q_{\phi}(z|x)$ should be as large and as flexible possible, a more restrictive choice can prove useful in performing posterior regularization.
\Citet{kim2018semi} propose reducing the amortization gap by only initializing the value of the latent variables according to the encoder, then optimizing their values separately for each example as in non-amortized variational inference \citep{jordan1999introduction}.
\Citet{he2018lagging} focus on improving the approximation quality of $q_\phi(z|x)$ by taking several optimization steps on $\phi$ for each update of $\theta$ in $p_{\theta}$.
\Citet{yang2017improved} replace an auto-regressive decoder with dilated convolutions to restrict the available context and to thus enforce the use of the latents.
\Citet{van2017neural} introduce categorically distributed latents with a fixed KL cost and replace the prior with the aggregate posterior, which eliminates any pressure to ignore the latents.

Controlling the KL term in \eqref{eq:elbo-prior-contrastive} by various means also received lots of attention.
Downweighting or annealing it during training is common \citep{bowman2015generating,higgins2016beta,sonderby2016ladder,NEURIPS2019_9bdb8b1f,NEURIPS2018_65b0df23}.
\Citet{kingma2016improved} ``flat spot'' the KL term so that small KL values are not penalized.
\Citet{razavi2018preventing} achieve a similar effect by choosing the variational family and the approximate posterior such that there is always an approximation gap.

Several works \citep{alemi2017fixing,zhao2019infovae,mccarthy2019improved,rezaabad2020learning,serdega2020vmi} recognize the issue of underspecification and propose to constrain the mutual information.
All work with single sample VAEs and the representational KL.
In contrast, we propose combining multi-sample Monte-Carlo objectives with  estimates based on the true KL.
As we will see, this results in our mutual information estimates being less tied to the quality of $q(z|x)$, while containing the single-sample, representational KL scenario as a degenerate case.

\section{Conditional Independence Assumptions and Posterior Collapse}
\label{sec:cia-an-pc}

In this section, we explore \emph{why} our models are underspecified to better understand whether constraining the mutual information is a good solution.
More specifically, we ask under what conditional independence assumptions (CIA\footnote{Due to CIA being a collective noun in other contexts, we let its singular form stand also for its plural.}) can posterior collapse be an optimal solution for \emph{any} model to abstract away from finite capacities and optimization difficulties.
The next proposition shows that the answer is in fact trivial: the independence assumptions can be satisfied with a model where the latent and the observable variables are independent and the model matches the data distribution perfectly if and only if the independence assumptions are compatible with the marginal distribution of the data.

In particular, we consider Bayesian networks, given as $p(x,z)=p(z) \prod_{i=1}^n p(x_i|\mathrm{pa}(i),\indcia)$, where $\mathrm{pa}(i)$ denotes the set of parent nodes of $x_i$, i.e. those on which it directly depends, and the first $k$ $x_i$ directly depend also on $z$.

\begin{proposition}
Let $\{x_i\}$ ($i \in [1,n]$) be a partitioning of $x$, and $\mathrm{pa}(i) \subseteq \{x_i\}$.
For any prior $p(z)$, there exists a distribution\footnote{Here we consider the set of all distributions and do not restrict attention to the parametrized family $\{p_\theta\}$.}
$p(x,z)$ satisfying, for all $x$ and $z$,
\begin{enumerate}[(i)]
\item $p(x,z)=p(x)p(z)$ \hspace{0.1em} (posterior collapse);
\item $p(x)=p_D(x)$ \hspace{0.1em} (perfectly modelling the data);
\item $p(x,z)= p(z) \prod_{i=1}^n p(x_i|\mathrm{pa}(i),\indcia)$ for some $k \leqslant n$ \hspace{0.1em} (CIA)
\end{enumerate}
if and only if $p_D(x)=\prod_{i=1}^n p_D(x_i|\mathrm{pa}(i))$ \hspace{0.1em} (the data distribution is compatible with the CIA of (iii)).
\end{proposition}
\begin{proof}
If $p_D(x)=\prod_{i=1}^n p_D(x_i|\mathrm{pa}(i))$, then it is easy to check that $p(x,z)=p_D(x)p(z)$ satisfies the conditions (i)--(iii) with $k=0$.
To prove the other direction, assume (i)--(iii) hold.
Then (i) and (iii) imply that $p(x)=\prod_{i=1}^n p(x_i|\mathrm{pa}(i),\indcia)$ for all $z$.
Also, from (i) it follows that $(\textrm{pa}(i), x_i) \indep z$, and hence for all $z$, $p(x_i|\mathrm{pa}(i),z)=p(x_i,\mathrm{pa}(i)|z)/p(\mathrm{pa}(i)|z)=p(x_i|\mathrm{pa}(i))$, giving $p(x)=\prod_{i=1}^n p(x_i|\mathrm{pa}(i))$.
Finally, from (ii) we have that the two joints, $p_D(x)$ and $p(x)$ are the same, therefore their marginals and conditionals are the same too, which implies $p_D(x)=\prod_{i=1}^n p_D(x_i|\mathrm{pa}(i))$.
\end{proof}

The proposition says that just by specifying CIA for otherwise dependent parts of the data, any model that suffers posterior collapse (in the sense that $x$ and $z$ are independent) will be suboptimal in terms of the model evidence $p(x)$.
This gives a degree of assurance that given such a structure, a well-optimized model with high enough capacity will not suffer posterior collapse.
Latent variable image models, that model pixels or patches conditionally independently given the latents fall into this category, which explains while posterior collapse is not so prevalent in that case.

Conversely, in theory and in the absence of CIA, there is a trivial latent variable model $p(x,z)=p_D(x)p(z)$ which is optimal in terms of the marginal likelihood $p(x)$ but does not use the latents.
The prototypical example for this case is auto-regressive likelihoods, such as RNN language models.

To summarize, \emph{CIA must be made} to guarantee that latents are used given powerful enough models and inference methods.
On the other hand, lacking the necessary CIA, we can still bias solutions by changing the objective.
One such change to compensate for the lack of model structure is adding a constraint on mutual information.

\section{Mutual Information Augmented Objectives}
\label{sec:mi-objective}

Mutual information is often used as a measure of latent variable usage in trained models or as part of the training objective to control latent usage and reduce underspecification.
First, as a measure of latent usage, it is a diagnostic of the inference method.
In this role, it is but a proxy for the generalization ability of the model or for the performance on down-stream tasks.
Second, as a constraint during training, it can be seen as compensating for the lack of structure in the model.
However, its role is not essential in either of these: evaluating representations without the down-stream tasks is fraught with peril, and equipping the model with structure sounds a rather more appealing direction to pursue.
Still, finding a good model structure is easier said than done, and in practice mutual information is useful both as a diagnostic for inference and as a tool for model specification.
Thus, we augment the marginal likelihood objective with a mutual information term and maximize
\begin{align}
\label{eq:mi-obj}
\E_{p_D(x)}\ln p(x) + \lambda I_p(X,Z),
\end{align}
where $p_D(x)$ is the data distribution and $\lambda \in \Rb, \lambda
\geqslant 0$.
Throughout, we assume that $I_p(X,Z)$ is bounded, which is satisfied by any model for which the optimization of $p(x)$ is well-posed and hence its $p(x|z)$ is bounded for all $z$.
Also note that for any given $I_p(X,Z)$, the model achieves its global maximum when $p(x)=p_D(x)$, and we can reasonably expect models that suffer from posterior collapse to effectively balance data fit and mutual information.

Motivated by the identity $I_p(X,Z)=\E_{p(x)} \KL(p(z|x)\|p(z))$, the mutual information term can be estimated by the average KL divergence.
While this true KL is hard to compute in general due to the intractable posterior, the availability the variational posterior offers the compelling alternative of estimating $I_p(X,Z)$ by
\begin{align*}
I_{p,q}(X,Z) \coloneqq \E_{p_D(x)} \KL\big(q(z|x)\|p(z)\big).
\end{align*}
If we plan to use the representations obtained from the variational posterior $q(z|x)$ on some task, then $I_{p,q}$ is a natural quantity to track \citep{zhao2019infovae,rezaabad2020learning}.
However, in this work, our primary concern lies not with artifacts of variational inference but with the model $p(x,z)$ and its ability to capture information in the latents.
Moreover, employing $I_{p,q}$ as a proxy objective is problematic because it overestimates $I_p$ as $q$ tends to underestimate the variance of the true posterior.\footnote{For VAEs, this follows from the properties of the KL divergence, while for Monte-Carlo objectives in general, it follows from how the looseness of the lower bound relates to the variance of the estimator \citep{maddison2017filtering}.}
Even worse, with $I_{p,q}$ the quality of $q(z|x)$ would influence our conclusions about latent variable usage.
Monte-Carlo objectives, whose $q(z|x)$ in itself is no longer a direct approximation to $p(z|x)$ \citep{mnih2016variational}, exacerbate the problem with latent usage estimation.
Experimentally, we found that, for models trained with Monte-Carlo objectives, $I_{p,q}$ can wildly under- or overestimate $I_p$.
Thus, to form a better estimate of $I_p$, we replace $q(z|x)$ with $p(z|x)$ in $I_{p,q}$:
\begin{align}
\label{eq:cross-mi}
I_{p_D}^p(X,Z) \coloneqq \E_{p_D(x)} \KL\big(p(z|x)\|p(z)\big).
\end{align}
This ``cross'' mutual information $I_{p_D}^p$ is the average true KL over the data distribution $p_D$.
In the typical doubly stochastic optimization setting \citep{titsias2014doubly}, averaging over $p_D(x)$ instead of $p(x)$ allows us to estimate the mutual information based on the current minibatch without sampling from the model.
The price for efficiency is a possible generalization problem: the average KL may be different over $p_D(x)$ and $p(x)$.
Generalization may eventually become a pressing issue, but as our experiments will demonstrate (\S\ref{sec:experiments}), we first have to deal with underfitting.
In addition to being expedient, as we will show in \S\ref{sec:connection-to-the-representational-kl} and \S\ref{sec:connection-to-the-beta-vae}, this choice makes the single-sample case of our estimators correspond to the representational KL and the \betavae{} \citep{higgins2016beta}.


\begin{definition}[Mutual information augmented objective]
The mutual information augmented objective, which is a  combination of the usual marginal likelihood objective and $I_{p_D}^p$, is defined as
\begin{align}
\label{eq:cross-mi-obj}
\mathcal{O}(\lambda) = \E_{p_D(x)} \ln p(x) + \lambda I_{p_D}^p(X,Z).
\end{align}
Furthermore, the pointwise version of the objective is defined as
\begin{equation}
\label{eq:cross-mi-obj-x}
\mathcal{O}(\lambda,x) = \ln p(x) + \lambda \KL\big(p(z|x)\|p(z)\big),
\end{equation}
which in turn satisfies $\mathcal{O}(\lambda) = \E_{p_D(x)} \mathcal{O}(\lambda,x)$.
\end{definition}
In the following, we propose estimators of $\mathcal{O}(\lambda,x)$ and the true KL within it to estimate $\mathcal{O}(\lambda)$ and the mutual information in a manner suitable for the doubly stochastic optimization setting.

\subsection{The KL Objective}

To find the maximum of $\mathcal{O}(\lambda)$ in $\theta$, both of its terms must be estimated well.
We delegate the task of estimating the marginal log-likelihood $\ln p(x)$ to a ``base'' Monte-Carlo estimator of the form $\hat{S}^K(x,z_{1:K}) = \ln\left( \tfrac{1}{K} \sum_{i=1}^K f(x, z_i)\right)$, where $z_{1:K}=(z_1, \ldots, z_K)$ are independent samples from the proposal distribution $q(z|x)$, and $f$ is some function of the observable and latent variables \citep{mnih2016variational}.
Ideally, $\hat{S}^K$ is chosen to have low bias and low variance, allowing optimization to strike a better balance with mutual information.
For our first contribution, the KL objective, we rewrite the true KL in a form more amenable to importance sampling:
\begin{align*}
\KL\big(p(z|x)\|p(z)\big)
&= \E_{p(z|x)} \ln \frac{p(z|x)}{p(z)}
 = \E_{p(z|x)} \big[ \ln p(x|z) \big] - \ln p(x)\\
&= \E_{q(z|x)} \bigg[ \frac{p(z|x)}{q(z|x)} \ln p(x|z) \bigg] - \ln p(x)\\
&= \frac{1}{p(x)} \E_{q(z|x)} \bigg[ \frac{p(x,z)}{q(z|x)} \ln p(x|z) \bigg]
   - \ln p(x).
\end{align*}
Plugging this into the definition of $\mathcal{O}(\lambda,x)$ in \eqref{eq:cross-mi-obj-x} and grouping the $\ln p(x)$ terms, which can be estimated with the base Monte-Carlo estimator $\hat{S}^K$, we get
\begin{align*}
\mathcal{O}(\lambda,x)
&= (1-\lambda)\ln p(x) +
   \lambda \frac{1}{p(x)} \E_{q(z|x)} \frac{p(x,z)}{q(z|x)} \ln p(x|z).
\end{align*}
This leaves the task of estimating the second term
\begin{align*}
\numberthis\label{eq:kl-plus-is}
\frac{1}{p(x)} \E_{q(z|x)} \frac{p(x,z)}{q(z|x)} \ln p(x|z),
\end{align*}
in which $p(x)$ can be estimated with a simple $K$-sample importance sampling estimator
\begin{align}
\label{eq:p-hat-k}
\hat{p}^K(x, z_{1:K}) \coloneqq \frac{1}{K} \sum_{i=1}^K \frac{p(x,z_i)}{q(z_i|x)}.
\end{align}
However, combining $\hat{p}^K$ with an importance sampling estimate of the expectation in \eqref{eq:kl-plus-is} using different samples begot very high variance in preliminary experiments.
Instead, we approximate \eqref{eq:kl-plus-is} with the self-normalized importance sampling estimator
\begin{align*}
\numberthis\label{eq:kl-estimator}
\hat{U}^K(x,z_{1:K})
 \coloneqq \hat{p}^K(x, z_{1:K})^{-1}
           \frac{1}{K} \sum_{i=1}^K \frac{p(x,z_i)}{q(z_i|x)} \ln p(x|z_i),
\end{align*}
which uses the same samples for estimating $1/p(x)$ and the expectation.
This leads to our first estimator for $\mathcal{O}(\lambda,x)$.

\begin{definition}[KL objective]
Let $\hat{S}^K$ be any $K$-sample Monte-Carlo estimator of $\ln p(x)$.
Then the augmented objective $\mathcal{O}(\lambda,x)$ can be estimated by
the KL objective
\begin{align}
\label{eq:kl-obj}
\mathcal{O}_{\KL}(\hat{S}, K, \lambda,x)
= \E_{z_{1:K}\sim q(z|x)} \hat{O}_{\KL}(\hat{S}, K, \lambda, x, z_{1:K}),
\end{align}
where
\begin{align}
\label{eq:kl-obj-estimator}
\mathcal{\hat{O}}_{\KL}(\hat{S}, K, \lambda, x, z_{1:K})
&= (1-\lambda)\hat{S}^K(x,z_{1:K}) + \lambda \hat{U}^K(x,z_{1:K}).
\end{align}
\end{definition}
Note that $\mathcal{\hat{O}}_{\KL}(\hat{S}^K, K, \lambda, x, z_{1:K})$ uses the same samples $z_{1:K} \sim q(z|x)$ to estimate both terms of the augmented objective \eqref{eq:cross-mi-obj-x}.
Grouping the terms differently, we can separate out the estimate of the KL:
\begin{align}
\label{eq:kl-obj-estimator-alt}
\mathcal{\hat{O}}_{\KL}(\hat{S}, K, \lambda, x, z_{1:K})
&= \underbrace{\hat{S}^K(x,z_{1:K})}_{\approx \ln p(x)} +
   \lambda \underbrace{\big( \hat{U}^K(x,z_{1:K}) - \hat{S}^K(x,z_{1:K}) \big)}_{\approx \KL(p(z|x)\|p(z))}.
\end{align}
We assume throughout that the estimators $\hat{S}^K(x,z_{1:K})$ and $\hat{U}^K(x,z_{1:K})$ have finite variance.
This implies the following properties of the estimators $\mathcal{\hat{O}}_{\KL}(\hat{S}^K, K, \lambda, x)$ of $\mathcal{O}(\lambda, x)$ and $\hat{U}^K(x,z_{1:K}) - \hat{S}^K(x,z_{1:K})$ of $\KL(p(z|x)\|p(z))$:

\begin{proposition}[properties of the KL objective]
\label{prop:properties of the KL objective} \mbox{}
\begin{enumerate}[(i)]
\item If $\hat{S}^K$ converges in probability or almost surely to $\ln p(x)$ as $K \to \infty$, then so do $\mathcal{\hat{O}}_{\KL}$ and $\hat{U}^K - \hat{S}^K$ to
$\mathcal{O}(\lambda, x)$ and $\KL(p(z|x)\|p(z))$, respectively.
\item If $\E_{z_{1:K} \sim q(z|x)}\hat{S}^K(x,z_{1:K}) \leqslant \ln p(x)$, then $\E_{z_{1:K} \sim q(z|x)}[\hat{U}^K(x,z_{1:K}) - \hat{S}^K(x,z_{1:K})] \geqslant \KL(p(z|x)\|p(z))$, that is, the estimator is biased upward.
\item The bias of the self-normalized importance sampling estimator $\hat{U}^K$ is bounded if $p(x,z)$ is bounded as shown in Proposition 7 of \citet{metelli2020importance}.
\item The variance of $\hat{U}^K$ decays with $K$, but unlike in non-normalized importance sampling, for any given $K$, there is no proposal distribution with which the variance is zero unless $\hat{U}^K$ is constant for all $z_{1:K}$ with probability 1.
\end{enumerate}
\end{proposition}
\vspace{-1em}
\begin{proof}
These follow from  the properties of self-normalized importance sampling \citep{art2013mcbook} except where noted.
\end{proof}

Note that although the objective is biased upwards, its bias is decreased with more samples and is bounded if $p(x,z)$ is bounded.
It can be assumed that $p(x,z)$ is bounded, else the optimization of $\ln p(x)$ is ill-posed even without the mutual information term.
Consequently we may be able to rely on $\lambda$ to counteract the bias thus bounded.
Importantly, the computation of $\hat{U}^K$ imposes minimal overhead as it needs to evaluate only $p(x|z_i)$, $p(z_i)$ and $q(z_i|x)$: the same quantities and same $z_i$ as needed for computing $\hat{S}^K$.
Referring back to \eqref{eq:mi-obj}, we argue that this KL objective allows for effective interpolation between fitting the data and capturing information in the latent variables.

\subsection{The Rényi Objective}
\label{sec:renyi-objective}

In this section, we introduce a second estimator of the augmented objective $\mathcal{O}(\lambda, x)$ \eqref{eq:cross-mi-obj-x}, based on the Rényi divergence, to address a potential issue with the KL objective's estimate \eqref{eq:kl-obj-estimator}.
This issue lies in the fact that the KL objective linearly combines two estimators ($\hat{S}^K$ and $\hat{U}^K$) of different quantities.
How their biases and variances relate deserves some consideration.
As we have seen, with $\hat{S}^K$ that underestimate $\ln p(x)$ (e.g. IWAE), $\hat{U}^K-\hat{S}^K$ overestimates the true KL.
Luckily, both biases can be reduced with more samples.

However, taking more samples may not help if the two estimators have very different variances, in which case optimizing the objective may be difficult.
This issue could be addressed by designing a $\hat{U}^K$ for every $\hat{S}^K$, but this would limit the applicability of our method in practice.
Instead, we apply $\hat{S}$ not only to estimate $\ln p(x)$ but \emph{a second time} too to estimate the Rényi divergence, itself a biased estimate of the true KL.
As we will see later, this works surprisingly well in practice despite the presence of the bias.

The Rényi divergence between two distributions $f(x)$ and $g(x)$ is defined as $D_\alpha(f\|g)=\frac{1}{\alpha-1}\ln \E_{g(x)} f(x)^\alpha g(x)^{-\alpha}$, where $\alpha$ is a positive real number.
For $\alpha<1$, $D_\alpha(f\|g) \leqslant \KL(f\|g)$, while for $\alpha>1$, $D_\alpha(f\|g) \geqslant \KL(f\|g)$.
Since $\lim_{\alpha \to 1} D_\alpha(f\|g) = \KL(f\|g)$,  $D_\alpha(f\|g)$ can approximate $\KL(f\|g)$ arbitrarily closely when $\alpha$ is sufficiently close to $1$.
This latter property motivates the use of $D_\alpha(p(z|x)\|p(z))$ as an approximation to the true KL.
To construct an estimator, we first rewrite the Rényi divergence as:
\begin{align*}
(\alpha-1)D_\alpha\big(p(z|x)\|p(z)\big)
&= \ln \E_{p(z)} \frac{p(z|x)^\alpha}{p(z)^\alpha}\\
&= \ln \E_{p(z)} \bigg[ \frac{p(z|x)^\alpha p(x)^\alpha}{p(z)^\alpha} \bigg]
   - \alpha \ln p(x)\\
&= \ln \E_{p(z)} \big[ p(x|z)^\alpha \big] - \alpha \ln p(x)\\
\numberthis\label{eq:renyi-rewrite}
&= \ln p^\alpha(x) - \alpha \ln p(x),
\end{align*}
where $p^\alpha(x)\coloneqq\E_{p(z)} p(x|z)^\alpha$.
We note in passing that $p^\alpha(x)$ has an intuitive interpretation, particularly when $\alpha \in \Nb$.
Consider the task of modelling the distribution of discrete data over some discrete set $\cX$ duplicated $\alpha$ times, that is $p^\alpha_D(x, \ldots, x) \coloneqq p_D(x)$.
That is, $p^\alpha_D$ is a probability distribution over $\cX^\alpha$, whereas $p_D$ is over $\cX$, and $p^\alpha_D(x_1, \ldots, x_\alpha)=0$ unless all $x_i$ are identical.
On this ``new'' task, $\alpha \ln p(x) = \ln p(x)^\alpha$ acts as the uninformed baseline, in which a separate set of latents is used for each branch $p(x_i|z)$, thus the cost of information in the latents must be paid $\alpha$ times.
This means that, based on its alternative form $\ln p^\alpha(x) - \alpha \ln p(x)$ in \eqref{eq:renyi-rewrite}, we can interpret the Rényi divergence as a measure of how much better the $\alpha$-duplicated model $p^\alpha(x)$ does at modelling the duplicated data compared to the worst-case solution $p(x)^\alpha$, which does not use the latents to amortize the cost of encoding the data multiple times.

We now derive a biased approximation to $\mathcal{O}(\lambda,x)$:
\begin{align*}
\mathcal{O}(\lambda,x)
&= \ln p(x) + \lambda \KL\big(p(z|x)\|p(z)\big)\\
&\approx \ln p(x) + \lambda D_\alpha\big(p(z|x)\|p(z)\big)\\
&= \ln p(x) + \frac{\lambda}{\alpha-1} \big(\ln p^\alpha(x) - \alpha \ln p(x)\big)\\
&= \frac{\lambda}{\alpha-1} \ln p^\alpha(x)
   - \bigg(\frac{\lambda\alpha}{\alpha-1}-1\bigg) \ln p(x).
\end{align*}

\begin{definition}[Rényi objective]
Let $\lambda, \alpha>0$, and let $\hat{S}^K(x,z_{1:K})$ and $\hat{S}^K_\alpha(x,z_{1:K})$ be $K$-sample Monte-Carlo estimators for $\ln p(x)$ and $\ln p^\alpha(x)$, respectively.
Then the augmented objective $\mathcal{O}(\lambda,x)$ can be estimated by
the Rényi objective
\begin{align}
\label{eq:renyi-obj}
\mathcal{O}_{R}(\hat{S}, K, \lambda, \alpha, x)
= \E_{z_{1:K}\sim q(z|x)} \mathcal{\hat{O}}_{R}(\hat{S}, K, \lambda, \alpha, x, z_{1:K}),
\end{align}
where
\begin{align}
\label{eq:renyi-obj-estimator}
\mathcal{\hat{O}}_{R}(\hat{S}, K, \lambda, \alpha, x, z_{1:K})
&= \frac{\lambda}{\alpha-1} \hat{S}^K_\alpha(x,z_{1:K})
   - \bigg(\frac{\lambda\alpha}{\alpha-1}-1\bigg) \hat{S}^K(x,z_{1:K}).
\end{align}
\end{definition}
Separating out the estimate of the Rényi divergence yields the alternative form
\begin{align}
\label{eq:renyi-obj-estimator-alt}
\mathcal{\hat{O}}_{R}(\hat{S}, K, \lambda, \alpha, x, z_{1:K})
&= \underbrace{\hat{S}^K(x, z_{1:K})}_{\approx \ln p(x)}
   + \lambda \underbrace{ \bigg( \frac{1}{\alpha-1}
                                 \Big( \hat{S}^K_\alpha(x,z_{1:K}) -
                                       \alpha \hat{S}^K(x,z_{1:K}) \Big)
                          \bigg)
                        }_{\approx D_\alpha(p(z|x)\|p(z))}.
\end{align}

Our goal was to address the mismatched biases and variances of the KL objective's $\hat{S}^K$ and $\hat{U}^K$.
Having eliminated $\hat{U}^K$, we are left with only $\hat{S}^K$ and $\hat{S}^K_\alpha$, estimating two closely related quantities, $\ln p(x)$ and $\ln p^\alpha(x)$.
Note that $\hat{S}^K_\alpha$ can be obtained with a slight modification of $\hat{S}^K$, as explained in the next section.
In \S\ref{sec:experiments}, we validate experimentally that the benefits this scheme affords outweigh the obvious drawback of additional bias in the Rényi objective.

\subsubsection{Estimating \texorpdfstring{$\mathbf{p^\alpha(x)}$}{palpha(x)} with IWAE}

Since $\ln p(x) = \ln\left(\E_{p(z)} p(x|z)\right)$ and $\ln p^\alpha(x) = \ln \left(\E_{p(z)} p(x|z)^\alpha\right)$, if $\hat{S}^K(x,z_{1:K})$ is computed explicitly in terms of $p(x|z)$, then $\hat{S}^K_\alpha(x,z_{1:K})$ can be derived simply by replacing $p(x|z)$ with $p(x|z)^\alpha$ in $\hat{S}^K(x, z_{1:K})$.
All base estimators considered in this paper have this property and we elucidate the derivation through the example of the IWAE.

The ELBO of variational autoencoders can be rather loose, and the variance of the resulting approximate posterior is usually smaller than that of the true posterior.
To tackle these issues, \cite{burda2015importance} proposed the importance weighted autoencoder (IWAE).
Whereas the ELBO is single-sample, the IWAE bound is based on $K \in \Nb$ samples:
\begin{align*}
\ln p(x)
= \ln \E_{p(z)} p(x|z)
&= \ln \E_{\substack{z_1,\ldots, z_K\\ \sim q(z|x)}}
  \frac{1}{K} \sum_{i=1}^K \frac{p(x|z_i)p(z_i)}{q(z_i|x)}\\
&\geqslant \E_{\substack{z_1,\ldots, z_K\\ \sim q(z|x)}}
  \ln \frac{1}{K} \sum_{i=1}^K \frac{p(x|z_i)p(z_i)}{q(z_i|x)}.
\end{align*}

In importance sampling terms, $p(x|z)$ is the integrand, the function whose expectation we want to compute with respect to $p(z)$.
While $p(x|z)$ here is a probability mass function, importance sampling does not require this.
In fact, we can estimate the expectation of $p(x|z)^\alpha$ analogously:
\begin{align*}
\numberthis
\label{eq:power-iwae-bound}
\ln p^\alpha(x)
&\geqslant \E_{\substack{z_1,\ldots, z_K\\ \sim q(z|x)}}
  \ln \frac{1}{K} \sum_{i=1}^K \frac{p(x|z_i)^\alpha p(z_i)}{q(z_i|x)}.
\end{align*}
To estimate the expectation in the above lower bound, we can recombine quantities already computed for the base estimator with a negligible computational overhead.

Note that there is an optimal proposal distribution  $q_{\text{opt}}(z|x)=p(x|z)^\alpha p(z) / p^\alpha(x)$ that leads to exactly computing $\ln p^\alpha(x)$ (i.e. estimating it with zero bias and variance).
On the other hand, the other term in the Rényi objective is best estimated using, in general, a different proposal distribution.
One solution would be to apply separate proposal distributions, another is to choose $\lambda$ and $\alpha$ such that in $\ln p(x) + \lambda D_\alpha$ the $\ln p(x)$ term is cancelled out, which we explore briefly in \S\ref{sec:power-objective} below.

\subsubsection{Other Estimators}

In the interest of space, we omit re-derivations of further estimators of $\ln p^\alpha(x)$ and their gradient estimators.
As in the IWAE case, all we need to show is that the estimators do not depend on properties of $p(x|z)$ that $p^\alpha(x|z)$ doesn't have.
This is true for the estimators used in our experiments: REINFORCE \citep{williams1987class,mnih2014neural}, VIMCO \citep{mnih2016variational}, STL \citep{roeder2017sticking} and DReG \citep{tucker2018doubly}.

\subsection{The Power Objective}
\label{sec:power-objective}

Notice that, in the Rényi estimator \eqref{eq:renyi-obj-estimator}, if $\lambda$ is set to $(\alpha-1)/\alpha$, then the coefficient of the $\ln p(x)$ term becomes zero, yielding
\begin{align*}
\numberthis\label{eq:power-obj}
\mathcal{\hat{O}}_P(\hat{S}, K, \alpha, x)
= \alpha^{-1} \hat{S}^K_\alpha(x,z_{1:K}),
\end{align*}
which we call the \emph{power objective}.
Note that for optimization the constant $\alpha^{-1}$ can be dropped from the objective and that if $\alpha=1$, the power objective is equal to the log-likelihood $\ln p(x)$.
Next we show that if $\alpha>1$, maximizing $p^\alpha(x)$ optimizes a lower bound on $p(x)$, and this lower bound is tight when the latents fully determine the observables.
\begin{restatable}{proposition}{proppowerlowerbound}
\label{prop:power-lower-bound}
Assume that $X$ is concentrated on the countable set $\cX$, $\alpha>1$, and let $x \in \cX$ be arbitrary. Then $p^\alpha(x) \leqslant p(x)$ with equality for all $x \in \cX$ if and only if $H_p(X|Z)=0$.
Furthermore, $(p(x))^\alpha \leqslant p^\alpha(x)$ and $\ln p^\alpha(x) - \alpha \ln p(x) = (\alpha-1)D_\alpha(p(z|x)\|p(z))$.
\end{restatable}
\begin{proof}
Since $X$ is countable, $p(x|z)$ is discrete, hence $p(x|z) \leqslant 1$ for all $x$ and $z$. Therefore, since $\alpha>1$, $p(x|z)^\alpha \leqslant p(x|z)$ with equality if and only if $p(x|z) \in \{0,1\}$. Thus,
$p^\alpha(x)
= \E_{p(z)} p(x|z)^\alpha \leqslant \E_{p(z)} p(x|z)
= p(x)$,
with equality if and only if $p(x|z) \in \{0,1\}$ for all $x$, for almost all $z$. The latter holds if and only if $X$ is a deterministic function of $Z$ with probability 1, which is equivalent to $H_p(X|Z)=0$.
The second statement of is a direct consequence of \eqref{eq:renyi-rewrite} and $\alpha-1>0$.
\end{proof}

So the power objective with $\alpha>1$ is an upper bound on the KL (and $I_{p_D}^p$ when averaged over $x$), but this upper bound nevertheless becomes tight when the latents determine the observable variables (i.e. $H_p(X|Z)=0$).

In summary, the power objective has three important differences from the Rényi objective, of which it is a special case.
First, with $\alpha>1$, it is guaranteed to be a lower bound on $\ln p(x)$.
Second, since there is only a single quantity, $\ln p^\alpha(x)$, being estimated, the question of using separate $q(z|x)$ proposal distributions to estimate the different terms does not arise.
Third, $\lambda$ and $\alpha$ are tied, so it may be harder to find a good balance between ease of optimization and low bias.

\section{Connection to the Representational KL}
\label{sec:connection-to-the-representational-kl}

\begin{proposition}[Single-sample KL estimate]
\label{prop:single-sample-kl-estimate}
In the single-sample case with $\hat{S}^1(x, z_1) = \ln \frac{p(x,z_1)}{q(z_1|x)}$, the KL objective's estimate of the true KL in \eqref{eq:kl-obj-estimator-alt} is the representational KL in expectation:
\begin{align*}
\E_{z_1 \sim q(z|x)} \big[ \hat{U}^1(x,z_1) - \hat{S}^1(x,z_1) \big] = \KL\big(q(z|x)\|p(z)\big).
\end{align*}
\end{proposition}
\vspace{-1em}
\begin{proof}
With $\hat{p}^1(x, z_1) = p(x,z_1)/q(z_1|x)$ from \eqref{eq:p-hat-k}, we have that
\begin{align*}
\hat{U}^1(x,z_1) - \hat{S}^1(x,z_1)
&= \hat{p}^1(x, z_1)^{-1} \frac{p(x,z_1)}{q(z_1|x)} \ln p(x|z_1) -
   \ln \frac{p(x,z_1)}{q(z_1|x)} \\
&= \ln p(x|z_1) - \ln \frac{p(x,z_1)}{q(z_1|x)}
= \ln \frac{q(z_1|x)}{p(z_1)}
\end{align*}
\end{proof}
So not only are the expectations the same, but our single-sample estimate is the same as the trivial single-sample estimate of the representational KL.
We now prove a similar result for the Rényi objective.
\begin{proposition}[Single-sample Rényi estimate]
\label{prop:single-sample-renyi-estimate}
In the single-sample case with $\hat{S}^1(x, z_1) = \ln \frac{p(x,z_1)}{q(z_1|x)}$, the Rényi objective's estimate of $D_\alpha(p(z|x)\|p(z))$ in \eqref{eq:renyi-obj-estimator-alt} is the representational KL in expectation:
\begin{align*}
\E_{z_1 \sim q(z|x)} \bigg[ \frac{1}{\alpha-1}
                                 \Big( \hat{S}^1_\alpha(x,z_1) -
                                       \alpha \hat{S}^1(x,z_1) \Big)
                 \bigg]
= \KL\big(q(z|x)\|p(z)\big).
\end{align*}
\end{proposition}
\vspace{-1em}
\begin{proof}
\begin{align*}
\frac{1}{\alpha-1} \bigg( \hat{S}^1_\alpha(x,z_1) - \alpha \hat{S}^1(x,z_1) \bigg)
&= \frac{1}{\alpha-1}\bigg(\ln \frac{p(x|z_1)^\alpha p(z_1)}{q(z_1|x)} - \alpha \ln \frac{p(x,z_1)}{q(z_1|x)}\bigg)\\
&= \frac{1}{\alpha-1}\bigg(\ln \frac{p(z_1)}{q(z_1|x)} - \alpha \ln \frac{p(z_1)}{q(z_1|x)}\bigg)\\
&= \ln \frac{q(z_1|x)}{p(z_1)}
\end{align*}
\end{proof}

In the single-sample case, where $q(z|x)$ approximates the true posterior much better compared to the multi-sample Monte-Carlo estimators, the representational KL approximates the true KL with a relatively small bias.
This bias is, however, asymptotically eliminated by our KL estimator as the number of samples $K$ grows (and reduced for our Rényi estimator), which significantly improves the usage of the latent variables.
This is demonstrated in our experiments (e.g. \Cref{fig:synthetic-dreg-samples,fig:synthetic-vimco-samples}) where we compare methods using KL estimates based on different number of samples from $q(z|x)$.

\section{Connection to the \texorpdfstring{$\boldsymbol{\beta}$}{Beta}-VAE}
\label{sec:connection-to-the-beta-vae}

We have seen that in the single-sample case our estimators are equal to the representational KL in expectation.
Still in the single-sample case, we show that the \betavae{} \citep{higgins2016beta} objective defined as
\begin{align}
\label{eq:beta-vae-obj}
\mathcal{O}_{\beta\text{-VAE}}(\beta, x)
= \E_{z_1 \sim q(z|x)} \big[ \ln p(x|z_1) \big] - \beta \KL\big(q(z|x)\|p(z)\big)
\end{align}
is equal to the KL, Rényi and power objectives with suitably chosen parameters.

\begin{proposition}[\betavae{} equivalence]
In the single-sample case with $\hat{S}^1(x, z_1) = \ln \frac{p(x,z_1)}{q(z_1|x)}$, the \betavae{} and our single-sample objectives are equal with a suitable choice of $\lambda$ or $\alpha$:
\begin{align*}
\mathcal{O}_{\beta\text{-VAE}}(\beta, x)
&= \mathcal{O}_{\KL}(\hat{S}, 1, 1-\beta, x)
 & \qquad \text{($\lambda=1-\beta$)}\\
&= \mathcal{O}_R(\hat{S}, 1, 1-\beta, \alpha, x)
 & \qquad \text{($\lambda=1-\beta$)}\\
&= \mathcal{O}_P(\hat{S}, 1, \beta^{-1}, x)
 & \qquad \text{($\alpha=\beta^{-1}$)}.
\end{align*}
\end{proposition}
\vspace{-1em}
\begin{proof}
We prove the KL objective case from \eqref{eq:kl-obj-estimator-alt} and \Cref{prop:single-sample-kl-estimate}:
\begin{align*}
\E_{z_1 \sim q(z|x)} \mathcal{\hat{O}}_{\KL}(\hat{S}, 1, \lambda, x, z_1)
&= \E_{z_1 \sim q(z|x)} \big[ \hat{S}^1(x,z_1) \big] +
   \lambda \E_{z_1 \sim q(z|x)} \big[ \hat{U}^1(x,z_1) - \hat{S}^1(x,z_1) \big]\\
&= \E_{z_1 \sim q(z|x)} \bigg[ \ln \frac{p(x,z_1)}{q(z_1|x)} \bigg]
    + \lambda \KL\big(q(z|x)\|p(z)\big)\\
&= \E_{z_1 \sim q(z|x)} \big[ \ln p(x|z_1) \big]
   + (\lambda - 1) \KL\big(q(z|x)\|p(z)\big).
\end{align*}
The proof for the Rényi case follows a similar route, starting from \eqref{eq:renyi-obj-estimator-alt} and using \Cref{prop:single-sample-renyi-estimate}.
Finally, for the power objective, $\lambda = (\alpha-1)/\alpha$, so $\alpha=\beta^{-1}$ recovers $\lambda = 1-\beta$.
\end{proof}

\section{Experiments}
\label{sec:experiments}

The goal of our experiments is to demonstrate the difficulty of inference with VAEs and Monte-Carlo objectives and to evaluate the proposed methods.
Our results indicate the presence of a severe tradeoff between data fit and latent variable usage.
We emphasize that, for progress to be made, the choice of evaluation method must acknowledge the existence of this tradeoff.
In this work, evaluation is performed in terms of Pareto frontiers of data likelihood vs latent usage curves; reporting a single, best data likelihood would always pick the point with zero latent usage.
Results with our proposed estimators, either with continuous latents and DReG or discrete latents and the VIMCO base estimator, markedly improve on their baselines, which do not have multiple samples or cannot use them as efficiently.
The improvement is especially significant with discrete latents.

Instead of trying to improve the predictive performance directly, first we demonstrate the difficulty of inference on a simple, synthetic data set and a model to which posterior collapse comes easy.
These experiments focus exclusively on data fit to highlight the tradeoff against latent variable usage.
After the experiments on synthetic data, we move on to language modelling with recurrent networks, a very hostile application for VAEs \citep{bowman2015generating}.

\subsection{Experiments with Synthetic Data}
\label{sec:experiments-with-synthetic-data}

Every data point is a single symbol drawn from a discrete uniform distribution over a vocabulary of 10000 symbols.
Thus, the optimal solution is to assign probability 1/10000 to each symbol, which can be trivially satisfied by a simple model that ignores the latents.
Note that $\lambda>0$ makes such solutions suboptimal.
In fact, for all priors such that $H(Z)\geqslant H(X)$, the optimal solution must capture all information in the latents (i.e. $I(X,Z)=H(X)$).
Our goal with these experiments is to demonstrate the difficulty of fitting the data while using the latents.

\subsubsection{Model Architecture}

The encoder implementing $q$ assigns an embedding \citep{mikolov2013distributed} to each word in the vocabulary, feeds the embedding to a two-layer, densely connected neural network with $\tanh$ non-linearities.
For continuous latents, denoting the output of the last encoder layer for a given $x$ with $o$, $q(z|x)$ follows an isotropic normal distribution $\N(f(o),\diag(\exp(g(o))))$, where $f$ and $g$ are affine transformations parameterized as densely connected neural network layers.
For categorical latents, $q(z|x)$ is proportional to $\exp(o_z)$ (i.e. it is $Categorical(\softmax(o))$).
The decoder is a neural network similar to the encoder but with a final softmax layer.
In the decoder, values of continuous latents are fed directly as input to the first layer, but values of categorical latents are first embedded.
To compensate for this discrepancy, decoders with categorical latents have only one hidden layer.
All embeddings and hidden layers are of size 128.

\subsubsection{Evaluation Methodology}

Since there is no single number to summarize the tradeoff between latent usage and the quality of the model, we plot the model's latent usage and the expected negative log-likelihood (NLL).
The NLL of optimal solutions is the entropy of the data distribution, $\ln 10000\approx 9.210$.
The models' NLLs are estimated using IWAE with 100 samples.
We quantify latent variable usage as the mutual information $I_p(X,Z)$, estimated as the average $\hat{U}^K(x, z_{1:K}) - \ln \hat{p}^K(x, z_{1:K})$ \eqref{eq:kl-estimator} with $K=100$ samples.
We validated empirically that the variance of these estimates is small ($<0.01$) over all feasible latent usage values.
Zero latent usage corresponds to posterior collapse.

These plots (like Figure 3b in \citealt{alemi2017fixing}) carry the same information as rate-distortion curves, which can be recovered by subtracting the rate $I_p$ from the NLL.
Only the measurements on the Pareto frontier are shown, and we tune hyperparemeters, such as the learning rate and $\lambda$, of the augmented objective \eqref{eq:cross-mi-obj} to push the Pareto frontier towards more efficient latent usage.
See Appendix~\ref{sec:optimization-settings} for details.

\subsubsection{Overview of Results}

Results for continuous and discrete latents are presented in separate sections, following a similar progression:
\begin{itemize}
\item We first show that the choice of base estimator matters: the variance issues of IWAE and REINFORCE limit the potential of the proposed method.
\item Next, with DReG and VIMCO we present clear improvements over the single-sample baselines.
\item Finally, as an ablation study, we verify that using multiple samples in both the marginal likelihood and the mutual information terms of the objective is necessary for best performance.
\end{itemize}
We also note that without augmenting the objective with a mutual information term, our model tends to fit the data perfectly with negligible latent usage.
These degenerate curves consisting of a single point are omitted from the plots.

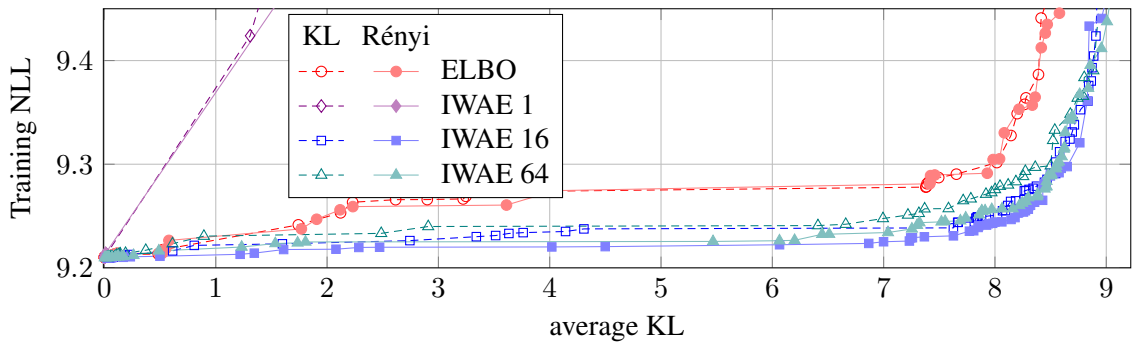
\begin{figure}
\begin{tikzpicture}
\begin{axis}[xmin=-0.01,xmax=9.22,ymin=9.20,ymax=9.45,
  xlabel=average KL, ylabel=Training NLL,
  height=0.33*\textwidth,
  width=\textwidth,
  grid=major,
  legend pos=north west,
  legend columns=2,
  legend style={
    font=\mystrut,
    legend cell align=left,
    at={(0.18,0.98)},
    anchor=north west,
  },
]

\addlegendimage{legend image with text=KL}
\addlegendentry{}
\addlegendimage{legend image with text=Rényi}
\addlegendentry{}

\pgfplotstablesort[sort key=validklp]{\sorted}
                  {data/synthetic/l40x0-bs256-hs128-t40-vae-kl}
\addplot[redhollow] table [x=validklp, y=validnll] {\sorted};
\addlegendentry{};

\pgfplotstablesort[sort key=validklp]{\sorted}
                  {data/synthetic/l40x0-bs256-hs128-t40-vae-renyi}
\addplot[redsolid] table [x=validklp, y=validnll] {\sorted};
\addlegendentry{ELBO};

\pgfplotstablesort[sort key=validklp]{\sorted}
                  {data/synthetic/l40x0-bs256-hs128-t40-iwae1-kl}
\addplot[violethollow] table [x=validklp, y=validnll] {\sorted};
\addlegendentry{};

\pgfplotstablesort[sort key=validklp]{\sorted}
                  {data/synthetic/l40x0-bs256-hs128-t40-iwae1-renyi}
\addplot[violetsolid] table [x=validklp, y=validnll] {\sorted};
\addlegendentry{IWAE 1};

\pgfplotstablesort[sort key=validklp]{\sorted}
                  {data/synthetic/l40x0-bs256-hs128-t40-iwae16-kl};
\addplot[bluehollow] table [x=validklp, y=validnll] {\sorted};
\addlegendentry{};

\pgfplotstablesort[sort key=validklp]{\sorted}
                  {data/synthetic/l40x0-bs256-hs128-t40-iwae16-renyi}
\addplot[bluesolid] table [x=validklp, y=validnll] {\sorted};
\addlegendentry{IWAE 16};

\pgfplotstablesort[sort key=validklp]{\sorted}
                  {data/synthetic/l40x0-bs256-hs128-t40-iwae64-kl};
\addplot[tealhollow] table [x=validklp, y=validnll] {\sorted};
\addlegendentry{};

\pgfplotstablesort[sort key=validklp]{\sorted}
                  {data/synthetic/l40x0-bs256-hs128-t40-iwae64-renyi};
\addplot[tealsolid] table [x=validklp, y=validnll] {\sorted};
\addlegendentry{IWAE 64};

\end{axis}
\end{tikzpicture}
\caption{\small KL and Rényi objectives (empty and full markers) on continuous synthetic data with base estimators ELBO and IWAE N, where N is the number of samples used for estimating both $p(x)$ and $\KL(p(z|x)\|p(z))$.
Here and in the following plots, without augmenting the objective with the mutual information term, models tend to fit the data perfectly with negligible latent usage.
These degenerate curves consisting of a single point are omitted from the plots.
Also, with base estimators ELBO and IWAE 1, our estimators are equivalent to a \betavae{} with analytical and estimated KL (see \S\ref{sec:connection-to-the-beta-vae}), respectively.}
\label{fig:synthetic-iwae}
\end{figure}

\begin{figure}
\begin{tikzpicture}
\begin{axis}[xmin=-0.01,xmax=9.22,ymin=9.20,ymax=9.45,
  xlabel=average KL, ylabel=Training NLL,
  height=0.33*\textwidth,
  width=\textwidth,
  grid=major,
  legend pos=north west,
  legend columns=2,
  legend style={
    font=\mystrut,
    legend cell align=left,
    at={(0.18,0.98)},
    anchor=north west,
  },
]

\addlegendimage{legend image with text=KL}
\addlegendentry{}
\addlegendimage{legend image with text=Rényi}
\addlegendentry{}

\pgfplotstablesort[sort key=validklp]{\sorted}
                  {data/synthetic/l40x0-bs256-hs128-t40-vae-kl}
\addplot[redhollow] table [x=validklp, y=validnll] {\sorted};
\addlegendentry{};

\pgfplotstablesort[sort key=validklp]{\sorted}
                  {data/synthetic/l40x0-bs256-hs128-t40-vae-renyi}
\addplot[redsolid] table [x=validklp, y=validnll] {\sorted};
\addlegendentry{ELBO};

\pgfplotstablesort[sort key=validklp]{\sorted}
                  {data/synthetic/l40x0-bs256-hs128-t40-dreg1-kl}
\addplot[violethollow] table [x=validklp, y=validnll] {\sorted};
\addlegendentry{};

\pgfplotstablesort[sort key=validklp]{\sorted}
                  {data/synthetic/l40x0-bs256-hs128-t40-dreg1-renyi}
\addplot[violetsolid] table [x=validklp, y=validnll] {\sorted};
\addlegendentry{DReG 1};

\pgfplotstablesort[sort key=validklp]{\sorted}
                  {data/synthetic/l40x0-bs256-hs128-t40-dreg16-kl};
\addplot[bluehollow] table [x=validklp, y=validnll] {\sorted};
\addlegendentry{};

\pgfplotstablesort[sort key=validklp]{\sorted}
                  {data/synthetic/l40x0-bs256-hs128-t40-dreg16-renyi}
\addplot[bluesolid] table [x=validklp, y=validnll] {\sorted};
\addlegendentry{DReG 16};

\pgfplotstablesort[sort key=validklp]{\sorted}
                  {data/synthetic/l40x0-bs256-hs128-t40-dreg64-kl};
\addplot[tealhollow] table [x=validklp, y=validnll] {\sorted};
\addlegendentry{};

\pgfplotstablesort[sort key=validklp]{\sorted}
                  {data/synthetic/l40x0-bs256-hs128-t40-dreg64-renyi};
\addplot[tealsolid] table [x=validklp, y=validnll] {\sorted};
\addlegendentry{DReG 64};

\end{axis}
\end{tikzpicture}
\caption{\small KL and Rényi objectives on continuous synthetic data with base estimators ELBO, DReG.}
\label{fig:synthetic-dreg}
\end{figure}
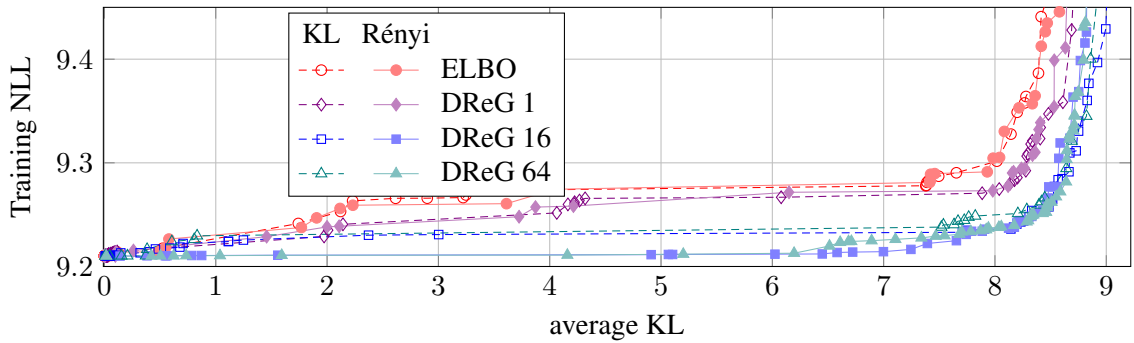

\subsubsection{Continuous Latents}
With 40 continuous latents and an isotropic standard normal prior, the base estimator is either the standard ELBO, IWAE or IWAE-DReG.
In the ELBO, the KL term \eqref{eq:elbo-prior-contrastive} is computed analytically, while in IWAE 1, the single sample from the latents is used to estimate it.
An improvement to IWAE 1 is the STL estimator from \citet{roeder2017sticking}, which removes a zero-expectation term from the objective and whose gradients have zero variance when the variational approximation is exact.
IWAE-DReG, whose objective is also identical to that of IWAE, is a generalization of STL to multiple samples, fixing the signal-to-noise problem of gradient estimates of IWAE \citep{rainforth2018tighter}, wherein the magnitude of the gradient decreases faster with more samples than the variance of the gradient estimates.
Although DReG can also be applied to Reweighted Wake-Sleep \citep{bornschein2014reweighted} and Jackknife Variational Inference \citep{nowozin2018debiasing}, we only use IWAE-DReG in our experiments and let DReG stand for IWAE-DReG without ambiguity.

Our results in \Cref{fig:synthetic-iwae,fig:synthetic-dreg} are in agreement with these previous works.
In the context of this work, our findings can be summarized as follows.
\begin{itemize}
\item Compared to the theoretical optimum, a horizontal line at $\ln 10000 \approx 9.210$, all Pareto curves slope increasingly upwards with more latent usage.
\item Higher-variance estimators have steeper curves than lower variance estimators, which is most apparent in the contrast between two of our baselines, IWAE 1 and DReG 1.
\item The single sample estimators ELBO, IWAE 1, DReG 1 - which are equivalent to a \betavae{} (\S\ref{sec:connection-to-the-beta-vae})\,- are less efficient in their latent usage than our proposed multi-sample estimators.
\item The Rényi objective performs slightly better than the KL objective with the same base estimator.
\end{itemize}

In more detail, \Cref{fig:synthetic-iwae} shows that both of our estimators perform much worse with IWAE 1 as the base estimator than with the ELBO.
As the number of samples grows, this is reversed, although with a high number of samples we do see the predicted degredation \citep{rainforth2018tighter}.

Next, \Cref{fig:synthetic-dreg} confirms DReG's advantage over the IWAE, performing more efficiently than the ELBO even with a single sample.
Our multi-sample objectives both improve on the stronger baseline DReG 1 represents.
The Rényi objective outperforms the KL objective by a small but consistent margin before all estimators start degenerating quickly nearing the maximal average KL possible.

In \Cref{fig:synthetic-dreg-samples}, we take a closer look at the best performing DReG 16 base estimator to determine whether the improvements are due to a better estimate of the marginal likelihood $\ln p(x)$, or the mutual information term $I_{p_D}^p(X,Z)$ in \eqref{eq:cross-mi-obj}.
DReG 16/1 and DReG 1/16 (which use multiple samples to estimate only the marginal likelihood or the mutual information, respectively) are less efficient than DReG 16, sometimes being outperformed even by DReG 1.
Out of the two, DReG 1/16 performs better, indicating that the variance of the base estimator DReG 1 is low and the problem lies with the mutual information estimate.
These results support not only our observations in \S\ref{sec:mi-objective} on the unsuitability of $I_{p,q}$ as an estimate of $I_p$ in the multi-sample case (DReG 16/1), when $q(z|x)$ is not a good approximation of $p(z|x)$, but also demonstrate that $I_{p,q}$ can be improved by taking multiple samples even when $\ln p(x)$ is estimated with a single sample and $q(z|x)$ better approximates the true posterior $p(z|x)$.

\begin{figure}
\begin{tikzpicture}
\begin{axis}[xmin=-0.01,xmax=9.22,ymin=9.20,ymax=9.45,
  xlabel=average KL, ylabel=Training NLL,
  height=0.33*\textwidth,
  width=\textwidth,
  grid=major,
  legend pos=north west,
  legend columns=2,
  legend style={
    font=\mystrut,
    legend cell align=left,
    at={(0.18,0.98)},
    anchor=north west,
  },
]

\addlegendimage{legend image with text=KL}
\addlegendentry{}
\addlegendimage{legend image with text=Rényi}
\addlegendentry{}

\pgfplotstablesort[sort key=validklp]{\sorted}
                  {data/synthetic/l40x0-bs256-hs128-t40-dreg1-kl}
\addplot[redhollow] table [x=validklp, y=validnll] {\sorted};
\addlegendentry{};

\pgfplotstablesort[sort key=validklp]{\sorted}
                  {data/synthetic/l40x0-bs256-hs128-t40-dreg1-renyi}
\addplot[redsolid] table [x=validklp, y=validnll] {\sorted};
\addlegendentry{DReG 1};

\pgfplotstablesort[sort key=validklp]{\sorted}
                  {data/synthetic/l40x0-bs256-hs128-t40-dreg16-kl1};
\addplot[violethollow] table [x=validklp, y=validnll] {\sorted};
\addlegendentry{};

\pgfplotstablesort[sort key=validklp]{\sorted}
                  {data/synthetic/l40x0-bs256-hs128-t40-dreg16-renyi1}
\addplot[violetsolid] table [x=validklp, y=validnll] {\sorted};
\addlegendentry{DReG 16/1};

\pgfplotstablesort[sort key=validklp]{\sorted}
                  {data/synthetic/l40x0-bs256-hs128-t40-dreg1-kl16};
\addplot[bluehollow] table [x=validklp, y=validnll] {\sorted};
\addlegendentry{};

\pgfplotstablesort[sort key=validklp]{\sorted}
                  {data/synthetic/l40x0-bs256-hs128-t40-dreg1-renyi16}
\addplot[bluesolid] table [x=validklp, y=validnll] {\sorted};
\addlegendentry{DReG 1/16};

\pgfplotstablesort[sort key=validklp]{\sorted}
                  {data/synthetic/l40x0-bs256-hs128-t40-dreg16-kl};
\addplot[tealhollow] table [x=validklp, y=validnll] {\sorted};
\addlegendentry{};

\pgfplotstablesort[sort key=validklp]{\sorted}
                  {data/synthetic/l40x0-bs256-hs128-t40-dreg16-renyi};
\addplot[tealsolid] table [x=validklp, y=validnll] {\sorted};
\addlegendentry{DReG 16};

\end{axis}
\end{tikzpicture}
\caption{\small KL and Rényi objectives on continuous synthetic data with base estimator DReG. DReG 1 uses a single sample to estimate both $p(x)$ and $\KL(p(z|x)\|p(z))$. DReG 16/1 uses 16 samples to estimate $p(x)$ and 1 sample for the KL. DReG 1/16 uses 1 sample to estimate $p(x)$ and 16 samples for the KL. Finally, DReG 16 uses the same 16 samples for both terms.}
\label{fig:synthetic-dreg-samples}
\end{figure}
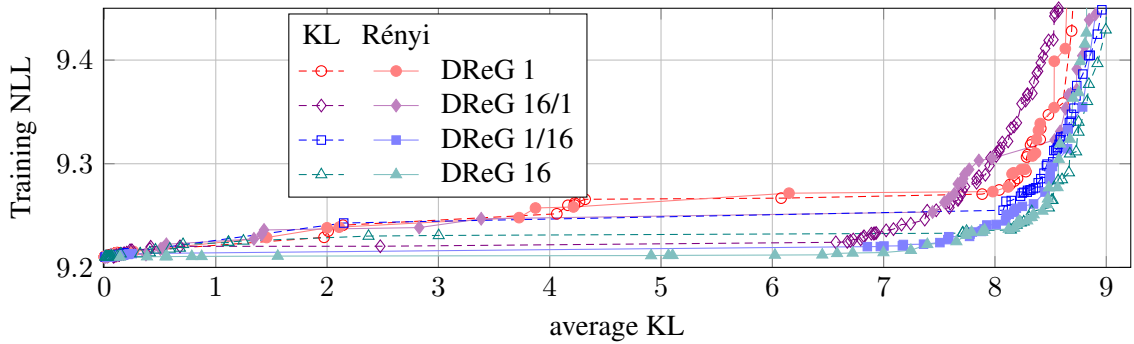

\subsubsection{Discrete Latents}

Next we performed experiments with 8 categorical latent variables, each with a uniform prior over 10 categories.
Using the high-variance REINFORCE base estimator (\Cref{fig:synthetic-reinforce}), we could only get a small improvement over the single-sample case with 16 samples and a similar degradation with 64.

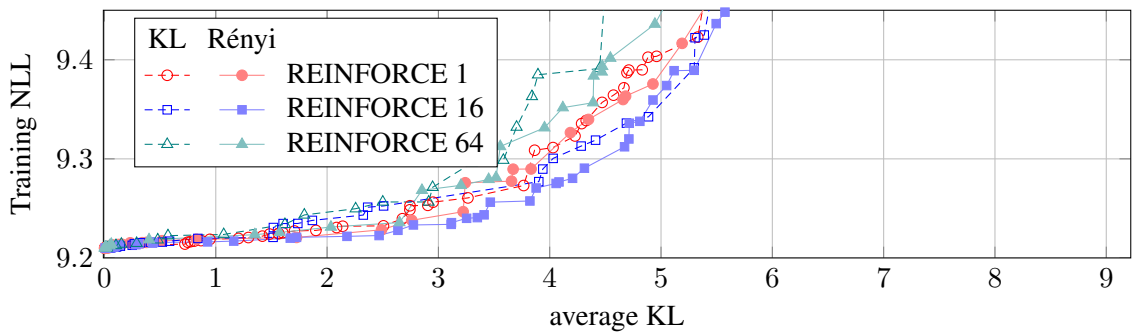
\begin{figure}
\begin{tikzpicture}
\begin{axis}[xmin=-0.01,xmax=9.22,ymin=9.20,ymax=9.45,
  xlabel=average KL, ylabel=Training NLL,
  height=0.32*\textwidth,
  width=\textwidth,
  grid=major,
  legend pos=north west,
  legend columns=2,
  legend style={
    font=\mystrut,
    legend cell align=left,
  },
]

\addlegendimage{legend image with text=KL}
\addlegendentry{}
\addlegendimage{legend image with text=Rényi}
\addlegendentry{}

\pgfplotstablesort[sort key=validklp]{\sorted}
                  {data/synthetic/l8x10-bs256-hs128-t40-vimco1-kl}
\addplot[redhollow] table [x=validklp, y=validnll] {\sorted};
\addlegendentry{};

\pgfplotstablesort[sort key=validklp]{\sorted}
                  {data/synthetic/l8x10-bs256-hs128-t40-vimco1-renyi}
\addplot[redsolid] table [x=validklp, y=validnll] {\sorted};
\addlegendentry{REINFORCE 1};

\pgfplotstablesort[sort key=validklp]{\sorted}
                  {data/synthetic/l8x10-bs256-hs128-t40-reinforce16-kl};
\addplot[bluehollow] table [x=validklp, y=validnll] {\sorted};
\addlegendentry{};

\pgfplotstablesort[sort key=validklp]{\sorted}
                  {data/synthetic/l8x10-bs256-hs128-t40-reinforce16-renyi}
\addplot[bluesolid] table [x=validklp, y=validnll] {\sorted};
\addlegendentry{REINFORCE 16};

\pgfplotstablesort[sort key=validklp]{\sorted}
                  {data/synthetic/l8x10-bs256-hs128-t40-reinforce64-kl};
\addplot[tealhollow] table [x=validklp, y=validnll] {\sorted};
\addlegendentry{};

\pgfplotstablesort[sort key=validklp]{\sorted}
                  {data/synthetic/l8x10-bs256-hs128-t40-reinforce64-renyi};
\addplot[tealsolid] table [x=validklp, y=validnll] {\sorted};
\addlegendentry{REINFORCE 64};

\end{axis}
\end{tikzpicture}
\caption{\small KL and Rényi objectives on discrete synthetic data with base estimator REINFORCE.}
\label{fig:synthetic-reinforce}
\end{figure}

However, with the much lower variance VIMCO estimator, we achieved almost perfect results (see \Cref{fig:synthetic-vimco}) with 16 samples.
For our purposes, it suffices to think of VIMCO as IWAE applied to discrete latents with lower variance gradients.
It may be surprising that these results are even better than with continuous latents, especially at near-maximal latent usage.
To intuit why, consider the minimum variance posterior achievable for a given level of average KL.
For discrete latents, a hard posterior (i.e. of zero variance) is possible depending on the number of latents and categories.
For continuous latents, the posterior can never be hard: the mutual information determines a lower bound on the average variance of the posterior.

Increasing the number of samples further to 64 made results much worse, indicating a potential issue with VIMCO, which may be similar to the signal-to-noise issue that DReG addresses, but this is beyond the scope of this paper.
Note that there is no single-sample VIMCO since its baseline for the contribution of a sample is computed as the average over the rest of the samples, which is undefined.
Assuming a zero baseline, we recover REINFORCE 1, which we use for comparison wherever VIMCO 1 would be needed.

\begin{figure}
\begin{tikzpicture}
\begin{axis}[xmin=-0.01,xmax=9.22,ymin=9.20,ymax=9.45,
  xlabel=average KL, ylabel=Training NLL,
  height=0.32*\textwidth,
  width=\textwidth,
  grid=major,
  legend columns=2,
  legend style={
    font=\mystrut,
    at={(0.6,0.98)},
    anchor=north west,
    legend cell align=left,
  },
]

\addlegendimage{legend image with text=KL}
\addlegendentry{}
\addlegendimage{legend image with text=Rényi}
\addlegendentry{}

\pgfplotstablesort[sort key=validklp]{\sorted}
                  {data/synthetic/l8x10-bs256-hs128-t40-vimco1-kl}
\addplot[redhollow] table [x=validklp, y=validnll] {\sorted};
\addlegendentry{};

\pgfplotstablesort[sort key=validklp]{\sorted}
                  {data/synthetic/l8x10-bs256-hs128-t40-vimco1-renyi}
\addplot[redsolid] table [x=validklp, y=validnll] {\sorted};
\addlegendentry{REINFORCE 1};

\pgfplotstablesort[sort key=validklp]{\sorted}
                  {data/synthetic/l8x10-bs256-hs128-t40-vimco16-kl};
\addplot[bluehollow] table [x=validklp, y=validnll] {\sorted};
\addlegendentry{};

\pgfplotstablesort[sort key=validklp]{\sorted}
                  {data/synthetic/l8x10-bs256-hs128-t40-vimco16-renyi}
\addplot[bluesolid] table [x=validklp, y=validnll] {\sorted};
\addlegendentry{VIMCO 16};

\pgfplotstablesort[sort key=validklp]{\sorted}
                  {data/synthetic/l8x10-bs256-hs128-t40-vimco64-kl};
\addplot[tealhollow] table [x=validklp, y=validnll] {\sorted};
\addlegendentry{};

\pgfplotstablesort[sort key=validklp]{\sorted}
                  {data/synthetic/l8x10-bs256-hs128-t40-vimco64-renyi};
\addplot[tealsolid] table [x=validklp, y=validnll] {\sorted};
\addlegendentry{VIMCO 64};

\end{axis}
\end{tikzpicture}
\caption{\small KL and Rényi objectives on discrete synthetic data with base estimator VIMCO. The single-sample VIMCO is equivalent to REINFORCE.}
\label{fig:synthetic-vimco}
\end{figure}
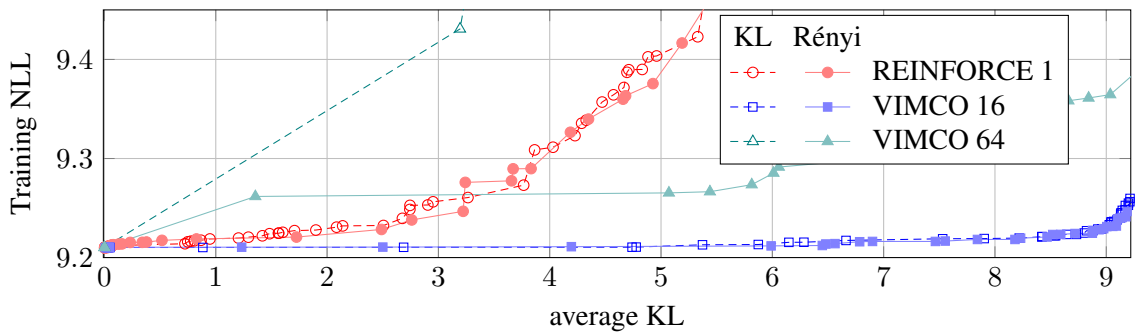

Similarly to \Cref{fig:synthetic-dreg-samples} in the continuous case, in \Cref{fig:synthetic-vimco-samples}, we try to tease apart the contributions of using multiple samples for estimating the marginal likelihood and mutual information terms of \eqref{eq:cross-mi-obj}.
The results are much more pronounced here.
Both VIMCO 16/1 and VIMCO 1/16 improve significantly on REINFORCE 1 but only with the Rényi objective.
Still, both fall way short of VIMCO 16, which employs multiple samples for both terms.
Again, these results support not only our observations in \S\ref{sec:mi-objective} on the unsuitability of $I_{p,q}$ as an estimate of $I_p$ in the multi-sample case (DReG 16/1), when $q(z|x)$ is not a good approximation of $p(z|x)$, but also demonstrate that $I_{p,q}$ can be improved by taking multiple samples even when $\ln p(x)$ is estimated with a single sample and $q(z|x)$ better approximates the true posterior $p(z|x)$.

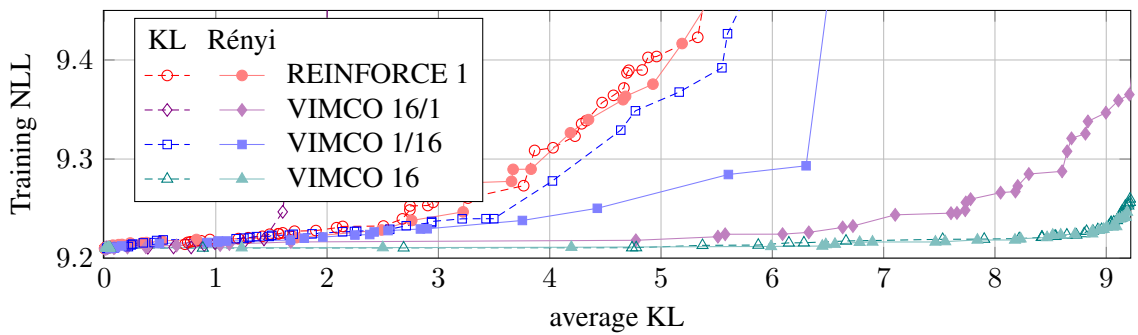
\begin{figure}
\begin{tikzpicture}
\begin{axis}[xmin=-0.01,xmax=9.22,ymin=9.20,ymax=9.45,
  xlabel=average KL, ylabel=Training NLL,
  height=0.32*\textwidth,
  width=\textwidth,
  grid=major,
  legend pos=north west,
  legend columns=2,
  legend style={
    font=\mystrut,
    legend cell align=left,
    at={(0.03,0.97)},
    anchor=north west,
  },
]

\addlegendimage{legend image with text=KL}
\addlegendentry{}
\addlegendimage{legend image with text=Rényi}
\addlegendentry{}

\pgfplotstablesort[sort key=validklp]{\sorted}
                  {data/synthetic/l8x10-bs256-hs128-t40-vimco1-kl}
\addplot[redhollow] table [x=validklp, y=validnll] {\sorted};
\addlegendentry{};

\pgfplotstablesort[sort key=validklp]{\sorted}
                  {data/synthetic/l8x10-bs256-hs128-t40-vimco1-renyi}
\addplot[redsolid] table [x=validklp, y=validnll] {\sorted};
\addlegendentry{REINFORCE 1};

\pgfplotstablesort[sort key=validklp]{\sorted}
                  {data/synthetic/l8x10-bs256-hs128-t40-vimco16-kl1};
\addplot[violethollow] table [x=validklp, y=validnll] {\sorted};
\addlegendentry{};

\pgfplotstablesort[sort key=validklp]{\sorted}
                  {data/synthetic/l8x10-bs256-hs128-t40-vimco16-renyi1}
\addplot[violetsolid] table [x=validklp, y=validnll] {\sorted};
\addlegendentry{VIMCO 16/1};

\pgfplotstablesort[sort key=validklp]{\sorted}
                  {data/synthetic/l8x10-bs256-hs128-t40-vimco1-kl16};
\addplot[bluehollow] table [x=validklp, y=validnll] {\sorted};
\addlegendentry{};

\pgfplotstablesort[sort key=validklp]{\sorted}
                  {data/synthetic/l8x10-bs256-hs128-t40-vimco1-renyi16}
\addplot[bluesolid] table [x=validklp, y=validnll] {\sorted};
\addlegendentry{VIMCO 1/16};

\pgfplotstablesort[sort key=validklp]{\sorted}
                  {data/synthetic/l8x10-bs256-hs128-t40-vimco16-kl};
\addplot[tealhollow] table [x=validklp, y=validnll] {\sorted};
\addlegendentry{};

\pgfplotstablesort[sort key=validklp]{\sorted}
                  {data/synthetic/l8x10-bs256-hs128-t40-vimco16-renyi};
\addplot[tealsolid] table [x=validklp, y=validnll] {\sorted};
\addlegendentry{VIMCO 16};

\end{axis}
\end{tikzpicture}
\caption{\small KL and Rényi objectives on discrete synthetic data with base estimator VIMCO. REINFORCE 1 (standing in for VIMCO 1) uses a single sample to estimate both $p(x)$ and $\KL(p(z|x)\|p(z))$. VIMCO 16/1 uses 16 samples to estimate $p(x)$ and 1 samples for the KL. VIMCO 1/16 uses 1 sample to estimate $p(x)$ and 16 samples for the KL. Finally, VIMCO 16 uses the same 16 samples for both terms.}
\label{fig:synthetic-vimco-samples}
\end{figure}

In Appendix~\ref{sec:additional-experiments-on-synthetic-data}, we also present results for the VQ-VAE \citep{van2017neural} \emph{without} augmenting its objective with a mutual information term.
Since the KL cost in VQ-VAEs is determined by the latent space and is fixed during training, we tune the number of latent variables and the number of categories to control the mutual information.
Results of the VQ-VAE are better than REINFORCE but much worse than VIMCO 16 with either the KL or the Rényi objective.

\subsection{Language Modelling Experiments}
\label{sec:language-modelling-experiments}

One of the most challenging applications of variational autoencoders is language modelling with per-sentence latents, as found by \citet{bowman2015generating}.
They recognize the generality of the underspecification issue and attribute the increased difficulty to ``the sensitivity of the LSTM to subtle variations in its hidden state as introduced by the sampling process''.
In this section, we first show that the data fit vs latent usage tradeoff is even more pronounced in the language modelling case than on the synthetic task, then confirm that the proposed estimators improve validation set results in terms of the Pareto frontiers.
Once again, the improvement is strongest with discrete latents.

\subsubsection{Data Set}

We do sentence-level language modelling on the Penn Treebank (PTB) corpus by \citet{marcus1993building} with preprocessing from \citet{mikolov2010recurrent}.
Our goal here is to compare inference methods, not to establish a new state of the art, so to reduce the computational burden brought about by hyperparemeter tuning, we truncated sentences to 35 tokens in both the training and validation sets, with a reduction of 3\% in the total number of tokens.
This truncation is non-standard.

\subsubsection{Model Architecture}

The model architecture is like in the synthetic case except the encoder embeds the input tokens and feeds them to a one-layer, bidirectional LSTM \citep{hochreiter1997long} and the output $o$ (from which the parameters of the variational posterior $q(z|x)$ are computed) is the average of the last states of LSTMs corresponding to the two directions.
There is a fixed number of latents for all sentences, regardless of their length.
The decoder is a unidirectional LSTM whose input is the embedding of the previous token plus the values of the latent variables.
Unless stated otherwise, the embedding and hidden sizes are all 128.
Similarly to the synthetic case, we use either 40 real-valued latents with an isotropic standard normal prior or 8 categorical latent variables, each with a uniform prior over 10 categories.

\subsubsection{Evaluation Methodology}

Following previous works, both the reported NLL and the average KL values are averages over tokens in the data set.
Since there is only a single set of latents per sentence, this means that the average sentence-level KL is about 21 (the average number of tokens per sentence) times larger.
For expediency, only the base estimators that performed best on the synthetic data set are considered, leaving us with DReG for continuous, and VIMCO for discrete latents.
Training and validation NLLs are estimated using IWAE with 100 samples.
Test NLLs in \Cref{tab:ptb-best-nlls} are estimated using IWAE with 500 samples.

\begin{figure}
\begin{tikzpicture}
\begin{axis}[xmin=-0.01,xmax=2.01,ymin=2.2,ymax=3.2,
  xlabel=per-token KL, ylabel=Training NLL,
  height=0.32*\textwidth,
  width=\textwidth,
  grid=major,
  legend pos=south east,
  legend columns=2,
  legend style={
    font=\mystrut,
    legend cell align=left,
  },
]

\addlegendimage{legend image with text=KL}
\addlegendentry{}
\addlegendimage{legend image with text=Rényi}
\addlegendentry{}

\pgfplotstablesort[sort key=validklp]{\sorted}
                  {data/l40x0-bs256-hs128-t40-dreg1-kl}
\addplot[redhollow] table [x=validklp, y=validnll] {\sorted};
\addlegendentry{};

\pgfplotstablesort[sort key=validklp]{\sorted}
                  {data/l40x0-bs256-hs128-t40-dreg1-renyi}
\addplot[redsolid] table [x=validklp, y=validnll] {\sorted};
\addlegendentry{DReG 1};

\pgfplotstablesort[sort key=validklp]{\sorted}
                  {data/l40x0-bs256-hs128-t40-dreg16-kl};
\addplot[bluehollow] table [x=validklp, y=validnll] {\sorted};
\addlegendentry{};

\pgfplotstablesort[sort key=validklp]{\sorted}
                  {data/l40x0-bs256-hs128-t40-dreg16-renyi}
\addplot[bluesolid] table [x=validklp, y=validnll] {\sorted};
\addlegendentry{DReG 16};

\pgfplotstablesort[sort key=validklp]{\sorted}
                  {data/l40x0-bs256-hs128-t40-dreg64-kl};
\addplot[tealhollow] table [x=validklp, y=validnll] {\sorted};
\addlegendentry{};

\pgfplotstablesort[sort key=validklp]{\sorted}
                  {data/l40x0-bs256-hs128-t40-dreg64-renyi};
\addplot[tealsolid] table [x=validklp, y=validnll] {\sorted};
\addlegendentry{DReG 64};

\end{axis}
\end{tikzpicture}
\caption{\small KL and Rényi objectives (empty and full markers, respectively) on Penn Treebank with base estimators DReG N, where N is the number of samples used for estimating both $p(x)$ and $\KL(p(z|x)\|p(z))$.}
\label{fig:ptb-dreg-train}
\end{figure}
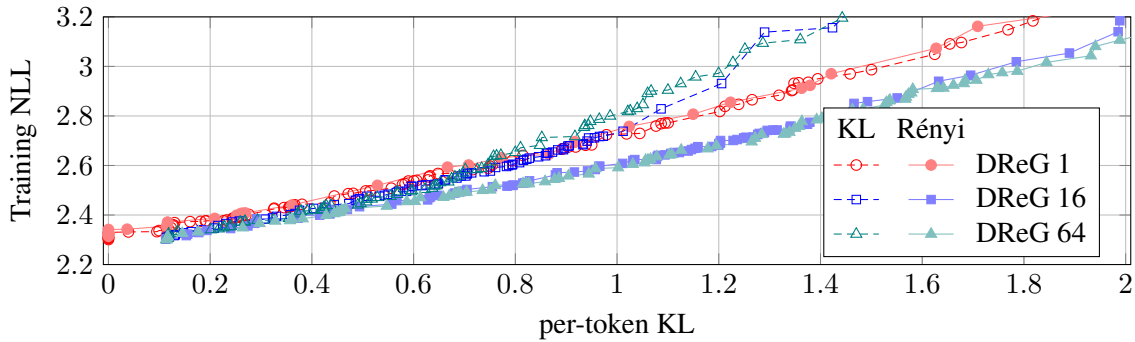

\begin{figure}
\begin{tikzpicture}
\begin{axis}[xmin=-0.01,xmax=2.01,ymin=2.2,ymax=3.2,
  xlabel=per-token KL, ylabel=Training NLL,
  height=0.32*\textwidth,
  width=\textwidth,
  grid=major,
  legend pos=south east,
  legend columns=2,
  legend style={
    font=\mystrut,
    legend cell align=left,
  },
]

\addlegendimage{legend image with text=KL}
\addlegendentry{}
\addlegendimage{legend image with text=Rényi}
\addlegendentry{}

\pgfplotstablesort[sort key=validklp]{\sorted}
                  {data/l8x10-bs256-hs128-t40-vimco1-kl}
\addplot[redhollow] table [x=validklp, y=validnll] {\sorted};
\addlegendentry{};

\pgfplotstablesort[sort key=validklp]{\sorted}
                  {data/l8x10-bs256-hs128-t40-vimco1-renyi}
\addplot[redsolid] table [x=validklp, y=validnll] {\sorted};
\addlegendentry{REINFORCE 1};

\pgfplotstablesort[sort key=validklp]{\sorted}
                  {data/l8x10-bs256-hs128-t40-vimco16-kl};
\addplot[bluehollow] table [x=validklp, y=validnll] {\sorted};
\addlegendentry{};

\pgfplotstablesort[sort key=validklp]{\sorted}
                  {data/l8x10-bs256-hs128-t40-vimco16-renyi}
\addplot[bluesolid] table [x=validklp, y=validnll] {\sorted};
\addlegendentry{VIMCO 16};

\pgfplotstablesort[sort key=validklp]{\sorted}
                  {data/l8x10-bs256-hs128-t40-vimco64-kl};
\addplot[tealhollow] table [x=validklp, y=validnll] {\sorted};
\addlegendentry{};

\pgfplotstablesort[sort key=validklp]{\sorted}
                  {data/l8x10-bs256-hs128-t40-vimco64-renyi};
\addplot[tealsolid] table [x=validklp, y=validnll] {\sorted};
\addlegendentry{VIMCO 64};

\end{axis}
\end{tikzpicture}
\caption{\small KL and Rényi objectives on PTB with base estimator VIMCO.}
\label{fig:ptb-vimco-train}
\end{figure}
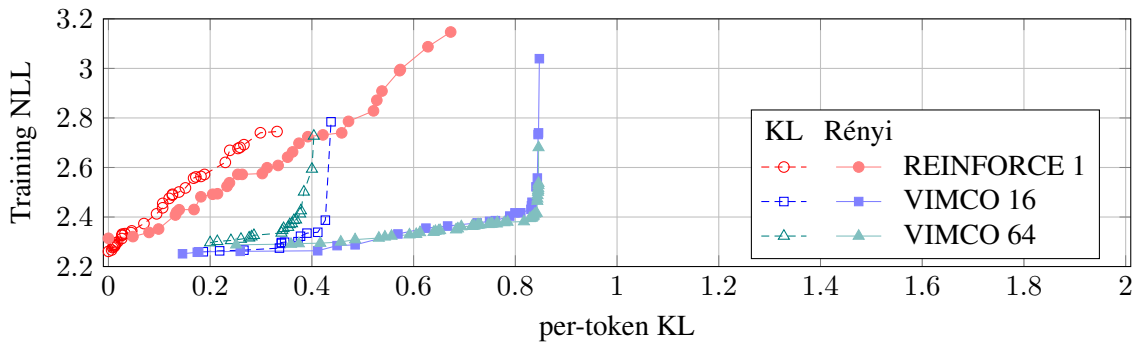

\subsubsection{Training Fit}
In terms of training fit, the language modelling results are similar to those of the synthetic case.
In \Cref{fig:ptb-dreg-train}, there is improvement in the continuous case with multiple samples, initially with both the KL and the Rényi objectives but only with the Rényi at high latent usage.

\Cref{fig:ptb-vimco-train} shows the discrete latents case.
Once again, with either the KL or the Rényi objective, the results improve markedly, outperforming the continuous models.
Note that the KL range is limited by the entropy of the latent space at around $\ln(10^8) / 21 \approx 0.87$.

Furthermore, as Appendix~\ref{sec:additional-language-modelling-experiments-asymmetric-samples} shows, improvements in training fit require multiple samples in both terms, even more so than in the experiments on the synthetic data earlier.

\begin{figure}
\begin{tikzpicture}
\begin{axis}[xmin=-0.01,xmax=2.01,ymin=4.5,ymax=5.5,
  xlabel=per-token KL, ylabel=Validation NLL,
  height=0.33*\textwidth,
  width=\textwidth,
  grid=major,
  legend pos=south east,
  legend columns=2,
  legend style={
    font=\mystrut,
    legend cell align=left,
  },
]

\addlegendimage{legend image with text=DReG}
\addlegendentry{}
\addlegendimage{legend image with text=VIMCO}
\addlegendentry{}

\pgfplotstablesort[sort key=validklp]{\sorted}
                  {data/l40x0-bs256-hs128-t40-dreg1-kl-valid}
\addplot[redhollow] table [x=validklp, y=validnll] {\sorted};
\addlegendentry{}

\pgfplotstablesort[sort key=validklp]{\sorted}
                  {data/l8x10-bs256-hs128-t40-vimco1-kl-valid}
\addplot[redsolid] table [x=validklp, y=validnll] {\sorted};
\addlegendentry{1 K}

\pgfplotstablesort[sort key=validklp]{\sorted}
                  {data/l8x10-bs256-hs128-t40-vimco16-kl-valid}
\addplot[bluehollow] table [x=validklp, y=validnll] {\sorted};
\addlegendentry{}

\pgfplotstablesort[sort key=validklp]{\sorted}
                  {data/l40x0-bs256-hs128-t40-dreg16-kl-valid}
\addplot[bluesolid] table [x=validklp, y=validnll] {\sorted};
\addlegendentry{16 K}

\pgfplotstablesort[sort key=validklp]{\sorted}
                  {data/l40x0-bs256-hs128-t40-dreg16-renyi-valid}
\addplot[tealhollow] table [x=validklp, y=validnll] {\sorted};
\addlegendentry{}

\pgfplotstablesort[sort key=validklp]{\sorted}
                  {data/l8x10-bs256-hs128-t40-vimco16-renyi-valid}
\addplot[tealsolid] table [x=validklp, y=validnll] {\sorted};
\addlegendentry{16 R}

\end{axis}
\end{tikzpicture}
\caption{\small Validation NLL with naive dropout using DReG and VIMCO on PTB. The row with \emph{16 K} refers to 16-sample DReG or VIMCO with the KL objective, while \emph{R} stands for the Rényi objective.}
\label{fig:ptb-naive-dropout-valid}
\end{figure}
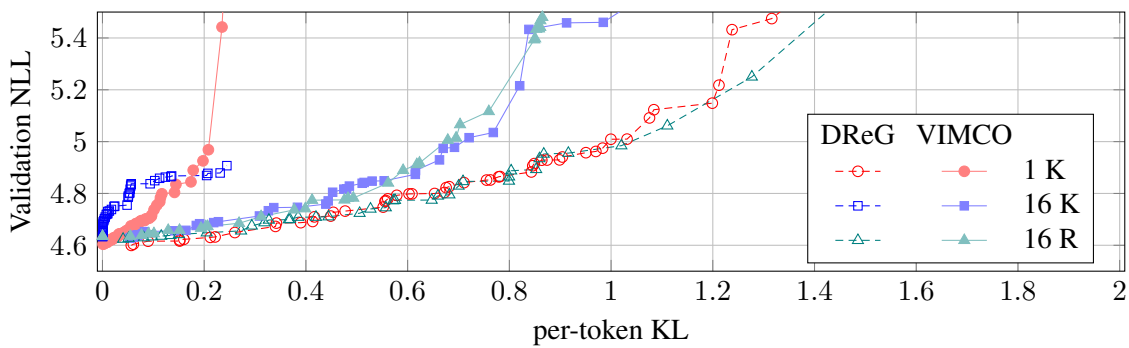

\subsubsection{Validation and Test Results}

Our initial results on the validation set were not very positive.
We tuned hyperparameters on the validation set, where following \cite{melis2017state}, we introduced three additional hyperparameters in the decoder: the rate of dropout applied to the input embedding, the recurrent state, and the output embedding.
As \Cref{fig:ptb-naive-dropout-valid} shows, with this strong form of regularization, the multi-sample results brought no improvement with continuous latents.
With discrete latents, the results improved but not nearly to the extent observed on the training set, and discrete latents performed considerably worse than their continuous counterparts.

\begin{figure}
\begin{tikzpicture}
\begin{axis}[xmin=-0.01,xmax=2.01,ymin=4.5,ymax=5.5,
  xlabel=per-token KL, ylabel=Validation NLL,
  height=0.33*\textwidth,
  width=\textwidth,
  grid=major,
  legend pos=south east,
  legend columns=2,
  legend style={
    font=\mystrut,
    legend cell align=left,
  },
]

\addlegendimage{legend image with text=DReG}
\addlegendentry{}
\addlegendimage{legend image with text=VIMCO}
\addlegendentry{}

\pgfplotstablesort[sort key=validklp]{\sorted}
                  {data/l40x0-bs256-hs128-t40-dreg1-kl-valid-l2}
\addplot[redhollow] table [x=validklp, y=validnll] {\sorted};
\addlegendentry{}

\pgfplotstablesort[sort key=validklp]{\sorted}
                  {data/l8x10-bs256-hs128-t40-vimco1-kl-valid-l2}
\addplot[redsolid] table [x=validklp, y=validnll] {\sorted};
\addlegendentry{1 K}

\pgfplotstablesort[sort key=validklp]{\sorted}
                  {data/l40x0-bs256-hs128-t40-dreg16-kl-valid-l2}
\addplot[bluehollow] table [x=validklp, y=validnll] {\sorted};
\addlegendentry{}

\pgfplotstablesort[sort key=validklp]{\sorted}
                  {data/l8x10-bs256-hs128-t40-vimco16-kl-valid-l2}
\addplot[bluesolid] table [x=validklp, y=validnll] {\sorted};
\addlegendentry{16 K}

\pgfplotstablesort[sort key=validklp]{\sorted}
                  {data/l40x0-bs256-hs128-t40-dreg16-renyi-valid-l2}
\addplot[tealhollow] table [x=validklp, y=validnll] {\sorted};
\addlegendentry{}

\pgfplotstablesort[sort key=validklp]{\sorted}
                  {data/l8x10-bs256-hs128-t40-vimco16-renyi-valid-l2}
\addplot[tealsolid] table [x=validklp, y=validnll] {\sorted};
\addlegendentry{16 R}

\end{axis}
\end{tikzpicture}
\caption{\small Validation NLL with L2 regularization using DReG and VIMCO on PTB.}
\label{fig:ptb-l2-penalty-valid}
\end{figure}
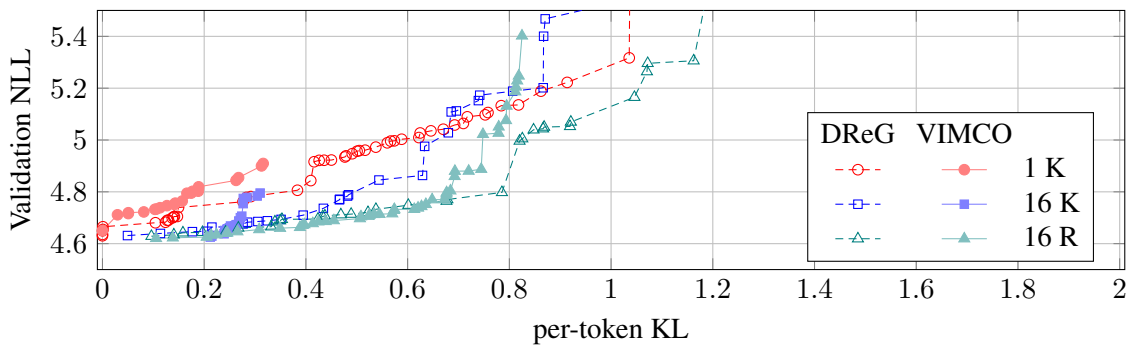

We suspected that the stochasticity from dropout may accentuate the problems of variational inference.
To verify, we repeated the validation experiment with L2 regularization instead of dropout (see \Cref{fig:ptb-l2-penalty-valid}).
In this setting, the validation results are more consistent with training fit. As before, multiple samples significantly improve efficiency.
Unlike the training fit though, validation NLL with discrete latents degrades faster than with continuous ones.
These results suggest that the additional stochasticity of dropout indeed poses a challenge.
On the other hand, with only L2 regularization, the best NLL is higher than with dropout.
This translates to a few perplexity points, which may seem small on the graphs presented here, but in language modelling, kingdoms have been won or lost on such differences \citep{merity2017regularizing}.

To have the best of both worlds, the best NLL of dropout and the Pareto curves of L2 regularization, we went back to dropout, but this time we tried using the same dropout mask for all latent samples belonging to the same sentence in a minibatch.
As \Cref{fig:ptb-same-sample-dropout-valid} shows, this change was successful, and we observed that our proposed estimators improve latent usage for both continuous and discrete latents, and discrete and continuous latents are on par up to an average KL of about 0.5.
This constitutes a significant advance in modelling with discrete latents.

\begin{figure}
\begin{tikzpicture}
\begin{axis}[xmin=-0.01,xmax=2.01,ymin=4.5,ymax=5.5,
  xlabel=per-token KL, ylabel=Validation NLL,
  height=0.33*\textwidth,
  width=\textwidth,
  grid=major,
  legend pos=south east,
  legend columns=2,
  legend style={
    font=\mystrut,
    legend cell align=left,
  },
]

\addlegendimage{legend image with text=DReG}
\addlegendentry{}
\addlegendimage{legend image with text=VIMCO}
\addlegendentry{}

\pgfplotstablesort[sort key=validklp]{\sorted}
                  {data/l40x0-bs256-hs128-t40-dreg1-kl-valid}
\addplot[redhollow] table [x=validklp, y=validnll] {\sorted};
\addlegendentry{}

\pgfplotstablesort[sort key=validklp]{\sorted}
                  {data/l8x10-bs256-hs128-t40-vimco1-kl-valid}
\addplot[redsolid] table [x=validklp, y=validnll] {\sorted};
\addlegendentry{1 K}

\pgfplotstablesort[sort key=validklp]{\sorted}
                  {data/l40x0-bs256-hs128-t40-dreg16-kl-valid-nsd}
\addplot[bluehollow] table [x=validklp, y=validnll] {\sorted};
\addlegendentry{}

\pgfplotstablesort[sort key=validklp]{\sorted}
                  {data/l8x10-bs256-hs128-t40-vimco16-kl-valid-nsd}
\addplot[bluesolid] table [x=validklp, y=validnll] {\sorted};
\addlegendentry{16 K}

\pgfplotstablesort[sort key=validklp]{\sorted}
                  {data/l40x0-bs256-hs128-t40-dreg16-renyi-valid-nsd}
\addplot[tealhollow] table [x=validklp, y=validnll] {\sorted};
\addlegendentry{}

\pgfplotstablesort[sort key=validklp]{\sorted}
                  {data/l8x10-bs256-hs128-t40-vimco16-renyi-valid-nsd}
\addplot[tealsolid] table [x=validklp, y=validnll] {\sorted};
\addlegendentry{16 R}

\end{axis}
\end{tikzpicture}
\caption{\small Validation NLL with DReG and VIMCO on PTB, using the same dropout mask for all latent samples belonging to the same sentence in a minibatch.}
\label{fig:ptb-same-sample-dropout-valid}
\end{figure}
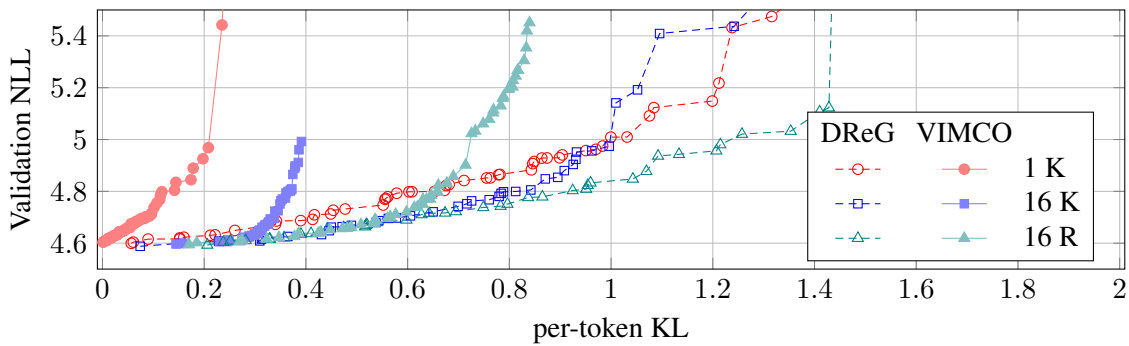

Contrary to results of \citet{pelsmaeker2019effective}, we did not find that introducing latent variables significantly improves outright perplexity.
As \Cref{tab:ptb-best-nlls} shows, the best test perplexity improves only slightly with more samples for this combination of a small model and strong regularization.

\begin{table}
  \centering
  \begin{tabular}{@{}llrrrrrr@{}}
    \toprule
    & Base & \multicolumn{3}{c}{DReG} & \multicolumn{3}{c}{VIMCO} \\
    \cmidrule(lr){3-5} \cmidrule(l){6-8}
    & \#Samples & \multicolumn{1}{c}{1} & \multicolumn{2}{c}{16}
      & \multicolumn{1}{c}{1} & \multicolumn{2}{c}{16} \\
    \cmidrule(lr){3-3} \cmidrule(l){4-5} \cmidrule(l){6-6} \cmidrule(l){7-8}
    & & K / R & \multicolumn{1}{c}{\hspace{0.5em}K}
              & \multicolumn{1}{c}{\hspace{0.5em}R}
      & K / R & \multicolumn{1}{c}{\hspace{0.5em}K}
              & \multicolumn{1}{c}{\hspace{0.em}R} \\
    \midrule
    \parbox[t]{5mm}{\multirow{3}{*}{\rotatebox[origin=c]{90}{\parbox{1.4cm}{\centering hs=128}}}}
    & Perplexity   & \nlltoppl{4.52728} & \nlltoppl{4.51893} & \nlltoppl{4.51719} & \nlltoppl{4.54946} & \nlltoppl{4.52252} & \nlltoppl{4.52724} \\
    & NLL          & 4.527              & 4.518              & 4.517              & 4.549              & 4.523              & 4.527              \\
    & Per-token KL & 0.047              & 0.131              & 0.126              & 0.000              & 0.164              & 0.120              \\
    \midrule
    \parbox[t]{5mm}{\multirow{3}{*}{\rotatebox[origin=c]{90}{\parbox{1.4cm}{\centering hs=256}}}}
    & Perplexity   & \nlltoppl{4.42942} & \nlltoppl{4.41602} & \nlltoppl{4.40914} & \nlltoppl{4.43953} & \nlltoppl{4.42395} & \nlltoppl{4.41962} \\
    & NLL          & 4.429              & 4.416              & 4.409              & 4.439              & 4.424              & 4.420              \\
    & Per-token KL & 0.000              & 0.140              & 0.138              & 0.001              & 0.167              & 0.094              \\
    \bottomrule
  \end{tabular}
  \caption{Best test results for DReG and VIMCO in the continuous and discrete cases, estimated with IWAE and 500 samples. These optima are at low (or, in the single-sample case, mostly negligible) latent usage.}
  \label{tab:ptb-best-nlls}
\end{table}

We also performed experiments with the power objective (\Cref{sec:experiments-with-the-power-objective}) to better understand the tradeoff it represents.
We found that its trivial implementation cost comes at the price of decreased efficiency relative to the general Rényi objective, of which it is a special case.

In closing, we would like to emphasize the importance of the results in \Cref{fig:ptb-same-sample-dropout-valid}.
It is not only that small and large improvements have been made, but that the evaluation throughout focussed on the apparent tradeoff between data fit and latent usage.
Every point on the Pareto curves is the result of tuning several hyperparemeters: the learning rate, $\lambda$, $\alpha$ for the Rényi objective, and three different dropout rates.
These curves capture and communicate what most published experiments do not and what single numbers (e.g. in \Cref{tab:ptb-best-nlls}) cannot: a reliable comparison of a latent variable model to a strongly regularized baseline over a wide range of latent usage.

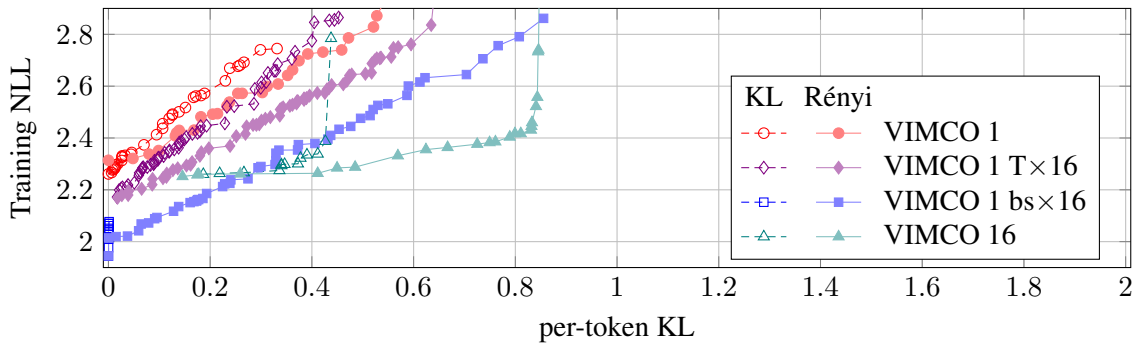
\begin{figure}
\begin{tikzpicture}
\begin{axis}[xmin=-0.01,xmax=2.01,ymin=1.9,ymax=2.9,
  xlabel=per-token KL, ylabel=Training NLL,
  height=0.33*\textwidth,
  width=\textwidth,
  grid=major,
  legend pos=south east,
  legend columns=2,
  legend style={
    font=\mystrut,
    legend cell align=left,
  },
]

\addlegendimage{legend image with text=KL}
\addlegendentry{}
\addlegendimage{legend image with text=Rényi}
\addlegendentry{}

\pgfplotstablesort[sort key=validklp]{\sorted}
                  {data/l8x10-bs256-hs128-t40-vimco1-kl}
\addplot[redhollow] table [x=validklp, y=validnll] {\sorted};
\addlegendentry{};

\pgfplotstablesort[sort key=validklp]{\sorted}
                  {data/l8x10-bs256-hs128-t40-vimco1-renyi}
\addplot[redsolid] table [x=validklp, y=validnll] {\sorted};
\addlegendentry{VIMCO 1};

\pgfplotstablesort[sort key=validklp]{\sorted}
                  {data/l8x10-bs256-hs128-t40x16-vimco1-kl}
\addplot[violethollow] table [x=validklp, y=validnll] {\sorted};
\addlegendentry{};

\pgfplotstablesort[sort key=validklp]{\sorted}
                  {data/l8x10-bs256-hs128-t40x16-vimco1-renyi}
\addplot[violetsolid] table [x=validklp, y=validnll] {\sorted};
\addlegendentry{VIMCO 1 T$\times$16};

\pgfplotstablesort[sort key=validklp]{\sorted}
                  {data/l8x10-bs4096-hs128-t40-vimco1-kl};
\addplot[bluehollow] table [x=validklp, y=validnll] {\sorted};
\addlegendentry{};

\pgfplotstablesort[sort key=validklp]{\sorted}
                  {data/l8x10-bs4096-hs128-t40-vimco1-renyi}
\addplot[bluesolid] table [x=validklp, y=validnll] {\sorted};
\addlegendentry{VIMCO 1 bs$\times$16};

\pgfplotstablesort[sort key=validklp]{\sorted}
                  {data/l8x10-bs256-hs128-t40-vimco16-kl};
\addplot[tealhollow] table [x=validklp, y=validnll] {\sorted};
\addlegendentry{};

\pgfplotstablesort[sort key=validklp]{\sorted}
                  {data/l8x10-bs256-hs128-t40-vimco16-renyi};
\addplot[tealsolid] table [x=validklp, y=validnll] {\sorted};
\addlegendentry{VIMCO 16};

\end{axis}
\end{tikzpicture}
\caption{\small Training NLL on PTB with KL and Rényi objectives and base estimator VIMCO. VIMCO 1 T$\times$16 is trained 16 times longer. VIMCO 1 bs$\times$16 has a 16 times larger batch size. While their best NLL at very low latent usage is lower than that of VIMCO 16, they lose this advantage at higher latent usage.}
\label{fig:ptb-vimco-train-x16}
\end{figure}

\begin{figure}
\begin{tikzpicture}
\begin{axis}[xmin=-0.01,xmax=2.01,ymin=4.5,ymax=5.5,
  xlabel=per-token KL, ylabel=Validation NLL,
  height=0.33*\textwidth,
  width=\textwidth,
  grid=major,
  legend pos=south east,
  legend columns=2,
  legend style={
    font=\mystrut,
    legend cell align=left,
  },
]

\addlegendimage{legend image with text=KL}
\addlegendentry{}
\addlegendimage{legend image with text=Rényi}
\addlegendentry{}

\pgfplotstablesort[sort key=validklp]{\sorted}
                  {data/l8x10-bs256-hs128-t40-vimco1-kl-valid}
\addplot[redhollow] table [x=validklp, y=validnll] {\sorted};
\addlegendentry{};

\pgfplotstablesort[sort key=validklp]{\sorted}
                  {data/l8x10-bs256-hs128-t40-vimco1-renyi-valid}
\addplot[redsolid] table [x=validklp, y=validnll] {\sorted};
\addlegendentry{VIMCO 1};

\pgfplotstablesort[sort key=validklp]{\sorted}
                  {data/l8x10-bs256-hs128-t40x16-vimco1-kl-valid}
\addplot[violethollow] table [x=validklp, y=validnll] {\sorted};
\addlegendentry{};

\pgfplotstablesort[sort key=validklp]{\sorted}
                  {data/l8x10-bs256-hs128-t40x16-vimco1-renyi-valid}
\addplot[violetsolid] table [x=validklp, y=validnll] {\sorted};
\addlegendentry{VIMCO 1 T$\times$16};

\pgfplotstablesort[sort key=validklp]{\sorted}
                  {data/l8x10-bs4096-hs128-t40-vimco1-kl-valid};
\addplot[bluehollow] table [x=validklp, y=validnll] {\sorted};
\addlegendentry{};

\pgfplotstablesort[sort key=validklp]{\sorted}
                  {data/l8x10-bs4096-hs128-t40-vimco1-renyi-valid}
\addplot[bluesolid] table [x=validklp, y=validnll] {\sorted};
\addlegendentry{VIMCO 1 bs$\times$16};

\pgfplotstablesort[sort key=validklp]{\sorted}
                  {data/l8x10-bs256-hs128-t40-vimco16-kl-valid-nsd};
\addplot[tealhollow] table [x=validklp, y=validnll] {\sorted};
\addlegendentry{};

\pgfplotstablesort[sort key=validklp]{\sorted}
                  {data/l8x10-bs256-hs128-t40-vimco16-renyi-valid-nsd};
\addplot[tealsolid] table [x=validklp, y=validnll] {\sorted};
\addlegendentry{VIMCO 16};

\end{axis}
\end{tikzpicture}
\caption{\small Validation NLL on PTB with KL and Rényi objectives and base estimator VIMCO. VIMCO 1 T$\times$16 is trained 16 times longer. VIMCO 1 bs$\times$16 has a 16 times larger batch size. VIMCO 16 performs better than either.}
\label{fig:ptb-vimco-valid-x16}
\end{figure}
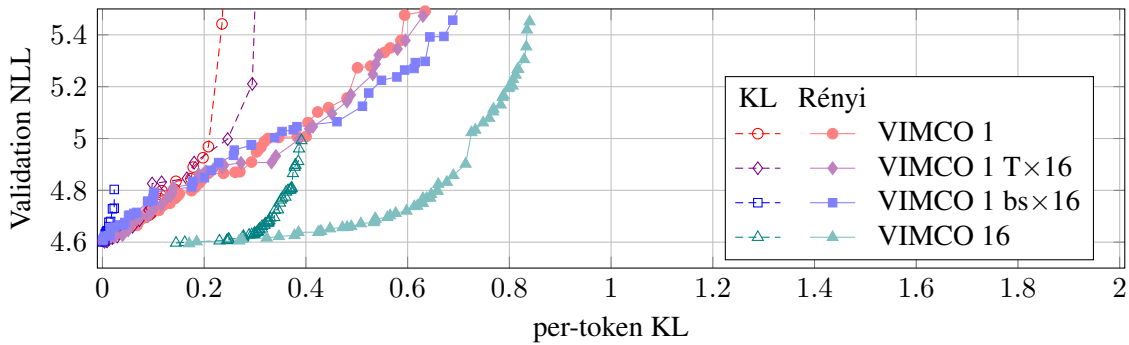

\subsubsection{Single- vs multi-sample at the same computational budget}
\label{sec:single-vs-multi-sample}

The argument for using more samples in Monte-Carlo objectives is that their lower bound is tighter, resulting in a better data fit vs latent usage tradeoff.
Having equipped multi-sample Monte-Carlo objectives with mutual information constraints, we have indeed demonstrated improvements in this tradeoff.
Here, we present additional experiments to answer whether performing more computation with a single-sample estimator can compensate the advantages of multiple samples.
To this end, we compare 16-sample to single-sample estimators, where the latter are trained for 16 times more optimization steps (\emph{T$\times$16}) or with a batch size that is 16 times larger (\emph{bs$\times$16}) to make the computational cost more equal.
In fact, these perform an equal number of gradient calculations but more forward computation as they evaluate $q(z|x)$ 16 times more than the 16-sample Monte-Carlo estimator.

In \Cref{fig:ptb-vimco-train-x16}, we observe that with more computation, the best training fit of VIMCO 1 improves beyond that of VIMCO 16 , but the tradeoff remains severe, as evidenced by that the shapes of the Pareto curves hardly change from the single-sample baseline.

With regards to validation results, we expected the strong regularization provided by tuning three different dropout rates to compensate for possible regularization effects of fewer optimization steps and smaller batch sizes.
Consequently, all single-sample estimators shall exhibit similar validation results at zero latent usage.
\Cref{fig:ptb-vimco-valid-x16} confirms this, and shows that the steeper training curves translate to steeper validation curves, indicating that increasing the computation does not address the ineffeciency of the inference method, and the apparent improvement in training comes at the cost of more overfitting.
In \Cref{sec:additional-language-modelling-experiments-robustness}, \Cref{fig:ptb-dreg-train-x16} and \Cref{fig:ptb-dreg-valid-x16} tell a similar story for DReG, but the results are less conclusive there since the curves are much closer.

\section{Conclusions}
\label{sec:conclusions}

We identified the underspecification of the model and the looseness of the lower bound as two important issues that cause posterior collapse and put forward a natural combination of Monte-Carlo objectives and mutual information constraints to address both at the same time.
The proposed estimators of the mutual information reuse samples from the Monte-Carlo objective of the marginal likelihood to estimate the KL of the true posterior from the prior.
We showed that the representational KL, often used in mutual information constraints, corresponds to the single-sample version of our estimators.
Taking more samples both tightens the lower bound and reduces the variance of the estimate of the true KL.
Our experimental results support these theoretical predictions and underline the need to use multiple samples for both terms of the objective.

Recognizing that the problem of underfitting becomes more acute with increased latent usage, we emphasized evaluating estimators on the training set, where regularization does not cloud the picture, instead of going outright for improvements in held-out performance.
In addition, evaluation was performed in terms of the Pareto frontier of negative log-likelihood vs latent usage curves since reporting a single number cannot capture the tradeoff between the two quantities.

The results demonstrated increased efficiency in latent usage on both the synthetic and language modelling tasks.
For discrete latent spaces in particular, the improvements have been dramatic: from a very weak baseline, data fit improved beyond that of models with continuous latents on both data sets.
In terms of validation results on Penn Treebank, the best continuous and discrete models and estimators are closely matched up until a significant, per-token KL of 0.5 (about 10.5 as a per-sentence KL).

This work is towards opening the door to making the latents truly useful.
We believe that substantial gains in genaralization and utility in down-stream tasks can be achieved by shaping the latent space.
It also remains to be explored how to best encode the true posterior as the representation for use in down-stream models since the approximate posterior is no longer a suitable choice in the context of Monte-Carlo objectives.

In summary, our Mutual Information constrained Monte-Carlo Objectives (MICMCOs) help achieve a better tradeoff between modelling the data and using the latent variables to drive the generative process: a prerequisite for fulfilling the promise of generative modelling.
This tradeoff is still quite severe though, and there is a lot of room for improvement.


\appendix

\section{Additional Experiments on Synthetic Data}
\label{sec:additional-experiments-on-synthetic-data}

Since the KL cost in VQ-VAEs is determined by the latent space and is fixed during training, we tune the number of latent variables and the number of categories to control the mutual information.
As \Cref{fig:synthetic-vq-vae} shows, the VQ-VAE is generally better than REINFORCE 1 but quite far from the optimum and much worse than VIMCO 16 with either the KL or the Rényi objective.

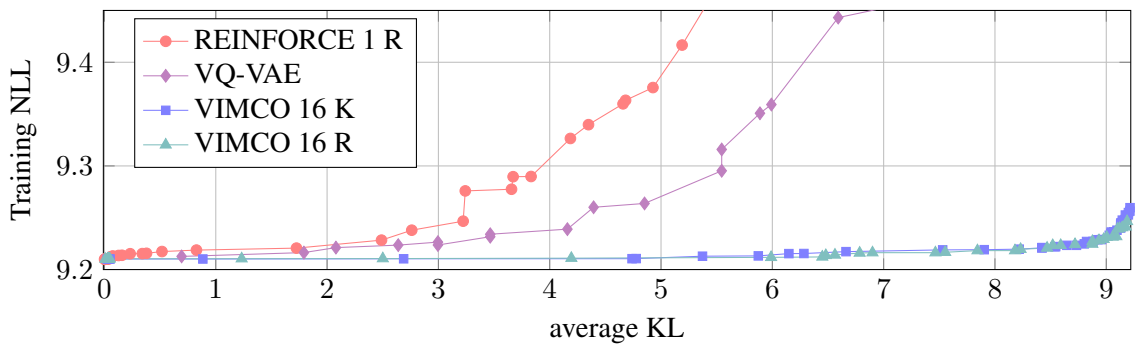
\begin{figure}
\begin{tikzpicture}
\begin{axis}[xmin=-0.01,xmax=9.22,ymin=9.20,ymax=9.45,
  xlabel=average KL, ylabel=Training NLL,
  height=0.33*\textwidth,
  width=\textwidth,
  grid=major,
  legend pos=north west,
  legend columns=1,
  legend style={
    font=\mystrut,
    legend cell align=left,
  },
]

\pgfplotstablesort[sort key=validklp]{\sorted}
                  {data/synthetic/l8x10-bs256-hs128-t40-vimco1-renyi}
\addplot[redsolid] table [x=validklp, y=validnll] {\sorted};
\addlegendentry{REINFORCE 1 R};

\pgfplotstablesort[sort key=validklp]{\sorted}
                  {data/synthetic/bs256-hs128-t40-vq-vae}
\addplot[violetsolid] table [x=validklp, y=validnll] {\sorted};
\addlegendentry{VQ-VAE};

\pgfplotstablesort[sort key=validklp]{\sorted}
                  {data/synthetic/l8x10-bs256-hs128-t40-vimco16-kl};
\addplot[bluesolid] table [x=validklp, y=validnll] {\sorted};
\addlegendentry{VIMCO 16 K}

\pgfplotstablesort[sort key=validklp]{\sorted}
                  {data/synthetic/l8x10-bs256-hs128-t40-vimco16-renyi}
\addplot[tealsolid] table [x=validklp, y=validnll] {\sorted};
\addlegendentry{VIMCO 16 R};

\end{axis}
\end{tikzpicture}
\caption{\small VQ-VAE on synthetic data with tuned latent sizes compared to REINFORCE 1 and VIMCO 16 with the KL and Rényi objectives.}
\label{fig:synthetic-vq-vae}
\end{figure}

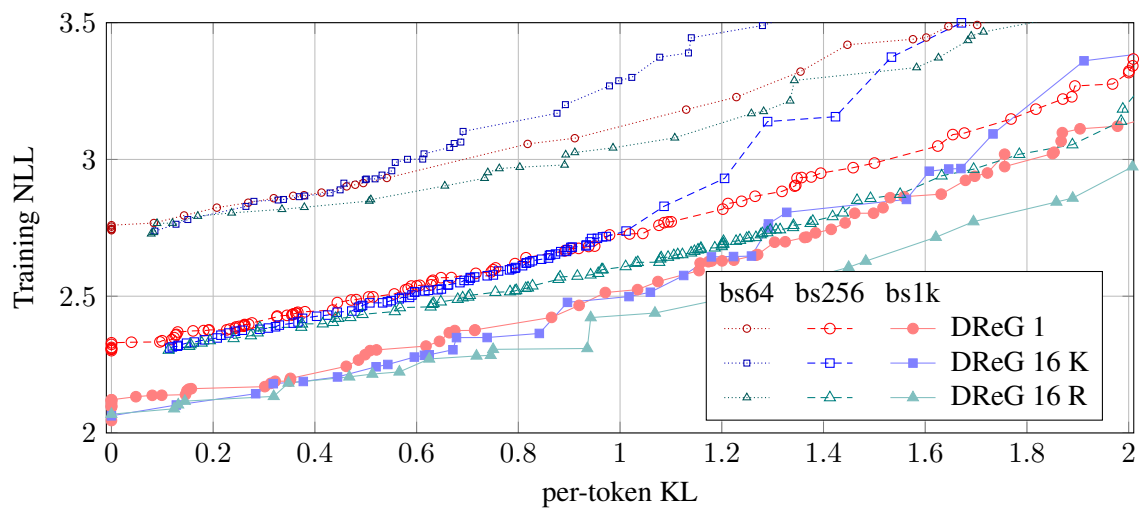
\begin{figure}
\begin{tikzpicture}
\begin{axis}[xmin=-0.01,xmax=2.01,ymin=2.0,ymax=3.5,
  xlabel=per-token KL, ylabel=Training NLL,
  height=1.4*0.33*\textwidth,
  width=\textwidth,
  grid=major,
  legend pos=south east,
  legend columns=3,
  legend style={
    font=\mystrut,
    legend cell align=left,
  },
]

\addlegendimage{legend image with text=bs64}
\addlegendentry{}
\addlegendimage{legend image with text=bs256}
\addlegendentry{}
\addlegendimage{legend image with text=bs1k}
\addlegendentry{}

\pgfplotstablesort[sort key=validklp]{\sorted}
                  {data/l40x0-bs64-hs128-t40-dreg1-kl}
\addplot[redweak] table [x=validklp, y=validnll] {\sorted};
\addlegendentry{}

\pgfplotstablesort[sort key=validklp]{\sorted}
                  {data/l40x0-bs256-hs128-t40-dreg1-kl}
\addplot[redhollow]
        table [x=validklp, y=validnll] {\sorted};
\addlegendentry{}

\pgfplotstablesort[sort key=validklp]{\sorted}
                  {data/l40x0-bs1024-hs128-t40-dreg1-kl}
\addplot[redsolid] table [x=validklp, y=validnll] {\sorted};
\addlegendentry{DReG 1}

\pgfplotstablesort[sort key=validklp]{\sorted}
                  {data/l40x0-bs64-hs128-t40-dreg16-kl};
\addplot[blueweak] table [x=validklp, y=validnll] {\sorted};
\addlegendentry{}

\pgfplotstablesort[sort key=validklp]{\sorted}
                  {data/l40x0-bs256-hs128-t40-dreg16-kl};
\addplot[bluehollow] table [x=validklp, y=validnll] {\sorted};
\addlegendentry{}

\pgfplotstablesort[sort key=validklp]{\sorted}
                  {data/l40x0-bs1024-hs128-t40-dreg16-kl};
\addplot[bluesolid] table [x=validklp, y=validnll] {\sorted};
\addlegendentry{DReG 16 K}

\pgfplotstablesort[sort key=validklp]{\sorted}
                  {data/l40x0-bs64-hs128-t40-dreg16-renyi}
\addplot[tealweak] table [x=validklp, y=validnll] {\sorted};
\addlegendentry{}

\pgfplotstablesort[sort key=validklp]{\sorted}
                  {data/l40x0-bs256-hs128-t40-dreg16-renyi}
\addplot[tealhollow] table [x=validklp, y=validnll] {\sorted};
\addlegendentry{}

\pgfplotstablesort[sort key=validklp]{\sorted}
                  {data/l40x0-bs1024-hs128-t40-dreg16-renyi}
\addplot[tealsolid] table [x=validklp, y=validnll] {\sorted};
\addlegendentry{DReG 16 R}

\end{axis}
\end{tikzpicture}
\caption{\small The effect of batch size with DReG and 128 hidden units on PTB.}
\label{fig:ptb-dreg-bs}
\end{figure}

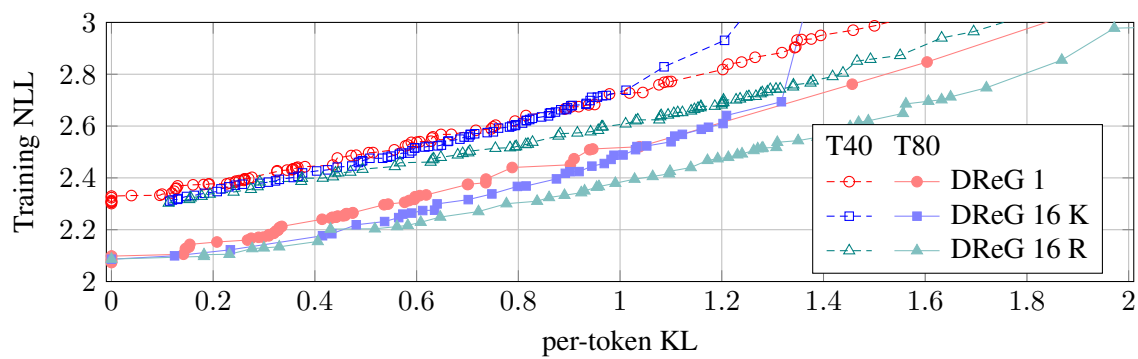
\begin{figure}
\begin{tikzpicture}
\begin{axis}[xmin=-0.01,xmax=2.01,ymin=2.0,ymax=3.0,
  xlabel=per-token KL, ylabel=Training NLL,
  height=0.33*\textwidth,
  width=\textwidth,
  grid=major,
  legend pos=south east,
  legend columns=2,
  legend style={
    font=\mystrut,
    legend cell align=left,
  },
]

\addlegendimage{legend image with text=T40}
\addlegendentry{}
\addlegendimage{legend image with text=T80}
\addlegendentry{}

\pgfplotstablesort[sort key=validklp]{\sorted}
                  {data/l40x0-bs256-hs128-t40-dreg1-kl}
\addplot[redhollow] table [x=validklp, y=validnll] {\sorted};
\addlegendentry{}

\pgfplotstablesort[sort key=validklp]{\sorted}
                  {data/l40x0-bs256-hs128-t80-dreg1-kl}
\addplot[redsolid] table [x=validklp, y=validnll] {\sorted};
\addlegendentry{DReG 1}

\pgfplotstablesort[sort key=validklp]{\sorted}
                  {data/l40x0-bs256-hs128-t40-dreg16-kl};
\addplot[bluehollow] table [x=validklp, y=validnll] {\sorted};
\addlegendentry{}

\pgfplotstablesort[sort key=validklp]{\sorted}
                  {data/l40x0-bs256-hs128-t80-dreg16-kl};
\addplot[bluesolid] table [x=validklp, y=validnll] {\sorted};
\addlegendentry{DReG 16 K}

\pgfplotstablesort[sort key=validklp]{\sorted}
                  {data/l40x0-bs256-hs128-t40-dreg16-renyi}
\addplot[tealhollow] table [x=validklp, y=validnll] {\sorted};
\addlegendentry{}

\pgfplotstablesort[sort key=validklp]{\sorted}
                  {data/l40x0-bs256-hs128-t80-dreg16-renyi}
\addplot[tealsolid] table [x=validklp, y=validnll] {\sorted};
\addlegendentry{DReG 16 R}

\end{axis}
\end{tikzpicture}
\caption{\small The effect of optimization length (40 or 80 thousand optimization steps) with DReG and 128 hidden units on PTB.}
\label{fig:ptb-dreg-ol-hs128}
\end{figure}

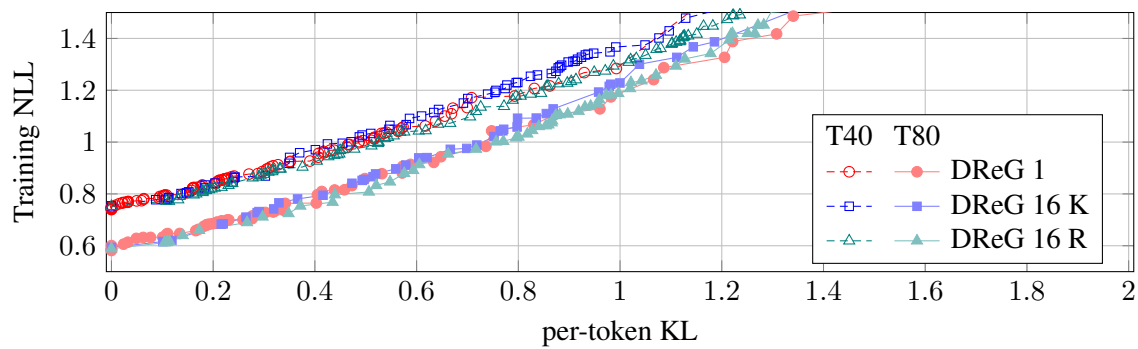
\begin{figure}
\begin{tikzpicture}
\begin{axis}[xmin=-0.01,xmax=2.01,ymin=0.5,ymax=1.5,
  xlabel=per-token KL, ylabel=Training NLL,
  height=0.33*\textwidth,
  width=\textwidth,
  grid=major,
  legend pos=south east,
  legend columns=2,
  legend style={
    font=\mystrut,
    legend cell align=left,
  },
]

\addlegendimage{legend image with text=T40}
\addlegendentry{}
\addlegendimage{legend image with text=T80}
\addlegendentry{}

\pgfplotstablesort[sort key=validklp]{\sorted}
                  {data/l40x0-bs256-hs256-t40-dreg1-kl}
\addplot[redhollow] table [x=validklp, y=validnll] {\sorted};
\addlegendentry{}

\pgfplotstablesort[sort key=validklp]{\sorted}
                  {data/l40x0-bs256-hs256-t80-dreg1-kl}
\addplot[redsolid] table [x=validklp, y=validnll] {\sorted};
\addlegendentry{DReG 1}

\pgfplotstablesort[sort key=validklp]{\sorted}
                  {data/l40x0-bs256-hs256-t40-dreg16-kl};
\addplot[bluehollow] table [x=validklp, y=validnll] {\sorted};
\addlegendentry{}

\pgfplotstablesort[sort key=validklp]{\sorted}
                  {data/l40x0-bs256-hs256-t80-dreg16-kl};
\addplot[bluesolid] table [x=validklp, y=validnll] {\sorted};
\addlegendentry{DReG 16 K}

\pgfplotstablesort[sort key=validklp]{\sorted}
                  {data/l40x0-bs256-hs256-t40-dreg16-renyi}
\addplot[tealhollow] table [x=validklp, y=validnll] {\sorted};
\addlegendentry{}

\pgfplotstablesort[sort key=validklp]{\sorted}
                  {data/l40x0-bs256-hs256-t80-dreg16-renyi}
\addplot[tealsolid] table [x=validklp, y=validnll] {\sorted};
\addlegendentry{DReG 16 R}

\end{axis}
\end{tikzpicture}
\caption{\small The effect of optimization length (40 or 80 thousand optimization steps)
  with DReG with 256 hidden units on PTB.}
\label{fig:ptb-dreg-ol-hs256}
\end{figure}

\begin{figure}
\begin{tikzpicture}
\begin{axis}[xmin=-0.01,xmax=2.01,ymin=2.0,ymax=3.5,
  xlabel=per-token KL, ylabel=Training NLL,
  height=1.4*0.33*\textwidth,
  width=\textwidth,
  grid=major,
  legend pos=south east,
  legend columns=3,
  legend style={
    font=\mystrut,
    legend cell align=left,
  },
]

\addlegendimage{legend image with text=bs64}
\addlegendentry{}
\addlegendimage{legend image with text=bs256}
\addlegendentry{}
\addlegendimage{legend image with text=bs1k}
\addlegendentry{}

\pgfplotstablesort[sort key=validklp]{\sorted}
                  {data/l8x10-bs64-hs128-t40-vimco1-kl}
\addplot[redweak] table [x=validklp, y=validnll] {\sorted};
\addlegendentry{}

\pgfplotstablesort[sort key=validklp]{\sorted}
                  {data/l8x10-bs256-hs128-t40-vimco1-kl}
\addplot[redhollow] table [x=validklp, y=validnll] {\sorted};
\addlegendentry{}

\pgfplotstablesort[sort key=validklp]{\sorted}
                  {data/l8x10-bs1024-hs128-t40-vimco1-kl}
\addplot[redsolid]
        table [x=validklp, y=validnll] {\sorted};
\addlegendentry{VIMCO 1}

\pgfplotstablesort[sort key=validklp]{\sorted}
                  {data/l8x10-bs64-hs128-t40-vimco16-kl};
\addplot[blueweak] table [x=validklp, y=validnll] {\sorted};
\addlegendentry{}

\pgfplotstablesort[sort key=validklp]{\sorted}
                  {data/l8x10-bs256-hs128-t40-vimco16-kl};
\addplot[bluehollow] table [x=validklp, y=validnll] {\sorted};
\addlegendentry{}

\pgfplotstablesort[sort key=validklp]{\sorted}
                  {data/l8x10-bs1024-hs128-t40-vimco16-kl};
\addplot[bluesolid] table [x=validklp, y=validnll] {\sorted};
\addlegendentry{VIMCO 16 K}

\pgfplotstablesort[sort key=validklp]{\sorted}
                  {data/l8x10-bs64-hs128-t40-vimco16-renyi}
\addplot[tealweak] table [x=validklp, y=validnll] {\sorted};
\addlegendentry{}

\pgfplotstablesort[sort key=validklp]{\sorted}
                  {data/l8x10-bs256-hs128-t40-vimco16-renyi}
\addplot[tealhollow] table [x=validklp, y=validnll] {\sorted};
\addlegendentry{}

\pgfplotstablesort[sort key=validklp]{\sorted}
                  {data/l8x10-bs1024-hs128-t40-vimco16-renyi}
\addplot[tealsolid] table [x=validklp, y=validnll] {\sorted};
\addlegendentry{VIMCO 16 R}

\end{axis}
\end{tikzpicture}
\caption{\small The effect of batch size with VIMCO and 128 hidden units on PTB.}
\label{fig:ptb-vimco-bs}
\end{figure}
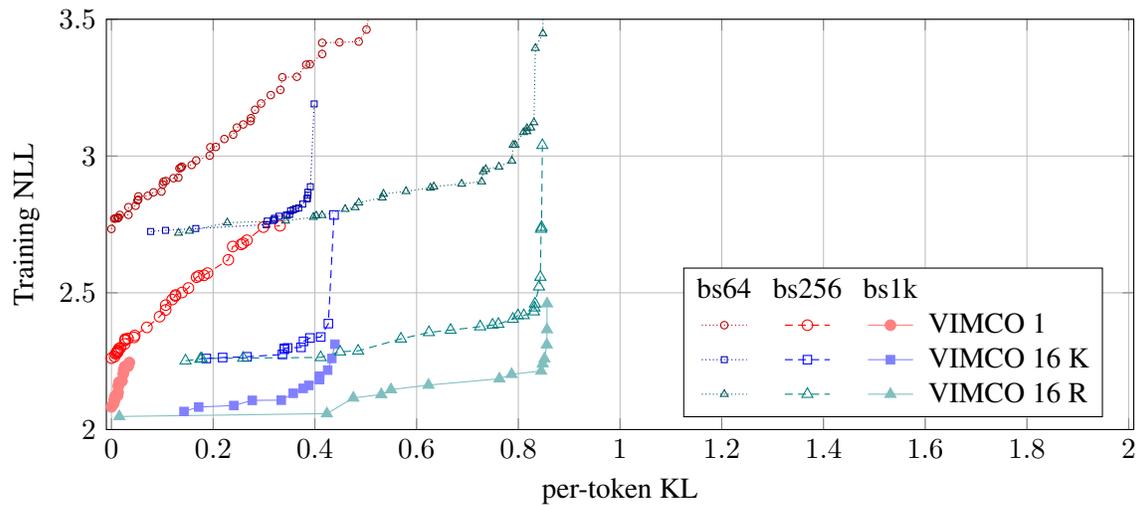

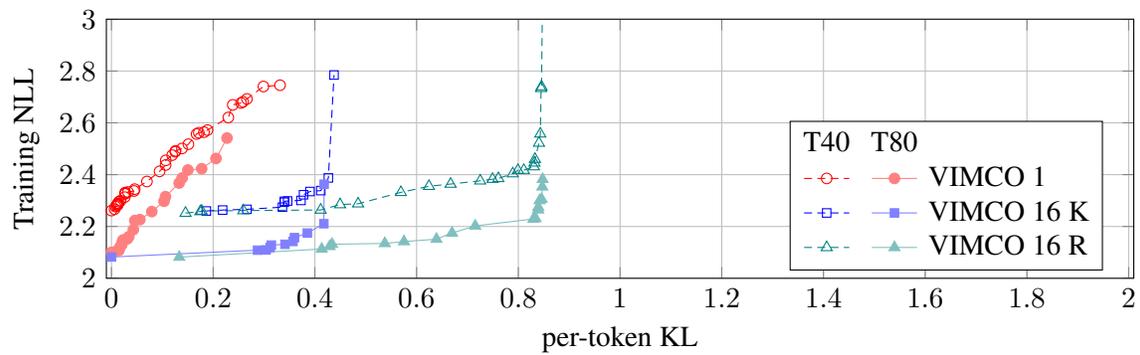
\begin{figure}
\begin{tikzpicture}
\begin{axis}[xmin=-0.01,xmax=2.01,ymin=2.0,ymax=3.0,
  xlabel=per-token KL, ylabel=Training NLL,
  height=0.33*\textwidth,
  width=\textwidth,
  grid=major,
  legend pos=south east,
  legend columns=2,
  legend style={
    font=\mystrut,
    legend cell align=left,
  },
]

\addlegendimage{legend image with text=T40}
\addlegendentry{}
\addlegendimage{legend image with text=T80}
\addlegendentry{}

\pgfplotstablesort[sort key=validklp]{\sorted}
                  {data/l8x10-bs256-hs128-t40-vimco1-kl}
\addplot[redhollow] table [x=validklp, y=validnll] {\sorted};
\addlegendentry{}

\pgfplotstablesort[sort key=validklp]{\sorted}
                  {data/l8x10-bs256-hs128-t80-vimco1-kl}
\addplot[redsolid] table [x=validklp, y=validnll] {\sorted};
\addlegendentry{VIMCO 1}

\pgfplotstablesort[sort key=validklp]{\sorted}
                  {data/l8x10-bs256-hs128-t40-vimco16-kl};
\addplot[bluehollow] table [x=validklp, y=validnll] {\sorted};
\addlegendentry{}

\pgfplotstablesort[sort key=validklp]{\sorted}
                  {data/l8x10-bs256-hs128-t80-vimco16-kl};
\addplot[bluesolid] table [x=validklp, y=validnll] {\sorted};
\addlegendentry{VIMCO 16 K}

\pgfplotstablesort[sort key=validklp]{\sorted}
                  {data/l8x10-bs256-hs128-t40-vimco16-renyi}
\addplot[tealhollow] table [x=validklp, y=validnll] {\sorted};
\addlegendentry{}

\pgfplotstablesort[sort key=validklp]{\sorted}
                  {data/l8x10-bs256-hs128-t80-vimco16-renyi}
\addplot[tealsolid] table [x=validklp, y=validnll] {\sorted};
\addlegendentry{VIMCO 16 R}

\end{axis}
\end{tikzpicture}
\caption{\small The effect of optimization length (40 or 80 thousand optimization steps)
  with VIMCO and 128 hidden units on PTB.}
\label{fig:ptb-vimco-ol-hs128}
\end{figure}

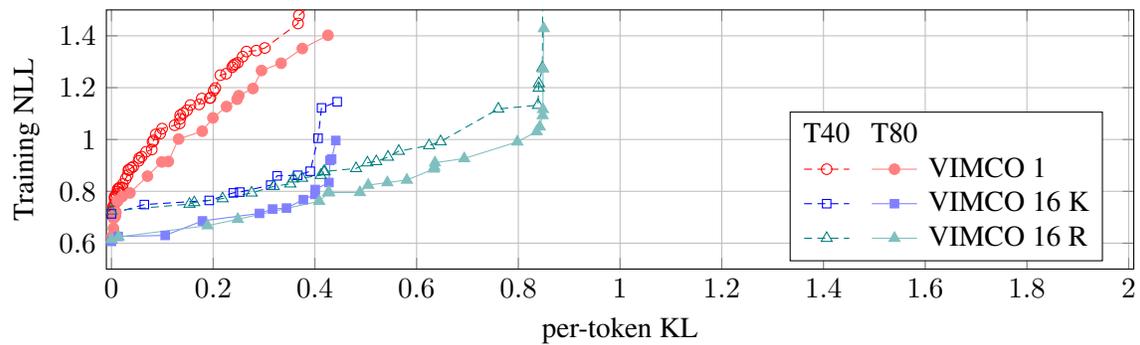
\begin{figure}
\begin{tikzpicture}
\begin{axis}[xmin=-0.01,xmax=2.01,ymin=0.5,ymax=1.5,
  xlabel=per-token KL, ylabel=Training NLL,
  height=0.33*\textwidth,
  width=\textwidth,
  grid=major,
  legend pos=south east,
  legend columns=2,
  legend style={
    font=\mystrut,
    legend cell align=left,
  },
]

\addlegendimage{legend image with text=T40}
\addlegendentry{}
\addlegendimage{legend image with text=T80}
\addlegendentry{}

\pgfplotstablesort[sort key=validklp]{\sorted}
                  {data/l8x10-bs256-hs256-t40-vimco1-kl}
\addplot[redhollow] table [x=validklp, y=validnll] {\sorted};
\addlegendentry{}

\pgfplotstablesort[sort key=validklp]{\sorted}
                  {data/l8x10-bs256-hs256-t80-vimco1-kl}
\addplot[redsolid] table [x=validklp, y=validnll] {\sorted};
\addlegendentry{VIMCO 1}

\pgfplotstablesort[sort key=validklp]{\sorted}
                  {data/l8x10-bs256-hs256-t40-vimco16-kl};
\addplot[bluehollow] table [x=validklp, y=validnll] {\sorted};
\addlegendentry{}

\pgfplotstablesort[sort key=validklp]{\sorted}
                  {data/l8x10-bs256-hs256-t80-vimco16-kl};
\addplot[bluesolid] table [x=validklp, y=validnll] {\sorted};
\addlegendentry{VIMCO 16 K}

\pgfplotstablesort[sort key=validklp]{\sorted}
                  {data/l8x10-bs256-hs256-t40-vimco16-renyi}
\addplot[tealhollow] table [x=validklp, y=validnll] {\sorted};
\addlegendentry{}

\pgfplotstablesort[sort key=validklp]{\sorted}
                  {data/l8x10-bs256-hs256-t80-vimco16-renyi}
\addplot[tealsolid] table [x=validklp, y=validnll] {\sorted};
\addlegendentry{VIMCO 16 R}

\end{axis}
\end{tikzpicture}
\caption{\small The effect of optimization length (40 or 80 thousand optimization steps) with VIMCO and 256 hidden units on PTB.}
\label{fig:ptb-vimco-ol-hs256}
\end{figure}

\begin{figure}
\begin{tikzpicture}
\begin{axis}[xmin=-0.01,xmax=2.01,ymin=2.2,ymax=3.2,
  xlabel=per-token KL, ylabel=Training NLL,
  height=0.33*\textwidth,
  width=\textwidth,
  grid=major,
  legend pos=north west,
  legend columns=2,
  legend style={
    font=\mystrut,
    legend cell align=left,
  },
]

\addlegendimage{legend image with text=KL}
\addlegendentry{}
\addlegendimage{legend image with text=Rényi}
\addlegendentry{}

\pgfplotstablesort[sort key=validklp]{\sorted}
                  {data/l40x0-bs256-hs128-t40-dreg1-kl}
\addplot[redhollow] table [x=validklp, y=validnll] {\sorted};
\addlegendentry{};

\pgfplotstablesort[sort key=validklp]{\sorted}
                  {data/l40x0-bs256-hs128-t40-dreg1-renyi}
\addplot[redsolid] table [x=validklp, y=validnll] {\sorted};
\addlegendentry{DReG 1};

\pgfplotstablesort[sort key=validklp]{\sorted}
                  {data/l40x0-bs256-hs128-t40-dreg1-kl16};
\addplot[violethollow] table [x=validklp, y=validnll] {\sorted};
\addlegendentry{};

\pgfplotstablesort[sort key=validklp]{\sorted}
                  {data/l40x0-bs256-hs128-t40-dreg1-renyi16}
\addplot[violetsolid] table [x=validklp, y=validnll] {\sorted};
\addlegendentry{DReG 1/16};

\pgfplotstablesort[sort key=validklp]{\sorted}
                  {data/l40x0-bs256-hs128-t40-dreg16-kl1};
\addplot[bluehollow] table [x=validklp, y=validnll] {\sorted};
\addlegendentry{};

\pgfplotstablesort[sort key=validklp]{\sorted}
                  {data/l40x0-bs256-hs128-t40-dreg16-renyi1}
\addplot[bluesolid] table [x=validklp, y=validnll] {\sorted};
\addlegendentry{DReG 16/1};

\pgfplotstablesort[sort key=validklp]{\sorted}
                  {data/l40x0-bs256-hs128-t40-dreg16-kl};
\addplot[tealhollow] table [x=validklp, y=validnll] {\sorted};
\addlegendentry{};

\pgfplotstablesort[sort key=validklp]{\sorted}
                  {data/l40x0-bs256-hs128-t40-dreg16-renyi};
\addplot[tealsolid] table [x=validklp, y=validnll] {\sorted};
\addlegendentry{DReG 16};

\end{axis}
\end{tikzpicture}
\caption{\small KL and Rényi objectives on PTB with base estimator DReG. DReG 1 uses a single sample to estimate both $p(x)$ and $\KL(p(z|x)\|p(z))$. DReG 16/1 uses 16 samples to estimate $p(x)$ and 1 samples for the KL. DReG 1/16 uses 1 sample to estimate $p(x)$ and 16 samples for the KL. Finally, DReG 16 uses the same 16 samples for both terms.}
\label{fig:ptb-dreg-train-samples}
\end{figure}
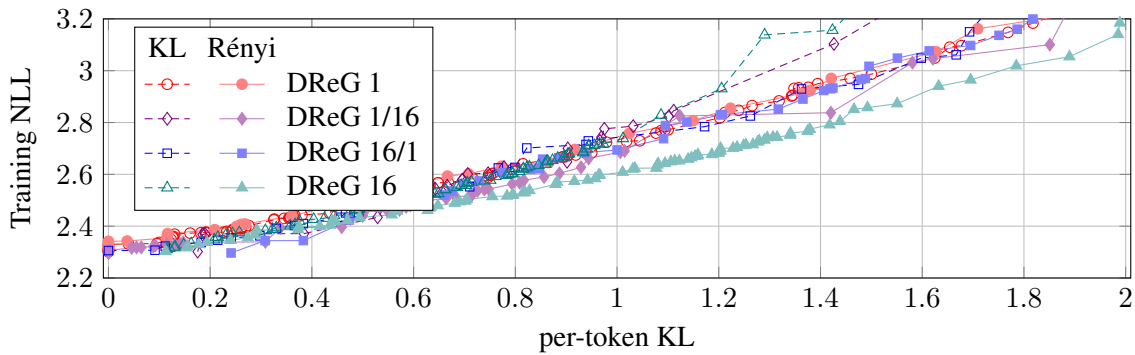

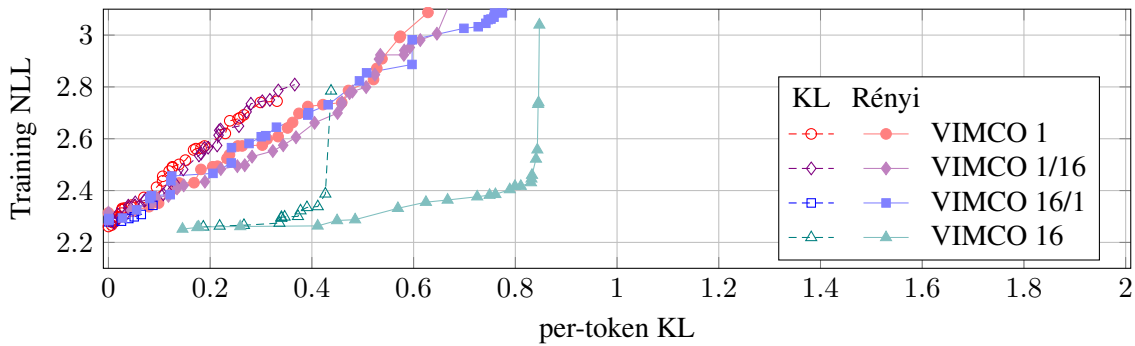
\begin{figure}
\begin{tikzpicture}
\begin{axis}[xmin=-0.01,xmax=2.01,ymin=2.1,ymax=3.1,
  xlabel=per-token KL, ylabel=Training NLL,
  height=0.33*\textwidth,
  width=\textwidth,
  grid=major,
  legend pos=south east,
  legend columns=2,
  legend style={
    font=\mystrut,
    legend cell align=left,
  },
]

\addlegendimage{legend image with text=KL}
\addlegendentry{}
\addlegendimage{legend image with text=Rényi}
\addlegendentry{}

\pgfplotstablesort[sort key=validklp]{\sorted}
                  {data/l8x10-bs256-hs128-t40-vimco1-kl}
\addplot[redhollow] table [x=validklp, y=validnll] {\sorted};
\addlegendentry{};

\pgfplotstablesort[sort key=validklp]{\sorted}
                  {data/l8x10-bs256-hs128-t40-vimco1-renyi}
\addplot[redsolid] table [x=validklp, y=validnll] {\sorted};
\addlegendentry{VIMCO 1};

\pgfplotstablesort[sort key=validklp]{\sorted}
                  {data/l8x10-bs256-hs128-t40-vimco1-kl16};
\addplot[violethollow] table [x=validklp, y=validnll] {\sorted};
\addlegendentry{};

\pgfplotstablesort[sort key=validklp]{\sorted}
                  {data/l8x10-bs256-hs128-t40-vimco1-renyi16}
\addplot[violetsolid] table [x=validklp, y=validnll] {\sorted};
\addlegendentry{VIMCO 1/16};

\pgfplotstablesort[sort key=validklp]{\sorted}
                  {data/l8x10-bs256-hs128-t40-vimco16-kl1};
\addplot[bluehollow] table [x=validklp, y=validnll] {\sorted};
\addlegendentry{};

\pgfplotstablesort[sort key=validklp]{\sorted}
                  {data/l8x10-bs256-hs128-t40-vimco16-renyi1}
\addplot[bluesolid] table [x=validklp, y=validnll] {\sorted};
\addlegendentry{VIMCO 16/1};

\pgfplotstablesort[sort key=validklp]{\sorted}
                  {data/l8x10-bs256-hs128-t40-vimco16-kl};
\addplot[tealhollow] table [x=validklp, y=validnll] {\sorted};
\addlegendentry{};

\pgfplotstablesort[sort key=validklp]{\sorted}
                  {data/l8x10-bs256-hs128-t40-vimco16-renyi};
\addplot[tealsolid] table [x=validklp, y=validnll] {\sorted};
\addlegendentry{VIMCO 16};

\end{axis}
\end{tikzpicture}
\caption{\small KL and Rényi objectives on PTB with base estimator VIMCO. VIMCO 1 uses a single sample to estimate both $p(x)$ and $\KL(p(z|x)\|p(z))$. VIMCO 16/1 uses 16 samples to estimate $p(x)$ and 1 samples for the KL. VIMCO 1/16 uses 1 sample to estimate $p(x)$ and 16 samples for the KL. Finally, VIMCO 16 uses the same 16 samples for both terms.}
\label{fig:ptb-vimco-train-samples}
\end{figure}

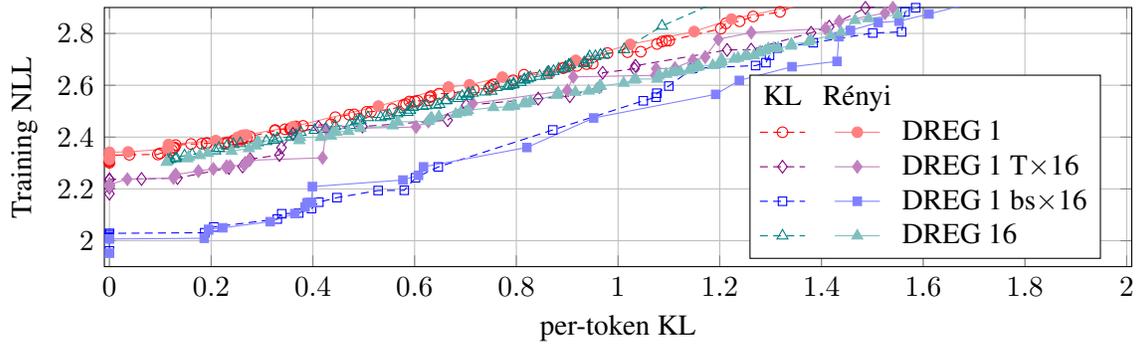
\begin{figure}
\begin{tikzpicture}
\begin{axis}[xmin=-0.01,xmax=2.01,ymin=1.9,ymax=2.9,
  xlabel=per-token KL, ylabel=Training NLL,
  height=0.33*\textwidth,
  width=\textwidth,
  grid=major,
  legend pos=south east,
  legend columns=2,
  legend style={
    font=\mystrut,
    legend cell align=left,
  },
]

\addlegendimage{legend image with text=KL}
\addlegendentry{}
\addlegendimage{legend image with text=Rényi}
\addlegendentry{}

\pgfplotstablesort[sort key=validklp]{\sorted}
                  {data/l40x0-bs256-hs128-t40-dreg1-kl}
\addplot[redhollow] table [x=validklp, y=validnll] {\sorted};
\addlegendentry{};

\pgfplotstablesort[sort key=validklp]{\sorted}
                  {data/l40x0-bs256-hs128-t40-dreg1-renyi}
\addplot[redsolid] table [x=validklp, y=validnll] {\sorted};
\addlegendentry{DREG 1};

\pgfplotstablesort[sort key=validklp]{\sorted}
                  {data/l40x0-bs256-hs128-t40x16-dreg1-kl}
\addplot[violethollow] table [x=validklp, y=validnll] {\sorted};
\addlegendentry{};

\pgfplotstablesort[sort key=validklp]{\sorted}
                  {data/l40x0-bs256-hs128-t40x16-dreg1-renyi}
\addplot[violetsolid] table [x=validklp, y=validnll] {\sorted};
\addlegendentry{DREG 1 T$\times$16};

\pgfplotstablesort[sort key=validklp]{\sorted}
                  {data/l40x0-bs4096-hs128-t40-dreg1-kl};
\addplot[bluehollow] table [x=validklp, y=validnll] {\sorted};
\addlegendentry{};

\pgfplotstablesort[sort key=validklp]{\sorted}
                  {data/l40x0-bs4096-hs128-t40-dreg1-renyi}
\addplot[bluesolid] table [x=validklp, y=validnll] {\sorted};
\addlegendentry{DREG 1 bs$\times$16};

\pgfplotstablesort[sort key=validklp]{\sorted}
                  {data/l40x0-bs256-hs128-t40-dreg16-kl};
\addplot[tealhollow] table [x=validklp, y=validnll] {\sorted};
\addlegendentry{};

\pgfplotstablesort[sort key=validklp]{\sorted}
                  {data/l40x0-bs256-hs128-t40-dreg16-renyi};
\addplot[tealsolid] table [x=validklp, y=validnll] {\sorted};
\addlegendentry{DREG 16};

\end{axis}
\end{tikzpicture}
\caption{\small Training NLL on PTB with KL and Rényi objectives and base estimator DREG. DREG 1 T$\times$16 is trained 16 times longer. DREG 1 bs$\times$16 has a 16 times larger batch size. While their best NLL at very low latent usage is lower than that of DREG 16, they lose this advantage at higher latent usage.}
\label{fig:ptb-dreg-train-x16}
\end{figure}

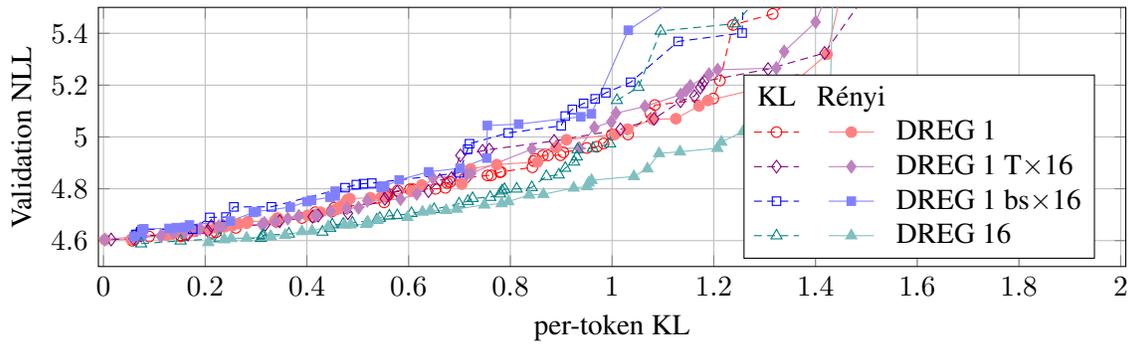
\begin{figure}
\begin{tikzpicture}
\begin{axis}[xmin=-0.01,xmax=2.01,ymin=4.5,ymax=5.5,
  xlabel=per-token KL, ylabel=Validation NLL,
  height=0.33*\textwidth,
  width=\textwidth,
  grid=major,
  legend pos=south east,
  legend columns=2,
  legend style={
    font=\mystrut,
    legend cell align=left,
  },
]

\addlegendimage{legend image with text=KL}
\addlegendentry{}
\addlegendimage{legend image with text=Rényi}
\addlegendentry{}

\pgfplotstablesort[sort key=validklp]{\sorted}
                  {data/l40x0-bs256-hs128-t40-dreg1-kl-valid}
\addplot[redhollow] table [x=validklp, y=validnll] {\sorted};
\addlegendentry{};

\pgfplotstablesort[sort key=validklp]{\sorted}
                  {data/l40x0-bs256-hs128-t40-dreg1-renyi-valid}
\addplot[redsolid] table [x=validklp, y=validnll] {\sorted};
\addlegendentry{DREG 1};

\pgfplotstablesort[sort key=validklp]{\sorted}
                  {data/l40x0-bs256-hs128-t40x16-dreg1-kl-valid}
\addplot[violethollow] table [x=validklp, y=validnll] {\sorted};
\addlegendentry{};

\pgfplotstablesort[sort key=validklp]{\sorted}
                  {data/l40x0-bs256-hs128-t40x16-dreg1-renyi-valid}
\addplot[violetsolid] table [x=validklp, y=validnll] {\sorted};
\addlegendentry{DREG 1 T$\times$16};

\pgfplotstablesort[sort key=validklp]{\sorted}
                  {data/l40x0-bs4096-hs128-t40-dreg1-kl-valid};
\addplot[bluehollow] table [x=validklp, y=validnll] {\sorted};
\addlegendentry{};

\pgfplotstablesort[sort key=validklp]{\sorted}
                  {data/l40x0-bs4096-hs128-t40-dreg1-renyi-valid}
\addplot[bluesolid] table [x=validklp, y=validnll] {\sorted};
\addlegendentry{DREG 1 bs$\times$16};

\pgfplotstablesort[sort key=validklp]{\sorted}
                  {data/l40x0-bs256-hs128-t40-dreg16-kl-valid-nsd};
\addplot[tealhollow] table [x=validklp, y=validnll] {\sorted};
\addlegendentry{};

\pgfplotstablesort[sort key=validklp]{\sorted}
                  {data/l40x0-bs256-hs128-t40-dreg16-renyi-valid-nsd};
\addplot[tealsolid] table [x=validklp, y=validnll] {\sorted};
\addlegendentry{DREG 16};

\end{axis}
\end{tikzpicture}
\caption{\small Validation NLL on PTB with KL and Rényi objectives and base estimator DREG. DREG 1 T$\times$16 is trained 16 times longer. DREG 1 bs$\times$16 has a 16 times larger batch size. DREG 16 generalizes better than either.}
\label{fig:ptb-dreg-valid-x16}
\end{figure}

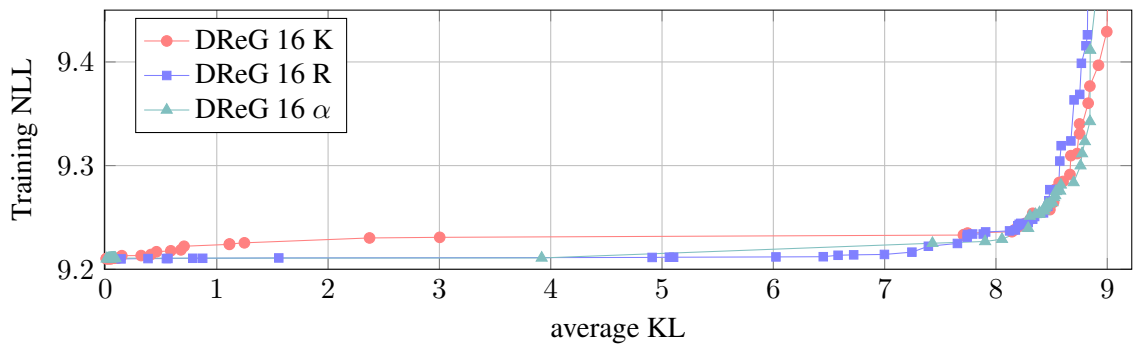
\begin{figure}
\begin{tikzpicture}
\begin{axis}[xmin=-0.01,xmax=9.22,ymin=9.20,ymax=9.45,
  xlabel=average KL, ylabel=Training NLL,
  height=0.33*\textwidth,
  width=\textwidth,
  grid=major,
  legend pos=north west,
  legend columns=1,
  legend style={
    font=\mystrut,
    legend cell align=left,
  },
]

\pgfplotstablesort[sort key=validklp]{\sorted}
                  {data/synthetic/l40x0-bs256-hs128-t40-dreg16-kl}
\addplot[redsolid] table [x=validklp, y=validnll] {\sorted};
\addlegendentry{DReG 16 K};

\pgfplotstablesort[sort key=validklp]{\sorted}
                  {data/synthetic/l40x0-bs256-hs128-t40-dreg16-renyi}
\addplot[bluesolid] table [x=validklp, y=validnll] {\sorted};
\addlegendentry{DReG 16 R};

\pgfplotstablesort[sort key=validklp]{\sorted}
                  {data/synthetic/l40x0-bs256-hs128-t40-dreg16-alpha};
\addplot[tealsolid] table [x=validklp, y=validnll] {\sorted};
\addlegendentry{DReG 16 $\alpha$};

\end{axis}
\end{tikzpicture}
\caption{\small The power objective with DReG on the synthetic data set compared to the KL and Rényi objectives.}
\label{fig:synthetic-dreg-power}
\end{figure}

\begin{figure}
\begin{tikzpicture}
\begin{axis}[xmin=-0.01,xmax=9.22,ymin=9.20,ymax=9.45,
  xlabel=average KL, ylabel=Training NLL,
  height=0.33*\textwidth,
  width=\textwidth,
  grid=major,
  legend pos=north west,
  legend columns=1,
  legend style={
    font=\mystrut,
    legend cell align=left,
  },
]

\pgfplotstablesort[sort key=validklp]{\sorted}
                  {data/synthetic/l8x10-bs256-hs128-t40-vimco16-kl}
\addplot[redsolid] table [x=validklp, y=validnll] {\sorted};
\addlegendentry{VIMCO 16 K};

\pgfplotstablesort[sort key=validklp]{\sorted}
                  {data/synthetic/l8x10-bs256-hs128-t40-vimco16-renyi}
\addplot[bluesolid] table [x=validklp, y=validnll] {\sorted};
\addlegendentry{VIMCO 16 R};

\pgfplotstablesort[sort key=validklp]{\sorted}
                  {data/synthetic/l8x10-bs256-hs128-t40-vimco16-alpha};
\addplot[tealsolid] table [x=validklp, y=validnll] {\sorted};
\addlegendentry{VIMCO 16 $\alpha$};

\end{axis}
\end{tikzpicture}
\caption{\small The power objective with VIMCO on the synthetic data set compared to the KL and Rényi objectives.}
\label{fig:synthetic-vimco-power}
\end{figure}
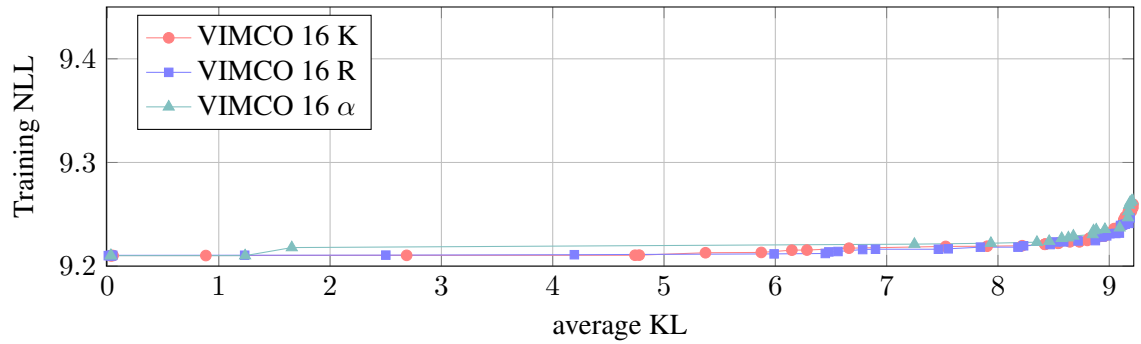

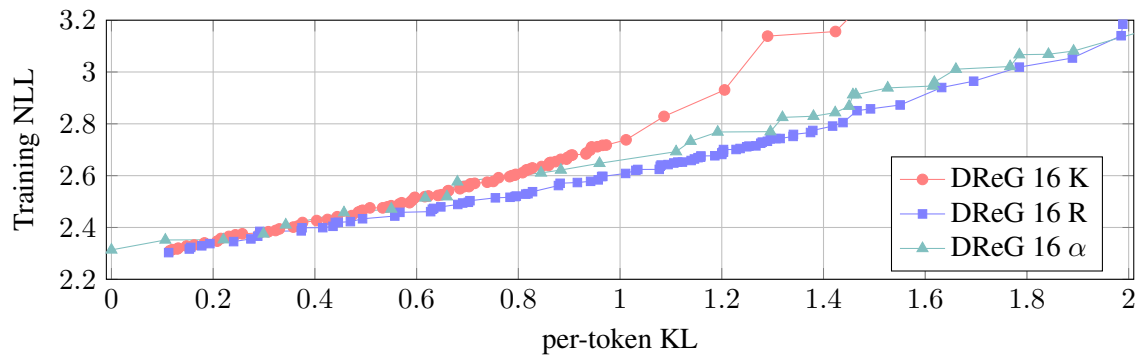
\begin{figure}
\begin{tikzpicture}
\begin{axis}[xmin=-0.01,xmax=2.01,ymin=2.2,ymax=3.2,
  xlabel=per-token KL, ylabel=Training NLL,
  height=0.33*\textwidth,
  width=\textwidth,
  grid=major,
  legend pos=south east,
  legend columns=1,
  legend style={
    font=\mystrut,
    legend cell align=left,
  },
]

\pgfplotstablesort[sort key=validklp]{\sorted}
                  {data/l40x0-bs256-hs128-t40-dreg16-kl}
\addplot[redsolid] table [x=validklp, y=validnll] {\sorted};
\addlegendentry{DReG 16 K};

\pgfplotstablesort[sort key=validklp]{\sorted}
                  {data/l40x0-bs256-hs128-t40-dreg16-renyi}
\addplot[bluesolid] table [x=validklp, y=validnll] {\sorted};
\addlegendentry{DReG 16 R};

\pgfplotstablesort[sort key=validklp]{\sorted}
                  {data/l40x0-bs256-hs128-t40-dreg16-alpha};
\addplot[tealsolid] table [x=validklp, y=validnll] {\sorted};
\addlegendentry{DReG 16 $\alpha$};

\end{axis}
\end{tikzpicture}
\caption{\small The power objective with DReG on PTB compared to the KL and Rényi objectives.}
\label{fig:ptb-dreg-power}
\end{figure}

\begin{figure}
\begin{tikzpicture}
\begin{axis}[xmin=-0.01,xmax=2.01,ymin=2.2,ymax=3.2,
  xlabel=per-token KL, ylabel=Training NLL,
  height=0.33*\textwidth,
  width=\textwidth,
  grid=major,
  legend pos=south east,
  legend columns=1,
  legend style={
    font=\mystrut,
    legend cell align=left,
  },
]

\pgfplotstablesort[sort key=validklp]{\sorted}
                  {data/l8x10-bs256-hs128-t40-vimco16-kl}
\addplot[redsolid] table [x=validklp, y=validnll] {\sorted};
\addlegendentry{VIMCO 16 K};

\pgfplotstablesort[sort key=validklp]{\sorted}
                  {data/l8x10-bs256-hs128-t40-vimco16-renyi}
\addplot[bluesolid] table [x=validklp, y=validnll] {\sorted};
\addlegendentry{VIMCO 16 R};

\pgfplotstablesort[sort key=validklp]{\sorted}
                  {data/l8x10-bs256-hs128-t40-vimco16-alpha};
\addplot[tealsolid] table [x=validklp, y=validnll] {\sorted};
\addlegendentry{VIMCO 16 $\alpha$};

\end{axis}
\end{tikzpicture}
\caption{\small The power objective with VIMCO on PTB compared to the KL and Rényi objectives.}
\label{fig:ptb-vimco-power}
\end{figure}
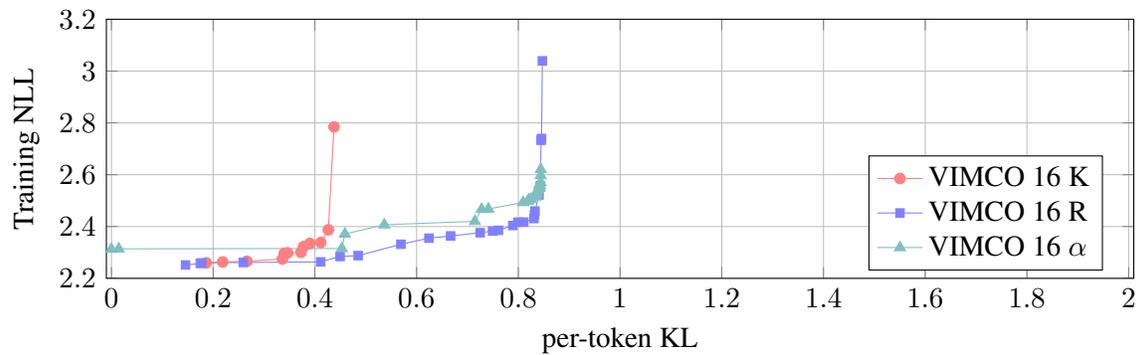

\section{Additional Language Modelling Experiments}
\label{sec:additional-language-modelling-experiments}

\subsection{Robustness}
\label{sec:additional-language-modelling-experiments-robustness}

We performed additional experiments to verify that our results about the relative merits of the estimators are robust to different choices of batch size, number of parameters and optimization length.
With continuous latents, we focussed on DReG, the best performing base estimator and varied the batch size (\Cref{fig:ptb-dreg-bs}), the length of optimization (\Cref{fig:ptb-dreg-ol-hs128}), and the length of optimization again at two times the model size (\Cref{fig:ptb-dreg-ol-hs256}).
For VIMCO, the best performing base estimator for the discrete latents, \Cref{fig:ptb-vimco-bs,fig:ptb-vimco-ol-hs128,fig:ptb-vimco-ol-hs256} tell a similar story.
In most cases, we observed that varying these nuisance factors only shifted the Pareto curves downwards or upwards, leaving their relative positions the same.
The exception is \Cref{fig:ptb-dreg-ol-hs256}, DREG with 256 hidden units, where the differences between the three estimators are very small.

Finally, similar to those in \S\ref{sec:single-vs-multi-sample}, we present experiments with single-sample DReG at an increased computational budget.
As training (\Cref{fig:ptb-dreg-train-x16}) and validation (\Cref{fig:ptb-dreg-valid-x16}) results show, the overall picture may be the same as for VIMCO, (that is, steeper training curves translate to steeper validation curves), but it is harder to assess this since the curves are closer.

\subsection{Asymmetric Samples}
\label{sec:additional-language-modelling-experiments-asymmetric-samples}

We now investigate whether the improvements are due to a better estimate of the marginal likelihood $\ln p(x)$, or the mutual information term $I_{p_D}^p(X,Z)$ in \eqref{eq:cross-mi-obj}.
\Cref{fig:ptb-dreg-train-samples,fig:ptb-vimco-train-samples} show that improvements in training fit require multiple samples in both terms, even more so than in the experiments on the synthetic data earlier (\Cref{fig:synthetic-dreg-samples,fig:synthetic-vimco-samples}).

\subsection{Experiments with the Power Objective}
\label{sec:experiments-with-the-power-objective}

Recall that the power objective \eqref{eq:power-obj} is a special case
of the Rényi objective \eqref{eq:renyi-obj-estimator} with the choice of $\lambda = (\alpha-1)/\alpha$.
With this $\lambda$, the objective simplifies to $\E_{p_D(x)} \ln p^\alpha(x)$.
Since $\ln p^\alpha(x) = \ln(p(x|z)^{\alpha} p(z)) = \alpha \ln p(x|z) + \ln p(z)$, implementing the power objective is as easy as upweighting the log-likelihood term $\ln p(x|z)$.
Here we investigate whether at the same time, by tying $\lambda$ and $\alpha$, we can maintain parity with the Rényi objective in terms efficiency of latent usage.
As \Cref{fig:synthetic-dreg-power,fig:synthetic-vimco-power,fig:ptb-dreg-power,fig:ptb-vimco-power} show, the power objective is often better than the KL objective but lags the Rényi objective as $\alpha$ determines both $\lambda$, the weight of the mutual information term and its bias with respect to the KL.

\section{Optimization Settings}
\label{sec:optimization-settings}

In all experiments, we use the Adam optimizer \citep{kingma2014adam} with $\beta_1=0$, $\beta_2=0.999$, and $\epsilon=1e-8$.
We tune hyperparameters using a black-box hyperparameter tuner based on batched Gaussian Process Bandits \citep{golovin2017google}.
Hyperparameters and their ranges are listed in \Cref{tab:hp-ranges}.
The learning rate is the only hyperparameter tuned in all experiments.
The rest only apply in specific circumstances.
\emph{Input and output dropout} are the dropout rates applied to the inputs and outputs of the LSTM, respectively, while \emph{state dropout} is the dropout rate for the LSTM's recurrent state from the previous time step \citep{gal2016theoretically}.
For a description of the VQ-VAE parameters see \citet{van2017neural}.

\begin{table}
  \centering
  \begin{tabular}{@{}lrrrr@{}}
    \toprule
    hyperparameter & min & max & scale & \\
    \midrule
    learning rate & 0.0001 & 0.01 & log & \\
    $\lambda$ & -0.05 & 0.9999 & & KL and Rényi only \\
    $\alpha$ & 0.86 & 16.0 & log & power objective only \\
    input dropout & 0.0 & 0.9 & & validation only \\
    state dropout & 0.0 & 0.8 & & validation only \\
    output dropout & 0.0 & 0.95 & & validation only \\
    L2 penalty coefficient & 5e-6 & 1e-3 & log & validation only \\
    number of latents & 1 & 8 & & VQ-VAE only \\
    number of categories of latents & 2 & 20 & & VQ-VAE only \\
    VQ $\beta$ & 0.01 & 20.0 & log & VQ-VAE only \\
    VQ decay & 0.9 & 0.99999 & log & VQ-VAE only \\
    \bottomrule
  \end{tabular}
  \caption{Hyperparameter tuning ranges.}
  \label{tab:hp-ranges}
\end{table}

{
\small
\bibliography{paper}

\begin{thebibliography}{53}
\providecommand{\natexlab}[1]{#1}
\providecommand{\url}[1]{\texttt{#1}}
\expandafter\ifx\csname urlstyle\endcsname\relax
  \providecommand{\doi}[1]{doi: #1}\else
  \providecommand{\doi}{doi: \begingroup \urlstyle{rm}\Url}\fi

\bibitem[Alemi et~al.(2017)Alemi, Poole, Fischer, Dillon, Saurous, and
  Murphy]{alemi2017fixing}
Alexander~A Alemi, Ben Poole, Ian Fischer, Joshua~V Dillon, Rif~A Saurous, and
  Kevin Murphy.
\newblock Fixing a broken elbo.
\newblock \emph{arXiv preprint arXiv:1711.00464}, 2017.

\bibitem[Bornschein and Bengio(2014)]{bornschein2014reweighted}
J{\"o}rg Bornschein and Yoshua Bengio.
\newblock Reweighted wake-sleep.
\newblock 2014.

\bibitem[Bowman et~al.(2015)Bowman, Vilnis, Vinyals, Dai, Jozefowicz, and
  Bengio]{bowman2015generating}
Samuel~R Bowman, Luke Vilnis, Oriol Vinyals, Andrew~M Dai, Rafal Jozefowicz,
  and Samy Bengio.
\newblock Generating sentences from a continuous space.
\newblock \emph{arXiv preprint arXiv:1511.06349}, 2015.

\bibitem[Burda et~al.(2015)Burda, Grosse, and
  Salakhutdinov]{burda2015importance}
Yuri Burda, Roger Grosse, and Ruslan Salakhutdinov.
\newblock Importance weighted autoencoders.
\newblock \emph{arXiv preprint arXiv:1509.00519}, 2015.

\bibitem[Cremer et~al.(2017)Cremer, Morris, and
  Duvenaud]{cremer2017reinterpreting}
Chris Cremer, Quaid Morris, and David Duvenaud.
\newblock Reinterpreting importance-weighted autoencoders.
\newblock \emph{arXiv preprint arXiv:1704.02916}, 2017.

\bibitem[Dieng et~al.(2017)Dieng, Tran, Ranganath, Paisley, and
  Blei]{dieng2017variational}
Adji~Bousso Dieng, Dustin Tran, Rajesh Ranganath, John Paisley, and David Blei.
\newblock Variational inference via $\chi$ upper bound minimization.
\newblock In \emph{Advances in Neural Information Processing Systems}, pages
  2732--2741, 2017.

\bibitem[Gal and Ghahramani(2016)]{gal2016theoretically}
Yarin Gal and Zoubin Ghahramani.
\newblock A theoretically grounded application of dropout in recurrent neural
  networks.
\newblock In \emph{Advances in Neural Information Processing Systems}, pages
  1019--1027, 2016.

\bibitem[Golovin et~al.(2017)Golovin, Solnik, Moitra, Kochanski, Karro, and
  Sculley]{golovin2017google}
Daniel Golovin, Benjamin Solnik, Subhodeep Moitra, Greg Kochanski, John Karro,
  and D~Sculley.
\newblock Google {Vizier}: A service for black-box optimization.
\newblock In \emph{Proceedings of the 23rd ACM SIGKDD International Conference
  on Knowledge Discovery and Data Mining}, pages 1487--1495. ACM, 2017.

\bibitem[He et~al.(2019)He, Spokoyny, Neubig, and
  Berg-Kirkpatrick]{he2018lagging}
Junxian He, Daniel Spokoyny, Graham Neubig, and Taylor Berg-Kirkpatrick.
\newblock Lagging inference networks and posterior collapse in variational
  autoencoders.
\newblock In \emph{International Conference on Learning Representations}, 2019.

\bibitem[Higgins et~al.(2017)Higgins, Matthey, Pal, Burgess, Glorot, Botvinick,
  Mohamed, and Lerchner]{higgins2016beta}
Irina Higgins, Loic Matthey, Arka Pal, Christopher Burgess, Xavier Glorot,
  Matthew Botvinick, Shakir Mohamed, and Alexander Lerchner.
\newblock $\beta$-{VAE}: Learning basic visual concepts with a constrained
  variational framework.
\newblock In \emph{International Conference on Machine Learning}, 2017.

\bibitem[Hochreiter and Schmidhuber(1995)]{hochreiter1995simplifying}
Sepp Hochreiter and J{\"u}rgen Schmidhuber.
\newblock Simplifying neural nets by discovering flat minima.
\newblock In \emph{Advances in neural information processing systems}, pages
  529--536, 1995.

\bibitem[Hochreiter and Schmidhuber(1997)]{hochreiter1997long}
Sepp Hochreiter and J{\"u}rgen Schmidhuber.
\newblock Long short-term memory.
\newblock \emph{Neural computation}, 9\penalty0 (8):\penalty0 1735--1780, 1997.

\bibitem[Huang et~al.(2018)Huang, Tan, Lacoste, and
  Courville]{NEURIPS2018_65b0df23}
Chin-Wei Huang, Shawn Tan, Alexandre Lacoste, and Aaron~C Courville.
\newblock Improving explorability in variational inference with annealed
  variational objectives.
\newblock In \emph{Advances in Neural Information Processing Systems},
  volume~31. Curran Associates, Inc., 2018.

\bibitem[Husz{\'a}r(2017)]{huszar2017representation}
Ferenc Husz{\'a}r.
\newblock Is maximum likelihood useful for representation learning?
\newblock
  \url{http://web.archive.org/web/20190704042553/https://www.inference.vc/maximum-likelihood-for-representation-learning-2/},
  2017.
\newblock Accessed: 2020-04-15.

\bibitem[Jordan et~al.(1999)Jordan, Ghahramani, Jaakkola, and
  Saul]{jordan1999introduction}
Michael~I Jordan, Zoubin Ghahramani, Tommi~S Jaakkola, and Lawrence~K Saul.
\newblock An introduction to variational methods for graphical models.
\newblock \emph{Machine learning}, 37\penalty0 (2):\penalty0 183--233, 1999.

\bibitem[Kim et~al.(2018)Kim, Wiseman, Miller, Sontag, and Rush]{kim2018semi}
Yoon Kim, Sam Wiseman, Andrew~C Miller, David Sontag, and Alexander~M Rush.
\newblock Semi-amortized variational autoencoders.
\newblock \emph{arXiv preprint arXiv:1802.02550}, 2018.

\bibitem[Kingma and Ba(2014)]{kingma2014adam}
Diederik~P Kingma and Jimmy Ba.
\newblock Adam: A method for stochastic optimization.
\newblock \emph{arXiv preprint arXiv:1412.6980}, 2014.

\bibitem[Kingma and Welling(2013)]{kingma2013auto}
Diederik~P Kingma and Max Welling.
\newblock Auto-encoding variational bayes.
\newblock \emph{arXiv preprint arXiv:1312.6114}, 2013.

\bibitem[Kingma et~al.(2016)Kingma, Salimans, Jozefowicz, Chen, Sutskever, and
  Welling]{kingma2016improved}
Diederik~P Kingma, Tim Salimans, Rafal Jozefowicz, Xi~Chen, Ilya Sutskever, and
  Max Welling.
\newblock Improved variational inference with inverse autoregressive flow.
\newblock In \emph{Advances in Neural Information Processing Systems}, pages
  4743--4751, 2016.

\bibitem[Kullback and Leibler(1951)]{kullback1951information}
Solomon Kullback and Richard~A Leibler.
\newblock On information and sufficiency.
\newblock \emph{The annals of mathematical statistics}, 22\penalty0
  (1):\penalty0 79--86, 1951.

\bibitem[Maal\o~e et~al.(2019)Maal\o~e, Fraccaro, Li\'{e}vin, and
  Winther]{NEURIPS2019_9bdb8b1f}
Lars Maal\o~e, Marco Fraccaro, Valentin Li\'{e}vin, and Ole Winther.
\newblock Biva: A very deep hierarchy of latent variables for generative
  modeling.
\newblock In \emph{Advances in Neural Information Processing Systems},
  volume~32, 2019.

\bibitem[Maddison et~al.(2017)Maddison, Lawson, Tucker, Heess, Norouzi, Mnih,
  Doucet, and Teh]{maddison2017filtering}
Chris~J Maddison, John Lawson, George Tucker, Nicolas Heess, Mohammad Norouzi,
  Andriy Mnih, Arnaud Doucet, and Yee Teh.
\newblock Filtering variational objectives.
\newblock In \emph{Advances in Neural Information Processing Systems}, pages
  6573--6583, 2017.

\bibitem[Marcus et~al.(1993)Marcus, Marcinkiewicz, and
  Santorini]{marcus1993building}
Mitchell~P Marcus, Mary~Ann Marcinkiewicz, and Beatrice Santorini.
\newblock Building a large annotated corpus of english: The {Penn} treebank.
\newblock \emph{Computational linguistics}, 19\penalty0 (2):\penalty0 313--330,
  1993.

\bibitem[McCarthy et~al.(2019)McCarthy, Li, Gu, and Dong]{mccarthy2019improved}
Arya~D McCarthy, Xian Li, Jiatao Gu, and Ning Dong.
\newblock Improved variational neural machine translation by promoting mutual
  information.
\newblock \emph{arXiv preprint arXiv:1909.09237}, 2019.

\bibitem[Melis et~al.(2017)Melis, Dyer, and Blunsom]{melis2017state}
G{\'a}bor Melis, Chris Dyer, and Phil Blunsom.
\newblock On the state of the art of evaluation in neural language models.
\newblock \emph{arXiv preprint arXiv:1707.05589}, 2017.

\bibitem[Merity et~al.(2017)Merity, Keskar, and Socher]{merity2017regularizing}
Stephen Merity, Nitish~Shirish Keskar, and Richard Socher.
\newblock Regularizing and optimizing {LSTM} language models.
\newblock \emph{arXiv preprint arXiv:1708.02182}, 2017.

\bibitem[Metelli et~al.(2020)Metelli, Papini, Montali, and
  Restelli]{metelli2020importance}
Alberto~Maria Metelli, Matteo Papini, Nico Montali, and Marcello Restelli.
\newblock Importance sampling techniques for policy optimization.
\newblock \emph{J. Mach. Learn. Res.}, 21:\penalty0 141--1, 2020.

\bibitem[Mikolov et~al.(2010)Mikolov, Karafi{\'a}t, Burget, Cernock{\`y}, and
  Khudanpur]{mikolov2010recurrent}
Tomas Mikolov, Martin Karafi{\'a}t, Lukas Burget, Jan Cernock{\`y}, and Sanjeev
  Khudanpur.
\newblock Recurrent neural network based language model.
\newblock In \emph{Interspeech}, volume~2, page~3, 2010.

\bibitem[Mikolov et~al.(2013)Mikolov, Sutskever, Chen, Corrado, and
  Dean]{mikolov2013distributed}
Tomas Mikolov, Ilya Sutskever, Kai Chen, Greg~S Corrado, and Jeff Dean.
\newblock Distributed representations of words and phrases and their
  compositionality.
\newblock In \emph{Advances in neural information processing systems}, pages
  3111--3119, 2013.

\bibitem[Mnih and Gregor(2014)]{mnih2014neural}
Andriy Mnih and Karol Gregor.
\newblock Neural variational inference and learning in belief networks.
\newblock \emph{arXiv preprint arXiv:1402.0030}, 2014.

\bibitem[Mnih and Rezende(2016)]{mnih2016variational}
Andriy Mnih and Danilo~J Rezende.
\newblock Variational inference for monte carlo objectives.
\newblock \emph{arXiv preprint arXiv:1602.06725}, 2016.

\bibitem[Nowozin(2018)]{nowozin2018debiasing}
Sebastian Nowozin.
\newblock Debiasing evidence approximations: On importance-weighted
  autoencoders and jackknife variational inference.
\newblock 2018.

\bibitem[Owen(2013)]{art2013mcbook}
Art~B. Owen.
\newblock \emph{Monte Carlo theory, methods and examples}.
\newblock 2013.

\bibitem[Pelsmaeker and Aziz(2019)]{pelsmaeker2019effective}
Tom Pelsmaeker and Wilker Aziz.
\newblock Effective estimation of deep generative language models.
\newblock \emph{arXiv preprint arXiv:1904.08194}, 2019.

\bibitem[Phuong et~al.(2018)Phuong, Welling, Kushman, Tomioka, and
  Nowozin]{phuong2018mutual}
Mary Phuong, Max Welling, Nate Kushman, Ryota Tomioka, and Sebastian Nowozin.
\newblock The mutual autoencoder: Controlling information in latent code
  representations, 2018.
\newblock URL \url{https://openreview.net/forum?id=HkbmWqxCZ}.

\bibitem[Rainforth et~al.(2018)Rainforth, Kosiorek, Le, Maddison, Igl, Wood,
  and Teh]{rainforth2018tighter}
Tom Rainforth, Adam~R Kosiorek, Tuan~Anh Le, Chris~J Maddison, Maximilian Igl,
  Frank Wood, and Yee~Whye Teh.
\newblock Tighter variational bounds are not necessarily better.
\newblock \emph{arXiv preprint arXiv:1802.04537}, 2018.

\bibitem[Razavi et~al.(2019)Razavi, van~den Oord, Poole, and
  Vinyals]{razavi2018preventing}
Ali Razavi, Aaron van~den Oord, Ben Poole, and Oriol Vinyals.
\newblock Preventing posterior collapse with delta-{VAE}s.
\newblock In \emph{International Conference on Learning Representations}, 2019.

\bibitem[Rezaabad and Vishwanath(2020)]{rezaabad2020learning}
Ali~Lotfi Rezaabad and Sriram Vishwanath.
\newblock Learning representations by maximizing mutual information in
  variational autoencoders.
\newblock In \emph{2020 IEEE International Symposium on Information Theory
  (ISIT)}, pages 2729--2734. IEEE, 2020.

\bibitem[Rezende and Mohamed(2015)]{rezende2015variational}
Danilo~Jimenez Rezende and Shakir Mohamed.
\newblock Variational inference with normalizing flows.
\newblock \emph{arXiv preprint arXiv:1505.05770}, 2015.

\bibitem[Rezende et~al.(2014)Rezende, Mohamed, and
  Wierstra]{rezende2014stochastic}
Danilo~Jimenez Rezende, Shakir Mohamed, and Daan Wierstra.
\newblock Stochastic backpropagation and approximate inference in deep
  generative models.
\newblock \emph{arXiv preprint arXiv:1401.4082}, 2014.

\bibitem[Roeder et~al.(2017)Roeder, Wu, and Duvenaud]{roeder2017sticking}
Geoffrey Roeder, Yuhuai Wu, and David~K Duvenaud.
\newblock Sticking the landing: Simple, lower-variance gradient estimators for
  variational inference.
\newblock In \emph{Advances in Neural Information Processing Systems 30}, 2017.

\bibitem[Serdega and Kim(2020)]{serdega2020vmi}
Andriy Serdega and Dae-Shik Kim.
\newblock Vmi-vae: Variational mutual information maximization framework for
  vae with discrete and continuous priors.
\newblock \emph{arXiv preprint arXiv:2005.13953}, 2020.

\bibitem[Shu et~al.(2018)Shu, Bui, Zhao, Kochenderfer, and
  Ermon]{DBLP:conf/nips/ShuBZKE18}
Rui Shu, Hung~H. Bui, Shengjia Zhao, Mykel~J. Kochenderfer, and Stefano Ermon.
\newblock Amortized inference regularization.
\newblock In \emph{NeurIPS}, pages 4398--4407, 2018.

\bibitem[S{\o}nderby et~al.(2016)S{\o}nderby, Raiko, Maal{\o}e, S{\o}nderby,
  and Winther]{sonderby2016ladder}
Casper~Kaae S{\o}nderby, Tapani Raiko, Lars Maal{\o}e, S{\o}ren~Kaae
  S{\o}nderby, and Ole Winther.
\newblock Ladder variational autoencoders.
\newblock \emph{Advances in neural information processing systems},
  29:\penalty0 3738--3746, 2016.

\bibitem[Titsias and L{\'a}zaro-Gredilla(2014)]{titsias2014doubly}
Michalis Titsias and Miguel L{\'a}zaro-Gredilla.
\newblock Doubly stochastic variational bayes for non-conjugate inference.
\newblock In \emph{International Conference on Machine Learning}, pages
  1971--1979, 2014.

\bibitem[Tomczak and Welling(2018)]{tomczak2018vae}
Jakub Tomczak and Max Welling.
\newblock Vae with a vampprior.
\newblock In \emph{International Conference on Artificial Intelligence and
  Statistics}, pages 1214--1223. PMLR, 2018.

\bibitem[Tucker et~al.(2018)Tucker, Lawson, Gu, and Maddison]{tucker2018doubly}
George Tucker, Dieterich Lawson, Shixiang Gu, and Chris~J Maddison.
\newblock Doubly reparameterized gradient estimators for monte carlo
  objectives.
\newblock \emph{arXiv preprint arXiv:1810.04152}, 2018.

\bibitem[van~den Oord et~al.(2017)van~den Oord, Vinyals, et~al.]{van2017neural}
Aaron van~den Oord, Oriol Vinyals, et~al.
\newblock Neural discrete representation learning.
\newblock In \emph{Advances in Neural Information Processing Systems}, pages
  6306--6315, 2017.

\bibitem[Williams(1987)]{williams1987class}
R~Williams.
\newblock A class of gradient-estimation algorithms for reinforcement learning
  in neural networks.
\newblock In \emph{Proceedings of the International Conference on Neural
  Networks}, pages II--601, 1987.

\bibitem[Yang et~al.(2017)Yang, Hu, Salakhutdinov, and
  Berg-Kirkpatrick]{yang2017improved}
Zichao Yang, Zhiting Hu, Ruslan Salakhutdinov, and Taylor Berg-Kirkpatrick.
\newblock Improved variational autoencoders for text modeling using dilated
  convolutions.
\newblock In \emph{International conference on machine learning}, pages
  3881--3890. PMLR, 2017.

\bibitem[Yeung et~al.(2017)Yeung, Kannan, Dauphin, and
  Fei-Fei]{yeung2017tackling}
Serena Yeung, Anitha Kannan, Yann Dauphin, and Li~Fei-Fei.
\newblock Tackling over-pruning in variational autoencoders.
\newblock \emph{arXiv preprint arXiv:1706.03643}, 2017.

\bibitem[Zhang et~al.(2018)Zhang, B{\"u}tepage, Kjellstr{\"o}m, and
  Mandt]{zhang2018advances}
Cheng Zhang, Judith B{\"u}tepage, Hedvig Kjellstr{\"o}m, and Stephan Mandt.
\newblock Advances in variational inference.
\newblock \emph{IEEE transactions on pattern analysis and machine
  intelligence}, 41\penalty0 (8):\penalty0 2008--2026, 2018.

\bibitem[Zhao et~al.(2019)Zhao, Song, and Ermon]{zhao2019infovae}
Shengjia Zhao, Jiaming Song, and Stefano Ermon.
\newblock Infovae: Balancing learning and inference in variational
  autoencoders.
\newblock In \emph{Proceedings of the AAAI Conference on Artificial
  Intelligence}, volume~33, pages 5885--5892, 2019.

\end{thebibliography}
}

\end{document}